\documentclass{article}
\usepackage{arxiv}

\usepackage[utf8]{inputenc} % allow utf-8 input
\usepackage[T1]{fontenc}    % use 8-bit T1 fonts
\usepackage{hyperref}       % hyperlinks
\usepackage{url}            % simple URL typesetting
\usepackage{booktabs}       % professional-quality tables
\usepackage{amsfonts}       % blackboard math symbols
\usepackage{nicefrac}       % compact symbols for 1/2, etc.
\usepackage{microtype}      % microtypography
\usepackage{lipsum}		% Can be removed after putting your text content
\usepackage{graphicx}
\usepackage{natbib}
\usepackage{doi}
\usepackage{savesym}
\usepackage{multicol}
\usepackage{multirow}
\usepackage{dsfont}
\usepackage{tabularx}
\usepackage[textsize=tiny]{todonotes}
\usepackage{amsmath}
\usepackage{mathtools}
\usepackage{graphicx}
\usepackage{times}
\usepackage{helvet}
\usepackage{courier}
\usepackage{paralist}
\usepackage{latexsym}
\usepackage{url}
\usepackage[all]{xy}
\usepackage{amsmath}
\usepackage{amssymb}
\usepackage{amsthm}
\usepackage{nccmath} % mfrac
\usepackage{wrapfig}
\usepackage{comment}
\usepackage{paralist}
\usepackage{xcolor}
\usepackage{graphicx}
\usepackage{pifont}
\usepackage{savesym}
\savesymbol{AND}
\usepackage{xspace}
\usepackage{tikz}
\usepackage{pgfplots}
\usepackage{pgf}
\usepackage{algorithm}
\usepackage{algorithmic}
\usepackage{xspace}
\usepackage{comment}
\usepackage{placeins}
\usepackage{cancel}
\usepackage{changepage}
\usepackage{bbm}
\usepackage{pgfplots}
\usepackage{filecontents}
\usepackage{placeins}
\usepackage{esvect}

\defcitealias{ch2020neural}{LS2020}
\definecolor{electricpurple}{rgb}{0.75, 0.0, 1.0}
\definecolor{darkbrown}{rgb}{0.4, 0.26, 0.13}
\definecolor{cocoabrown}{rgb}{0.82, 0.41, 0.12}
\definecolor{copper}{rgb}{0.72, 0.45, 0.2}
\definecolor{chocolate}{rgb}{0.82, 0.41, 0.12}
\usepackage{bm}

\usepackage[capitalize]{cleveref}
\if0
\usetikzlibrary{intersections}
\usetikzlibrary{arrows,calc,fit,patterns,plotmarks,shapes.geometric,shapes.misc,shapes.symbols,   shapes.arrows,   shapes.callouts,   shapes.multipart,   shapes.gates.logic.US,   shapes.gates.logic.IEC,   er,   automata,   backgrounds,   chains,   topaths,   trees,   petri,   mindmap,   matrix,   calendar,   folding, fadings,   through,   positioning,   scopes,   decorations.fractals,   decorations.shapes,   decorations.text,   decorations.pathmorphing,   decorations.pathreplacing,   decorations.footprints,   decorations.markings, shadows,circuits}
\tikzstyle{decision}=[diamond,draw]
\tikzstyle{line}=[draw]
\tikzstyle{elli}=[draw,ellipse]
\tikzstyle{arrow} = [thick]
\fi
\newcommand{\I}{\mathcal{I}}
\newcommand{\Ifeat}{\mathcal{I}^{\text{f}}}

\newcommand{\Iconv}{\mathcal{I}^{\text{cv}}}

\newcommand{\din}{d_{\text{in}}}
\newcommand{\dnet}{d_{\text{net}}}

\newcommand{\Tdgn}{\Theta^{\text{DGN}}}

\newcommand{\Tv}{{\Theta}^{\text{v}}}

\newcommand{\Tf}{\Theta^{\textrm{f}}}

\newcommand{\dc}{d_{\text{cv}}}
\newcommand{\dconv}{d_{\text{cv}}}

\newcommand{\dfc}{d_{\text{fc}}}

\newcommand{\dblock}{d_{\text{blk}}}

\newcommand{\wconv}{w_{\text{cv}}}

\newcommand{\ifout}{i_{\text{fout}}}
\newcommand{\icin}{i_{\text{cv}}}
\newcommand{\iin}{i_{\text{in}}}
\newcommand{\iout}{i_{\text{out}}}

\newcommand{\xv}{x^{\text{v}}}
\newcommand{\xf}{x^{\text{f}}}

\newcommand{\Pres}{P^{\text{res}}}
\newcommand{\Pfc}{P^{\text{fc}}}
\newcommand{\Pcnn}{P^{\text{cnn}}}

\newcommand{\eqdef}{\stackrel{\Delta}{=}}

\newcommand{\cscale}{c_{\text{scale}}}
\newcommand{\sigfc}{\sigma_{\text{fc}}}
\newcommand{\sigcnn}{\sigma_{\text{cv}}}

\newcommand{\bfc}{\beta_{\text{fc}}}
\newcommand{\bcnn}{\beta_{\text{cv}}}
\newcommand{\bres}{\beta_{\text{res}}}

\newtheorem{theorem}{Theorem}[section]
\newtheorem{lemma}{Lemma}[section]

\newtheorem{proposition}{Proposition}[section]

\newtheorem{assumption}{Assumption}[section]
\newtheorem{definition}{Definition}[section]

\newtheorem{notation}{Notation}[section]
 
\newcommand{\J}{\mathcal{J}}
\newcommand{\ip}[1]{\langle #1\rangle}

\def\R{\mathbb{R}}

\newcounter{subequation}[equation]

\def\mathdisplay#1{%
  \ifmmode \@badmath
  \else
    $\def\@currenvir{#1}%
    \let\dspbrk@context\z@
    \let\tag\tag@in@display \SK@equationtrue %\let\label\label@in@display
    \global\let\df@label\@empty \global\let\df@tag\@empty
    \global\tag@false
    \let\mathdisplay@push\mathdisplay@@push
    \let\mathdisplay@pop\mathdisplay@@pop
    \if@fleqn
      \edef\restore@hfuzz{\hfuzz\the\hfuzz\relax}%
      \hfuzz\maxdimen
      \setbox\z@\hbox to\displaywidth\bgroup
        \let\split@warning\relax \restore@hfuzz
        \everymath\@emptytoks \m@th $\displaystyle
    \fi
}

\title{Disentangling deep neural networks with rectified linear units using duality}

\author{ Chandrashekar Lakshminarayanan\\
Indian Institute of Technology Madras\\
\texttt{chandrashekar@cse.iitm.ac.in}\\
\And
Amit Vikram Singh\\
\texttt{amitkvikram@gmail.com}\\
}

\renewcommand{\shorttitle}{\textit{Disentangling DNNs with ReLUs using duality}}

\hypersetup{
pdftitle={Disentangling deep neural networks with rectified linear units using duality},
pdfsubject={cs-ml-dl},
pdfauthor={Chandrashekar Lakshminarayanan, Amit Vikram Singh},
pdfkeywords={First keyword, Second keyword, More},
}

\begin{document}
\maketitle

%\begin{comment}
\begin{abstract}
Despite their success deep neural networks (DNNs) are still largely considered as black boxes. The main issue is that the linear and non-linear operations are entangled in every layer, making it hard to interpret the hidden layer outputs. In this paper, we look at DNNs with rectified linear units (ReLUs), and focus on the gating property (`on/off' states) of the ReLUs. We extend the recently developed dual view in which the computation is broken path-wise to show that learning in the gates is more crucial, and learning the weights given the gates is characterised analytically via the so called \emph{neural path kernel} (NPK) which depends on inputs and gates. In this paper, we present novel results to show that convolution with global pooling and skip connection provide respectively rotational invariance and ensemble structure to the NPK. To address `black box'-ness, we propose a novel interpretable counterpart of DNNs with ReLUs namely deep linearly gated networks (DLGN): the pre-activations to the gates are generated by a deep linear network, and the gates are then applied as external masks to learn the weights in a different network. The DLGN is not an alternative architecture per se, but a disentanglement and an interpretable re-arrangement of the computations in a DNN with ReLUs. The DLGN disentangles the computations into two  `mathematically' interpretable linearities (i) the `primal' linearity between the input and the pre-activations in the gating network and (ii) the `dual' linearity in the path space in the weights network characterised by the NPK. We compare the performance of DNN, DGN and DLGN on CIFAR-10 and CIFAR-100 to show that, the DLGN recovers more than $83.5\%$ of the performance of state-of-the-art DNNs, i.e., while entanglement in the DNNs enable their improved performance,  the `disentangled and interpretable'  computations in the DLGN recovers most part of the performance. This brings us to an interesting question: `Is DLGN a universal spectral approximator?'%Finally, we use dual view to show that the degradation in the gates is the reason for degradation in the test performance due to upstream training with random labels (this was an open question in \cite{randlabel}).
%the DLGN counterparts of state of the art DNNs recover more than $83.5\%$ of the performance of the DNNs, which implies that while entanglement in the DNNs enable their improved performance,\
 \end{abstract}

\section{Introduction}\label{sec:intro}
Despite their success deep neural networks (DNNs) are still largely considered as black boxes.  The main issue is that in each layer of a DNN, the linear computation, i.e., multiplication by the weight matrix  and the non-linear activations are entangled. Such entanglement has its \emph{pros} and \emph{cons}. The commonly held view  is that such entanglement is the key to success of DNNs, in that, it allows DNNs to learn sophisticated structures in a layer-by-layer manner. However, in terms of interpretability, such entanglement has an adverse effect: only the final layer is linear and amenable to a feature/weight interpretation, and the hidden layers are non-interpretable due to the non-linearities. 

Prior works \citep{ntk,arora2019exact,cao2019generalization} showed that training an infinite width DNN with gradient descent is equivalent to a kernel method with the so called \emph{neural tangent kernel} matrix. As a pure kernel method, the NTK matrix performed better than other pure kernel methods. However, in relation to `black box'-ness, there are two issues with NTK theory: (i) \textbf{Issue I:} Infinite width NTK matrix does not explain fully the success of DNNs because it was observed that finite width DNNs outperform their infinite width NTK counterparts, and it was an open question to understand this performance gap \citep{arora2019exact}, and (ii) \textbf{Issue II:} Since the NTK is based on the gradients, it does not offer further insights about the inner workings of DNNs even for infinite width.

A dual view for DNNs with rectified linear units (ReLUs) was recently developed by \cite{npk} who exploited the gating property (i.e., `on/off' states) of the ReLUs. The dual view is essentially \emph{linearity in the path space}, i.e., the output is the summation of path contributions.  While the weights in a path are the same for each input, whether or not a path contributes to the output is entirely dictated by the gates in the path, which are `on/off' based on the input. To understand the role of the gates, a deep gated network (DGN) (see \Cref{fig:lgln}) was used to \emph{disentangle} the \emph{learning in the gates} from the \emph{learning in weights}. In a DGN,  the gates are generated (and learnt) in a `gating network' which is a DNN with ReLUs and are applied as external signals and the weights are learnt in a `weight network' consisting of \emph{gated linear units} (GaLUs) \citep{sss}. Each GaLU multiplies its pre-activation and the external gating signal. Using the DGN, two important insights were provided:  (i) learning in the gates is the most crucial for finite width networks to outperform infinite width NTK; this addresses \textbf{Issue I}, and (ii) in the limit of  infinite width, learning the weights with fixed gates, the NTK is equal to (but for a scalar) a so called neural path kernel (NPK) which is a kernel solely based on inputs and gates. This shifts \textbf{Issue II} on interpretability to that of interpretability of the gates as opposed to interpretability of the gradients.
\begin{figure}[!t]
\centering
\begin{minipage}{0.7\columnwidth}
\centering
\begin{minipage}{0.49\columnwidth}
\centering

\resizebox{0.8\columnwidth}{!}{
\begin{tikzpicture}

\node []  (fntext)at (5,1+0.25) {\tiny{Deep Gated Network}};
%Feature Network
\node [draw,
	minimum width=2cm,
	minimum height=0.625cm,
]  (fnbox)at (5,0.375+0.25) {};
\node []  (fntext)at (5,0.5+0.25) {\tiny{Gating Network}};

\node []  (fntext)at (5,0.25+0.25) {\tiny{DNN with ReLUs}};

%Feature Network Input
\node (fin) [left of=fnbox,node distance=1.25cm, coordinate] {};
\node[left=-1pt] at (fin.west){\tiny{Input}};
\draw[-stealth] (fin.center) -- (fnbox.west);

%Value Network

\node [draw,
	minimum width=2cm,
	minimum height=0.625cm,
]  (vnbox)at (5,-0.875) {};
\node []  (fntext)at (5,-0.75) {\tiny{Weight Network}};
\node []  (vntext)at (5,-1) {\tiny{DNN with GaLUs}};

%Value Network Input
\node (vin) [left of=vnbox,node distance=1.25cm, coordinate] {};
\node[left=-1pt] at (vin.west){\tiny{Input}};
\draw[-stealth] (vin.center) -- (vnbox.west);

%Feature Network Output
\node (vout) [right of=vnbox,node distance=1.25cm, coordinate] {};
\node[right=-1pt] at (vout.west){\tiny{Output}};
\draw[-stealth]  (vnbox.east)--(vout.center);

\draw[-stealth]  (fnbox.south)--(vnbox.north);

\node []  (gates)at (5.75,-0.125) {\tiny{Gating Signal}};

\end{tikzpicture}
}
\end{minipage}
\begin{minipage}{0.49\columnwidth}
\centering

\resizebox{0.8\columnwidth}{!}{
\begin{tikzpicture}

\node []  (fntext)at (5,1+0.25) {\tiny{Deep Linearly Gated Network}};
%Feature Network
\node [draw,
	minimum width=2cm,
	minimum height=0.625cm,
]  (fnbox)at (5,0.375+0.25) {};
\node []  (fntext)at (5,0.5+0.25) {\tiny{Gating Network}};

\node []  (fntext)at (5,0.25+0.25) {\tiny{\bf\color{red}{Deep Linear Network}}};

%Feature Network Input
\node (fin) [left of=fnbox,node distance=1.25cm, coordinate] {};
\node[left=-1pt] at (fin.west){\tiny{Input}};
\draw[-stealth] (fin.center) -- (fnbox.west);

%Value Network

\node [draw,
	minimum width=2cm,
	minimum height=0.625cm,
]  (vnbox)at (5,-0.875) {};
\node []  (fntext)at (5,-0.75) {\tiny{Weight Network}};
\node []  (vntext)at (5,-1) {\tiny{DNN with GaLUs}};

%Value Network Input
\node (vin) [left of=vnbox,node distance=1.25cm, coordinate] {};
\node[left=-1pt] at (vin.west){\tiny{\color{red}{Input =$\mathbf{1}$}}};
\draw[-stealth] (vin.center) -- (vnbox.west);

%Feature Network Output
\node (vout) [right of=vnbox,node distance=1.25cm, coordinate] {};
\node[right=-1pt] at (vout.west){\tiny{Output}};
\draw[-stealth]  (vnbox.east)--(vout.center);

\node []  (preacr)at (5.75,0.125) {\tiny{Pre-activations}};

\node []  (gating)at (5.75,-0.375) {\tiny{Gating Signal}};

\node []  (gates)at (5.0,-0.125) {\color{red}\tiny{Gates}};

\node [draw,
	minimum width=1cm,
	minimum height=0.1cm,
]  (gbox)at (5,-0.125) {};

\draw[-stealth]  (gbox.south)--(vnbox.north);
\draw[-stealth]  (fnbox.south)--(gbox.north);

\end{tikzpicture}
}
\end{minipage}
\end{minipage}
\caption{\small{DGN is a setup to understand the role of gating in DNNs with ReLUs. The DLGN setup completely disentangles and re-arranges the computations in an interpretable manner. The surprising fact that a constant $\mathbf{1}$ input is given to weight network of DLGN is justified by theory and experiments in \Cref{sec:analysis,sec:dlgn}.}}
\label{fig:lgln}
\end{figure}
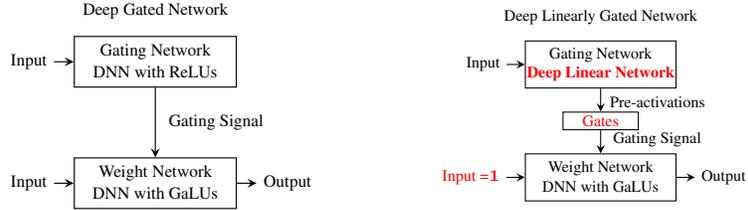

\textbf{Our Contribution.}  We \emph{extend the dual view} to address `black box'-ness by completely \emph{disentangling} the  `gating network' and the `weight network'. Our contributions are listed below.

$\bullet$ \textbf{Disentangling Gating Network.} For this, we propose a novel Deep \emph{Linearly} Gated Network (DLGN) as a \emph{mathematically interpretable} counterpart of a DNN with ReLUs (see \Cref{fig:lgln}).  In a DLGN, the gating network is a deep linear network, i.e., there is disentanglement because of the absence of non-linear activations. The gating network is mathematically interpretable, because, the transformations from input to the pre-activations are entirely linear; we call this \textbf{primal linearity}.

$\bullet$ \textbf{Dual View} (\Cref{sec:analysis}).  We present an unnoticed insight in prior work on fully connected networks that the \textbf{NPK is a product kernel} and is invariant to layer permutations. We present new results to show that (i) the \textbf{NPK is rotationally invariant} for convolutional networks with global average pooling, and (ii) the \textbf{NPK is an ensemble of many kernels} in the presence of skip connections. 

$\bullet$ \textbf{Disentangling Weight Network.} We then argue via theory and experiments that the  \emph{weight network is disentangled in the path space}, i.e., it learns path-by-path and not layer-by-layer. For this, in \Cref{sec:dlgn} we show via experiments that destroying the layer-by-layer structure by permuting the layers and providing a constant $\mathbf{1}$ as input (see DLGN in \Cref{fig:lgln}) do not  degrade performance. These counter intuitive results are difficult to reconcile using the commonly held `sophisticated structures are learnt in layers' interpretation. However, these experimental results follow from the theory in \Cref{sec:analysis}.  In other words, it is useful to think that the learning in the weight network happens path-by-path; we call this \textbf{dual linearity}, which (for infinite width) is interpreted via the NPK.

\textbf{Message.} The DLGN is not an alternative architecture per se, but a disentanglement and an interpretable re-arrangement of the computations in a DNN with ReLUs. The DLGN disentangles the computations into two  `mathematically' interpretable linearities (i) the `primal' linearity and (ii) the `dual' linearity interpreted via the NPK. Using the facts that the NPK is based on input and the gates, and in a DLGN, the pre-activations in the gating network are `primal' linear, we have complete disentanglement. We compare the performance of DNN, DGN and DLGN on CIFAR-10 and CIFAR-100 to show that, the \textbf{DLGN recovers more than $83.5\%$ of the performance of state-of-the-art DNNs}, i.e., while entanglement in the DNNs enable their improved performance,  the `disentangled and interpretable'  computations in the DLGN recovers most part of the performance. This brings us to an interesting question: `\textbf{Is DLGN a universal spectral approximator?}' (see \Cref{sec:dlgn}).

\subsection{Related Works.} 
We now compare our work with the related works.

$\bullet$ \textbf{Kernels.} Several works have examined theoretically as well as empirically two important kernels associated with a DNN namely its NTK based on the correlation of the gradients and the conjugate kernel based on the correlation of the outputs \citep{spectra,laplace,belkin,genntk,disentangling,ntk,arora2019exact,convgp,fcgp,lee2020finite}. In contrast, the NPK is based on the correlation of the gates. We do not build pure-kernel method with NPK, but use it as an aid to disentangle finite width DNN with ReLUs.  

$\bullet$ \textbf{ReLU, Gating, Dual Linearity.} A spline theory based on max-affine linearity was proposed in \citep{balestriero2018spline,balestriero2018hard} to show that a DNN with ReLUs performs hierarchical, greedy template matching. In contrast, the dual view exploits the gating property to simplify the NTK into the NPK. Gated linearity was studied in \citep{sss} for single layered networks, along with a non-gradient algorithm to tune the gates. In contrast, we look at networks of any depth, and the gates are tuned via standard optimisers. The main novelty in our work in contrast to the above is that in DLGN the feature generation is linear. The gating in this paper refers to the gating property of the ReLU itself and has no connection to \citep{highway}  where gating is a mechanism to regulate information flow. Also, the soft-gating used in our work and in \citep{npk} enables gradient flow via the gating network and is different from \emph{Swish} \citep{swish}, which is the multiplication of pre-activation and sigmoid.

$\bullet$ \textbf{Finite vs Infinite Width.} \cite{finitevsinfinite} perform an extensive comparison of finite versus infinite width DNNs. An aspect that is absent in their work, but present in the dual view is the disentanglement of gates and weights, and the fact that the learning in gates is crucial for finite width network to outperform infinite width DNNs. In our paper, we make use of theory developed for infinite width DNNs to provide empirical insights into inner workings of finite width networks.

$\bullet$ \textbf{Capacity.} Our experiments on destruction of layers, and providing constant $\mathbf{1}$ input are direct consequences of the insights from dual view theory. These are not explained by mere capacity based studies showing  DNNs are powerful to fit even random labelling of datasets \citep{ben}.

\section{Prior Work : Neural Tangent Kernel and Dual View}\label{sec:prelim}
In this section, we will focus on the dual view \citep{npk} and how the dual view helps to address the open question in the NTK theory. We begin with a brief summary of NTK.

\subsection{Infinite Width DNN  = Kernel Method With Neural Tangent Kernel}\label{sec:ntk}
An important kernel associated with a DNN is its \emph{neural tangent kernel} (NTK), which, for a pair of input examples $x,x'\in\R^{\din}$, and network weights $\Theta\in\R^{\dnet}$, is given by:

{\centering  $\text{NTK}(x,x')\quad = \quad \ip{\nabla_{\Theta}\hat{y}(x), \nabla_{\Theta}\hat{y}(x')}$,\quad\text{where}\par}
%\begin{align*}
 %\text{NTK}(x,x')\quad = \quad \ip{\nabla_{\Theta}\hat{y}(x), \nabla_{\Theta}\hat{y}(x')}, \quad\text{where}
%\end{align*}
$\hat{y}_\Theta(\cdot)\in\R$ is the DNN output. Prior works \citep{ntk,arora2019exact,cao2019generalization} have shown that, as the width of the DNN goes to infinity, the NTK matrix converges to a limiting deterministic matrix $\text{NTK}_{\infty}$, and training an infinitely wide DNN is equivalent to a kernel method with $\text{NTK}_{\infty}$.  While, as a pure kernel $\text{NTK}_{\infty}$ performed better than prior kernels by more than $10\%$, \cite{arora2019exact} observed that on CIFAR-10:

{\centering CNTK-GAP: $\mathbf{77.43\%}\leq $ CNN-GAP: $\mathbf{83.30\%}$\par}

where, CNN-GAP is a convolutional neural network with global average pooling and CNTK-GAP is its corresponding $\text{NTK}_{\infty}$ matrix. Due to this performance gap of about $5-6\%$, they concluded that $\text{NTK}_{\infty}$ does not explain fully the success of DNNs, and explaining this gap was an \textbf{open question}.%left the characterisation of this performance gap as an interesting \textbf{open question}. 

\subsection{Dual View For DNNs with ReLUs: Characterising the role of gates}

In the dual view, the computations are broken down path-by-path. The input and the gates (in each path) are encoded in a neural path feature vector and the weights (in each path) are encoded in a neural path value vector, and the output is the inner product of these two vectors.  
The learning in the gates and the learning in the weights are separated in a deep gated network (DGN) setup, which leads to the two main results of dual view presented in \Cref{sec:fixedgates} and \Cref{sec:gatelearning}, wherein, the neural path kernel, the Gram matrix of the neural path features will play a key role. 

\subsubsection{Neural Path Feature, Neural Path Value and Neural Path Kernel}
Consider a fully connected DNN with `$d$' layers and `$w$' hidden units in each layer. Let the DNN accept input $x\in \R^{\din}$ and produce an output $\hat{y}_{\Theta}(x)\in\R$. 
\begin{definition}\label{def:npf-npv}
A path starts from an input node, passes through a weight and a hidden unit in each layer and ends at the output node. We define the following quantities for a path $p$:
\begin{comment}
\begin{tabular}{lccl}
 Activity&:& $A_{\Theta}(x,p)$&=$\quad\Pi_{l=1}^{d-1} G_l\left(x,\I_l(p)\right)$.\\
Value&:& $v_{\Theta}(p)$&=$\quad\Pi_{l=1}^d\Theta\left(l,\I_{l-1}(p),\I_{l}(p)\right)$.\\
Feature&:&   $\phi_{\Theta}(x,p)$&=$\quad x\left(\I_0(p)\right)A_{\Theta}(x,p)$.
\end{tabular}
\end{comment}
\emph{
\begin{tabular}{lcl}
 Activity&:& $A_{\Theta}(x,p)$ is the product of the `$d-1$' gates in the path. \\
Value&:& $v_{\Theta}(p)$ is the product of the `$d$' weights in the path.\\
Feature&:&   $\phi_{\Theta}(x,p)$ is the product of the signal at the input node of the path and $A_{\Theta}(x,p)$.\\
\end{tabular}
}
The \emph{neural path feature} (NPF) given by $\phi_{\Theta}(x)=\left(\phi_{\Theta}(x,p),p=1,\ldots, \Pfc\right),\in\R^{\Pfc}$ and the \emph{neural path value} (NPV) given by $v_{\Theta}=\left(v_{\Theta}(p),p=1,\ldots,\Pfc\right),\in\R^{\Pfc}$, where $\Pfc=\din w^{(d-1)}$ is the total number of paths. 
\end{definition}
\begin{proposition}\label{prop:npf-npv}
The output of the DNN is then the inner product of the NPF and NPV: 
\begin{align}\label{eq:inner}
\hat{y}_{\Theta}(x)=\ip{\phi_{\Theta}(x),v_{\Theta}}=\sum_{p\in[P]}  \phi_{\Theta}(x,p) v_{\Theta}(p)
\end{align}
\end{proposition}
\textbf{Subnetwork Interpretation of DNNs with ReLUs.} A path is active only if all the gates in the path are active. This gives a subnetwork interpretation, i.e., for a given input $x\in\R^{\din}$, only a subset of the gates and consequently only a subset of the paths are active, and the input to output computation can be seen to be produced by this active subnetwork. The following matrix captures the correlation of the active subnetworks for a given pair of inputs $x,x'\in\R^{\din}$.
\begin{definition}[Overlap of active sub-networks]\label{def:overlap} 
The total number of `active' paths for both $x$ and $x'$ that pass through input node $i$ is defined to be:

{\centering{$\textbf{overlap}_{\Theta}(i,x,x') \eqdef {\left|\{p \colon  p\,\text{starts at node}\, i \,, A_{\Theta}(x,p)= A_{\Theta}(x',p)=1\} \right|}$}\par}
\end{definition}
\begin{lemma}[Neural Path Kernel (NPK)]\label{lm:npk}
Let $D\in\R^{\din}$ be a vector of non-negative entries  and for $u,u'\in\R^{\din}$ , let $\ip{u,u'}_{D}=\sum_{i=1}^{\din}D(i)u(i)u'(i)$. Then the neural path kernel (NPK) is given by: 
\begin{align*} 
\text{NPK}_{\Theta}(x,x')\eqdef \ip{\phi_{\Theta}(x),\phi_{\Theta}(x')}= \ip{x,x'}_{\textbf{overlap}_{\Theta}(\cdot,x,x')} 
\end{align*}
\end{lemma}

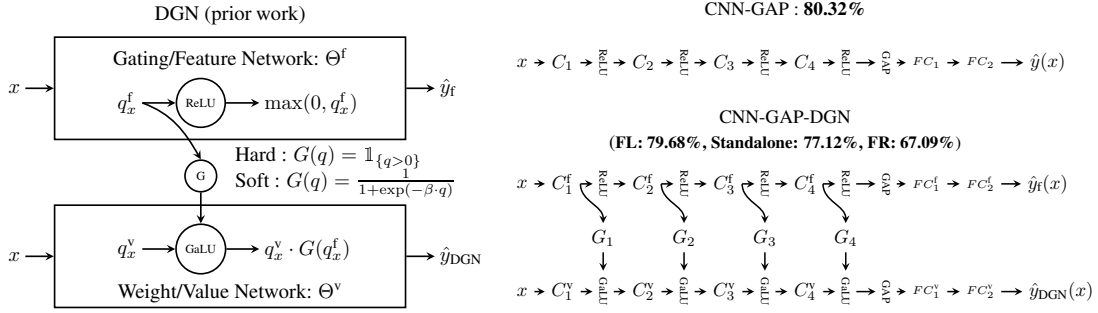
\begin{figure}
\centering
\begin{minipage}{0.9\columnwidth}
\centering
\begin{minipage}{0.45\columnwidth}
\resizebox{1.0\columnwidth}{!}{
\begin{tikzpicture}

\node []  (fntext)at (3.5,1.5) {DGN (prior work)};
%Feature Network
\node [draw,
	minimum width=6cm,
	minimum height=1.75cm,
	thick
]  (fnbox)at (3.5,0.25) {};
\node []  (fntext)at (3.5,0.75) {Gating/Feature Network: $\Tf$};

%Feature Network Input
\node (fin) [left of=fnbox,node distance=3.5cm, coordinate] {};
\node[left=-1pt] at (fin.west){$x$};
\draw[-stealth, thick] (fin.center) -- (fnbox.west);

%Feature Network Output
\node (fout) [right of=fnbox,node distance=3.5cm, coordinate] {};
\node[right=-1pt] at (fout.west){$\hat{y}_{\text{f}}$};
\draw[-stealth, thick]  (fnbox.east)--(fout.center);

%ReLU Circle
\node[draw,
	circle,
	minimum size=0.75cm,thick,
] (relu) at (3,0){\tiny{ReLU}};
%ReLU Input
\node (b) [left of=relu,node distance=1cm, coordinate] {};

\node[left=-1pt] at (b.center){$q^\text{f}_x$};
\draw[-stealth, thick] (b.east) -- (relu.west);

%ReLU Output
\node (c) [right of=relu,node distance=1cm, coordinate] {};
\node[right=-1pt] at (c.center){$\max(0,q^\text{f}_x)$};
\draw[-stealth, thick] (relu.east) -- (c.west);

%Gating Circle
\node[draw,
	circle,
	minimum size=0.0625cm,thick,
] (gating) at (3,-1.25){\tiny{G}};
%\node[right=6pt] at (gating.north){Hard : $G(q)=\mathbbm{1}_{\{q>0\}} $};
%\node[below right=6pt] at (gating.north){Soft : $G(q)=\frac{1}{1+\exp(-\beta\cdot q)}$};

\node[right=6pt] at (3.25,-0.925){Hard : $G(q)=\mathbbm{1}_{\{q>0\}} $};
\node[right=6pt] at (3.25,-1.375){Soft : $G(q)=\frac{1}{1+\exp(-\beta\cdot q)}$};

%Value Network

\node [draw,
	minimum width=6cm,
	minimum height=1.75cm,
	thick
]  (vnbox)at (3.5,-2.625) {};

\node []  (vntext)at (3.5,-3.25) {Weight/Value Network: $\Tv$};

%Value Network Input
\node (vin) [left of=vnbox,node distance=3.5cm, coordinate] {};
\node[left=-1pt] at (vin.west){$x$};
\draw[-stealth, thick] (vin.center) -- (vnbox.west);

%Feature Network Input
\node (vout) [right of=vnbox,node distance=3.5cm, coordinate] {};
\node[right=-1pt] at (vout.west){$\hat{y}_{\text{DGN}}$};
\draw[-stealth, thick]  (vnbox.east)--(vout.center);

%GaLU Circle
\node[draw,
	circle,
	minimum size=0.75cm,thick,
] (galu) at (3,-2.5){\tiny{GaLU}};

\draw [-stealth,thick]   (b) to[out=-20,in=120] (gating.north);
\draw [-stealth,thick]   (gating.south) -- (galu.north);

%GaLU Input
\node (d) [left of=galu,node distance=1cm, coordinate] {};
\node[left=-1pt] at (d.center){$q^\text{v}_x$};
\draw[-stealth, thick] (d.east) -- (galu.west);
%GaLU Output
\node (e) [right of=galu,node distance=1cm, coordinate] {};
\node[right=-1pt] at (e.center){$q^\text{v}_x\cdot G(q^\text{f}_x)$};
\draw[-stealth, thick] (galu.east) -- (e.west);
	
\end{tikzpicture}
}
\end{minipage}
\begin{minipage}{0.54\columnwidth}
\resizebox{1.0\columnwidth}{!}{
\begin{tikzpicture}
\node []  (dnn-text)at (5.375,2) {CNN-GAP : \textbf{80.32\%}};

\node []  (dnn-output) at (10.25,1) {$\hat{y}(x)$};
\node []  (dnn-fc2) at (9.0,1) {\tiny{$FC_2$}};
\draw [-stealth,thick]   (dnn-fc2.east) -- (dnn-output.west);

\node []  (dnn-fc1) at (8,1) {\tiny{$FC_1$}};
\draw [-stealth,thick]   (dnn-fc1.east) -- (dnn-fc2.west);

\node [rotate=-90]  (dnn-gap) at (7.25,1) {\tiny{GAP}};
\draw [-stealth,thick]   (dnn-gap.north) -- (dnn-fc1.west);

\node [rotate=-90] (dnn-relu-4) at (6.5,1){\tiny{ReLU}};
\node [] (dnn-c4) at (5.75,1){{$C_4$}};
\draw [-stealth,thick]   (dnn-c4.east) -- (dnn-relu-4.south);
\draw [-stealth,thick]   (dnn-relu-4.north) -- (dnn-gap.south);

\node [rotate=-90] (dnn-relu-3) at (5,1){\tiny{ReLU}};
\node [] (dnn-c3) at (4.25,1){{$C_3$}};
\draw [-stealth,thick]   (dnn-c3.east) -- (dnn-relu-3.south);
\draw [-stealth,thick]   (dnn-relu-3.north) -- (dnn-c4.west);

\node [rotate=-90] (dnn-relu-2) at (3.5,1){\tiny{ReLU}};
\node [] (dnn-c2) at (2.75,1){{$C_2$}};
\draw [-stealth,thick]   (dnn-c2.east) -- (dnn-relu-2.south);
\draw [-stealth,thick]   (dnn-relu-2.north) -- (dnn-c3.west);

\node [rotate=-90] (dnn-relu-1) at (2,1){\tiny{ReLU}};
\node [] (dnn-c1) at (1.25,1){{$C_1$}};
\draw [-stealth,thick]   (dnn-c1.east) -- (dnn-relu-1.south);
\draw [-stealth,thick]   (dnn-relu-1.north) -- (dnn-c2.west);

\node [] (dnn-input) at (0.5,1){$x$};
\draw [-stealth,thick]   (dnn-input.east) -- (dnn-c1.west);

%%%%%%%%%%%%%%%%%%%%%%%%%%%%%%%%%%%%%%%%%%%%%%%%%%%%%%%%%%%%%%%%%
\node []  (fntext)at (5.375,0) {CNN-GAP-DGN};
\node []  (fntext)at (5.375,-.5) {(\small{\textbf{FL: 79.68\%, Standalone: 77.12\%, FR: 67.09\%}})};

%\node []  (output) at (7.5,1.5) {$\hat{y}(x)$};

\node [align=right]  (dgn-f-output) at (10.25,-1.25) {$\hat{y}_{\text{f}}(x)$};
\node []  (dgn-f-fc2) at (9.0,-1.25) {\tiny{$FC^{\text{f}}_2$}};
\draw [-stealth,thick]   (dgn-f-fc2.east) -- (dgn-f-output.west);

\node []  (dgn-f-fc1) at (8,-1.25) {\tiny{$FC^{\text{f}}_1$}};
\draw [-stealth,thick]   (dgn-f-fc1.east) -- (dgn-f-fc2.west);

\node [rotate=-90]  (dgn-f-gap) at (7.25,-1.25) {\tiny{GAP}};
\draw [-stealth,thick]   (dgn-f-gap.north) -- (dgn-f-fc1.west);

\node [rotate=-90] (dgn-relu-4) at (6.5,-1.25){\tiny{ReLU}};
\node [] (dgn-f-c4) at (5.75,-1.25){{$C^{\text{f}}_4$}};
\draw [-stealth,thick]   (dgn-f-c4.east) -- (dgn-relu-4.south);
\draw [-stealth,thick]   (dgn-relu-4.north) -- (dgn-f-gap.south);

\node [rotate=-90] (dgn-relu-3) at (5,-1.25){\tiny{ReLU}};
\node [] (dgn-f-c3) at (4.25,-1.25){{$C^{\text{f}}_3$}};
\draw [-stealth,thick]   (dgn-f-c3.east) -- (dgn-relu-3.south);
\draw [-stealth,thick]   (dgn-relu-3.north) -- (dgn-f-c4.west);

\node [rotate=-90] (dgn-relu-2) at (3.5,-1.25){\tiny{ReLU}};
\node [] (dgn-f-c2) at (2.75,-1.25){{$C^{\text{f}}_2$}};
\draw [-stealth,thick]   (dgn-f-c2.east) -- (dgn-relu-2.south);
\draw [-stealth,thick]   (dgn-relu-2.north) -- (dgn-f-c3.west);

\node [rotate=-90] (dgn-relu-1) at (2,-1.25){\tiny{ReLU}};
\node [] (dgn-f-c1) at (1.25,-1.25){{$C^{\text{f}}_1$}};
\draw [-stealth,thick]   (dgn-f-c1.east) -- (dgn-relu-1.south);
\draw [-stealth,thick]   (dgn-relu-1.north) -- (dgn-f-c2.west);

\node [] (dgn-f-input) at (0.5,-1.25){$x$};
\draw [-stealth,thick]   (dgn-f-input.east) -- (dgn-f-c1.west);

\node [align=right]  (dgn-output) at (10.5,-3.25) {$\hat{y}_{\text{DGN}}(x)$};

\node [] (dgn-fc2) at (9,-3.25){\tiny{$FC^{\text{v}}_2$}};
\draw [-stealth,thick]   (dgn-fc2.east)--(dgn-output.west);

\node [] (dgn-fc1) at (8,-3.25){\tiny{$FC^{\text{v}}_1$}};
\draw [-stealth,thick]   (dgn-fc1.east)--(dgn-fc2.west);

\node [rotate=-90] (dgn-gap) at (7.25,-3.25){\tiny{GAP}};
\draw [-stealth,thick]   (dgn-gap.north)--(dgn-fc1.west);

\node [rotate=-90] (dgn-galu-4) at (6.5,-3.25){\tiny{GaLU}};
\draw [-stealth,thick]   (dgn-galu-4.north) -- (dgn-gap.south);

\node [] (dgn-v-c4) at (5.75,-3.25){{$C^{\text{v}}_4$}};
\draw [-stealth,thick]   (dgn-v-c4.east) -- (dgn-galu-4.south);

\node [rotate=-90] (dgn-galu-3) at (5,-3.25){\tiny{GaLU}};
\node [] (dgn-v-c3) at (4.25,-3.25){{$C^{\text{v}}_3$}};
\draw [-stealth,thick]   (dgn-v-c3.east) -- (dgn-galu-3.south);
\draw [-stealth,thick]   (dgn-galu-3.north) -- (dgn-v-c4.west);

\node [rotate=-90] (dgn-galu-2) at (3.5,-3.25){\tiny{GaLU}};
\node [] (dgn-v-c2) at (2.75,-3.25){{$C^{\text{v}}_2$}};
\draw [-stealth,thick]   (dgn-v-c2.east) -- (dgn-galu-2.south);
\draw [-stealth,thick]   (dgn-galu-2.north) -- (dgn-v-c3.west);

\node [rotate=-90] (dgn-galu-1) at (2,-3.25){\tiny{GaLU}};
\node [] (dgn-v-c1) at (1.25,-3.25){{$C^{\text{v}}_1$}};

\draw [-stealth,thick]   (dgn-v-c1.east) -- (dgn-galu-1.south);
\draw [-stealth,thick]   (dgn-galu-1.north) -- (dgn-v-c2.west);

\node [] (dgn-input) at (0.5,-3.25){$x$};
\draw [-stealth,thick]   (dgn-input.east) -- (dgn-v-c1.west);

\node[] (dgn-gating-1) at (2,-2.25){$G_1$};
\draw [-stealth,thick]   (dgn-f-c1.east) to[out=-90,in=90] (dgn-gating-1.north);
\draw [-stealth,thick]   (dgn-gating-1.south) -- (dgn-galu-1.west);

\node[] (dgn-gating-2) at (3.5,-2.25){$G_2$};
\draw [-stealth,thick]   (dgn-f-c2.east) to[out=-90,in=90] (dgn-gating-2.north);
\draw [-stealth,thick]   (dgn-gating-2.south) -- (dgn-galu-2.west);

\node[] (dgn-gating-3) at (5,-2.25){$G_3$};
\draw [-stealth,thick]   (dgn-f-c3.east) to[out=-90,in=90] (dgn-gating-3.north);
\draw [-stealth,thick]   (dgn-gating-3.south) -- (dgn-galu-3.west);

\node[] (dgn-gating-4) at (6.5,-2.25){$G_4$};
\draw [-stealth,thick]   (dgn-f-c4.east) to[out=-90,in=90] (dgn-gating-4.north);
\draw [-stealth,thick]   (dgn-gating-4.south) -- (dgn-galu-4.west);

\end{tikzpicture}
}
\end{minipage}
\end{minipage}
\caption{\small{Shows the DGN on the left. \textbf{Training:} In the case of fixed learnt gates, the feature network is pre-trained using $\hat{y}_f$ as the output, and then the feature network is frozen, and the value network is trained with $\hat{y}_{DGN}$ as the output. In the case of fixed random gates, the feature network is initialised at random and frozen, and the value network is trained with $\hat{y}_{DGN}$ as the output. In the case of fixed gates, hard gating $G(q)=\mathbbm{1}_{\{q>0\}}$ is used. \textbf{Standalone Training:} both feature and value network are initialised at random and trained together with $\hat{y}_{DGN}$ as the output. Here, soft gating $G(q)=\frac{1}{1+\exp(-\beta\cdot q)}$ is used to allow gradient flow through feature network. On the right side is the CNN-GAP and its DGN used in \citep{npk}. In CNN-GAP, $C_1,C_2,C_3,C_4$ are convolutional layers, $FC_1,FC_2$ are fully connected layers. In CNN-GAP-DGN, $G_l, l=1,2,3,4$ are the gates of layers, the superscripts f, and v stand for feature and value network respectively.}}
\label{fig:dgn}
\end{figure}

\subsubsection{Deep Gated Network: Separating Gates and Weights}
{Deep Gated Network (DGN)} is a setup to separate the gates from the weights. Consider a DNN with ReLUs with weights $\Theta\in\R^{\dnet}$. The DGN \emph{corresponding} to this DNN (left diagram in \Cref{fig:dgn}) has two networks of \emph{identical architecture} (to the DNN) namely the `gating network' and the `weight network' with distinct weights $\Tf\in\R^{\dnet}$ and $\Tv\in\R^{\dnet}$.  

The `gating network' has ReLUs which turn `on/off' based on their pre-activation signals, and the `weight network' has gated linear units (GaLUs) \citep{sss,npk}, which multiply their respective pre-activation inputs by the external gating signals provided by the `gating network'.  Since both the networks have identical architecture, the ReLUs and GaLUs in the respective networks have a one-to-one correspondence.  Gating network realises $\phi_{\Tf}(x)$ by turning `on/off' the corresponding GaLUs in the weight network. The weight network realises $v_{\Tv}$ and computes the output $\hat{y}_{\text{DGN}}(x)=\ip{\phi_{\Tf}(x),v_{\Tv}}$.  The gating network is also called as the feature network since it realises the neural path features, and the weight network is also called as the value network since it realises the neural path value.

\subsubsection{Learning Weights With Fixed Gates = Neural Path Kernel}\label{sec:fixedgates}
During training, a DNN learns both $\phi_{\Theta}(x)$ as well as $v_{\Theta}$ simultaneously, and a finite time characterisation of this learning in finite width DNNs is desirable. However, this is a hard problem. An easier problem is to understand in a DGN, how the weights in the value network are learnt when the gates are fixed in the feature network, i.e., how $\hat{y}_{DGN}(x)=\ip{\phi_{\Tf}(x),v_{\Tv}}$ is learnt by learning $v_{\Tv}$ with fixed $\phi_{\Tf}(x)$. While  $\hat{y}_{DGN}(x)=\ip{\phi_{\Tf}(x),v_{\Tv}}$ is linear in the dual variables, it is still non-linear in the value network weights $\Tv$. However, \cite{npk} showed that the dual linearity is characterised by the NPK in the infinite width regime. We state the assumption followed by Theorem 5.1 in \citep{npk}, wherein, the $\text{NTK}(x,x')=\ip{\nabla_{\Tv}\hat{y}_{\Tdgn_0}(x),\nabla_{\Tv}\hat{y}_{\Tdgn_0}(x') }$ is due to the gradient of $\hat{y}_{\text{DGN}}$ with respect to the value network weights, and the $\text{NPK}(x,x')=\ip{\phi_{\Tf_0}(x),\phi_{\Tf_0}(x')}$ is due to the feature network weights.

\begin{assumption}\label{assmp:main}
$\Tv_0\stackrel{\text{i.i.d}}\sim\text{Bernoulli}(\frac12)$ over $\{-{\sigma},+{\sigma}\}$ and statistically independent of $\Tf_0$.
\end{assumption}

\begin{theorem}[Theorem 5.1 in \citep{npk}]
\label{th:fcprev} Under \Cref{assmp:main} for a fully connected DGN : 
\begin{align*}
\text{NTK}^{\texttt{FC}}(x,x')\,\, \stackrel{(a)}\rightarrow\,\, &d \cdot \sigma^{2(d-1)} \cdot \text{NPK}^{\texttt{FC}}(x,x'), \quad\text{as}\,\, w\rightarrow \infty \\
					     & = d \cdot \sigma^{2(d-1)}\cdot\ip{x,x'}\cdot \textbf{overlap}(x,x')
\end{align*}
\end{theorem} 
\textbf{Remark.} In the fully connected case, $\textbf{overlap}(i,x,x')$ is identical for all $i=1,\ldots,\din$, and hence $\ip{x,x'}_{\textbf{overlap}(\cdot,x,x')}$ in \Cref{lm:npk} becomes $\ip{x,x'}\cdot \textbf{overlap}(x,x')$ in \Cref{th:fcprev}. It follows from NTK theory that an infinite width DGN with fixed gates is equivalent to kernel method with NPK.
\subsubsection{Learning in Gates Key For Finite Width To Be Better Than Infinite Width}\label{sec:gatelearning}
The fixed gates setting is an idealised setting, in that, it does not theoretically capture the learning of the gates, i.e., the neural path features $\phi_{\Theta}(x)$. However, the learning in the gates can be empirically characterised by comparing \textbf{fixed learnt (FL)}  gates coming from a pre-trained DNN and \textbf{fixed random (FR)} gates coming from randomly initialised DNN, and the infinite width NTK.  Using a CNN-GAP and its corresponding DGN, \cite{npk} showed on CIFAR-10 that (see \Cref{fig:dgn} for details on DGN training and the CNN-GAP architecture):

{\centering FR Gates : $\mathbf{67.09\%}\leq $ CNTK-GAP: $\mathbf{77.43\%}\leq $ FL Gates: $\mathbf{79.68\%}\approx $ CNN-GAP: $\mathbf{80.32\%}$\par}

based on which it was concluded that learning in the gates (i.e., neural path features) is crucial for finite width CNN-GAP to outperform the infinite width CNTK-GAP. It was also shown that the DGN can be trained \textbf{standalone} (as shown in \Cref{fig:dgn}) and is only \emph{marginally poor} to the DNN. 

\section{Deep Linearly Gated Networks: Complete Disentanglement}
{\centering $\text{DLGN} : x\rightarrow \stackrel{\text{Primal} }{\text{Linear}}\rightarrow \text{Pre-activations}\rightarrow\text{Gates}\stackrel{\text{lifting}}{\rightarrow} \phi_{\Tf}(x) {\rightarrow}\stackrel{\text{Dual} }{\text{Linear}}: \hat{y}(x)=\ip{\phi_{\Tf}(x),v_{\Tv}}$ \par}

The deep linearly gated network (DLGN) has two `mathematically' interpretable linearities, the `primal' and the `dual' linearities. The primal linearity is ensured in via construction and needs no theoretical justification. Once the pre-activations triggers gates, $\phi_{\Tf}(x)$ gets realised in the value network by activating the paths.  Now, the value network itself is `dual' linear, i.e., it simply computes/learns the inner product $\hat{y}(x)=\ip{\phi_{\Tf}(x),v_{\Tv}}$. Gating \emph{lifts} the `primal' linear computations in the feature network to `dual' linear computations in the value network. Dual linearity is characterised by the NPK (for infinite width) which in turn depends on the input and gates, and  the fact that the pre-activations to the gates are primal linear implies complete disentanglement and interpretability. 

Dual linearity is mathematically evident due to the inner product relationship, however, adopting it has the following conceptual issue: it is a commonly held view that `sophisticated features are learnt in the layers', that is, given that the input $x\in\R^{\din}$  is presented to the value network (as in \Cref{fig:dgn}), it could be argued that the GaLUs and linear operations are entangled which in turn enable learning of sophisticated features in the layers. In what follows, we demystify this layer-by-layer view via theory (infinite width case) in \Cref{sec:analysis} , and experiments (on finite width networks) in \Cref{sec:dlgn}, and then study the performance of DLGN in \Cref{sec:dlgn}. The layer-by-layer view is demystified by showing that (i) a constant $\mathbf{1}$ input can be given to the value network, (ii) layer-by-layer structure can be destroyed. The constant $\mathbf{1}$ input is meant to show that if the input is not given to the value network then it is not possible to learn sophisticated structures `from the input' in a layer-by-layer manner.  In terms of the dual linearity, providing a constant $\mathbf{1}$ input has only a minor impact, in that, the neural path feature becomes $\phi(x,p)=1\cdot A(x,p)$, i.e., it still encodes the path activity which is still input dependent. Since $\phi(x)$ depends only on gates, the NPK will depend only on the $\textbf{overlap}$ matrix; results in \Cref{sec:analysis}  captures this in theory. Now, it could be argued that, despite a constant $\mathbf{1}$ input, the gates are still arranged layer-by-layer, due to which, the value network is still able to learn sophisticated structures in its layers. \Cref{sec:analysis} has theory that points out that as long as the \textbf{correlation of the gates} is not lost, the layer-by-layer structure can be destroyed.

\subsection{Dual Linearity: New Insights and New Results}\label{sec:analysis}
We saw in \Cref{sec:fixedgates} that dual linearity is characterised by the NPK for infinite width case. In this section, we: (i) cover standard architectural choices namely convolutions with global-average-pooling and skip connections in \Cref{th:conv,th:res}; the prior result \Cref{th:fcprev} is only for the fully connected case, (ii)  present new insights on \Cref{th:fcprev} by restating it explicitly in terms of the gates in \Cref{th:fc}, and (iii) discuss how the NPK structure helps in demystifying the layer-by-layer view. Note: Results in this section are about the value network and hold for both DGN and DLGN.

\subsubsection{Fully Connected: Product of LayerWise Base Kernels}
\begin{theorem}
\label{th:fc} Let $G_l(x)\in[0,1]^w$ denote the gates in layer $l\in\{1,\ldots,d-1\}$ for input $x\in\R^{\din}$. Under \Cref{assmp:main}  ($\sigma=\frac{\cscale}{\sqrt{w}}$) as $w\rightarrow \infty $, we have for fully connected DGN/DLGN:
\begin{align*}
\text{NTK}^{\texttt{FC}}(x,x') \rightarrow d \cdot \sigma^{2(d-1)} \cdot \text{NPK}^{\texttt{FC}}(x,x') = d \cdot \cscale^{2(d-1)} \cdot \left(\ip{ x,x'} \cdot \Pi_{l=1}^{d-1} \frac{\ip{G_l(x),G_l(x')}}w\right),
\end{align*}
\end{theorem} 
$\bullet$ \textbf{Product Kernel : Role of Depth and Width.} \Cref{th:fc} is mathematically equivalent to  \Cref{th:fcprev}, which follows from the observation that $\textbf{overlap}(x,x')=\Pi_{l=1}^{(d-1)}\ip{G_l(x),G_l(x')}$. While this observation is very elementary in itself, it is significant at the same time;  \Cref{th:fc} provides the most simplest kernel expression that characterises the information in the gates. From \Cref{th:fc} it is evident that the role of width is \emph{averaging} (due to the division by $w$). Each layer therefore corresponds to a \emph{base kernel} $\frac{\ip{G_l(x),G_l(x')}}w$ which measures the \emph{\textbf{correlation of the gates}}. The role of depth is to provide the product of kernels. To elaborate, the feature network provides the gates $G_l(x)$, and the value network realises the product kernel in \Cref{th:fc} by laying out the GaLUs depth-wise, and connecting them to form a deep network. The depth-wise layout is important: for instance, if we were to concatenate the gating features as $\varphi(x)=(G_l(x),l=1,\ldots,d-1)\in\{0,1\}^{(d-1)w}$, it would have only resulted in the kernel $\ip{\varphi(x),\varphi(x')}=\sum_{l=1}^{d-1}{\ip{G_l(x),G_l(x')}}$, i.e., a \emph{sum  (not product)} of kernels. 

$\bullet$ \textbf{Constant $\mathbf{1}$ Input.} This has a minor impact, in that, the expression on right hand side of \Cref{th:fc} becomes $d \cdot \cscale^{2(d-1)} \cdot \din \cdot \Pi_{l=1}^{d-1} \frac{\ip{G_l(x),G_l(x')}}w$, i.e., the kernel still has information of the gates.

$\bullet$ \textbf{Destroying structure by permuting the layers.}  $\Pi_{l=1}^{d-1} \frac{\ip{G_l(x),G_l(x')}}w$ is permutation invariant, and hence permuting the layers has no effect.

\subsubsection{Convolution  Global Average Pooling: Rotationally Invariant Kernel}
We consider networks with circular convolution and global average pooling (architecture and notations are in the Appendix). In \Cref{th:conv}, let the circular rotation of vector $x\in\R^{\din}$ by `$r$' co-ordinates be defined as $rot(x,r)(i)=x(i+ r)$, if $i+r \leq \din$ and $rot(x,r)(i)=x(i+ r-\din)$ if $i+r > \din$.  
\begin{comment}
 We extend the dual view to neural network with $\dc$ convolutional layers ($l=1,\ldots,\dc$), followed by a \emph{global-average/max-pooling} layer ($l=\dc+1$) and $\dfc$ ($l=\dc+2,\ldots,\dc+\dfc+1$) fully connected  layers (see Appendix for notation). The convolutional window size is $\wconv<\din$, the number of filters per convolutional layer as well as the width of the fully connected layers is $w$. The main steps are (i) treating pooling layers like gates/masks, (ii) bundling together the paths that share the same path value (due to weight sharing in convolutions) and (iii) re-defining the NPF and NPV for these bundles. The important consequence of weight sharing (due to convolutions and pooling) is that the NPK becomes rotationally invariant resulting in \Cref{th:mainconv}.
\end{comment}
\begin{theorem}\label{th:conv} Under \Cref{assmp:main}, for  a suitable $\bcnn$ (see Appendix for expansion of $\bcnn$):
\begin{align*}
\text{NTK}^{\texttt{CONV}}(x,x')\rightarrow  \frac{\bcnn}{{\din}^2} \cdot \sum_{r=0}^{\din-1} \ip{x,rot(x',r)}_{\textbf{overlap}(\cdot, x,rot(x',r))},\,\, \text{as}\,\,  w\rightarrow\infty\,
\end{align*}
\end{theorem}
$\bullet$ $\sum_{r=0}^{\din-1} \ip{x,rot(x',r)}_{\textbf{overlap}(\cdot, x,rot(x',r))}=\sum_{r=0}^{\din-1}\sum_{i=1}^{\din} x(i) rot(x',r)(i)\textbf{overlap}(i,x,rot(x',r))$, where the inner `$\Sigma$' is the inner product between $x$ and $rot(x',r)$ weighted by $\textbf{overlap}$ and the outer `$\Sigma$' covers all possible rotations, which in addition to the fact that all the variables internal to the network rotate as the input rotates, results in the rotational invariance.  It was observed by \cite{arora2019exact} that networks with global-average-pooling are better than vanilla convolutional networks. The rotational invariance holds for convolutional architectures only in the presence of global-pooling.  So, this result explains why global-average-pooling helps. That said, rotational invariance is not a new observation; it was shown by \cite{li2019enhanced} that  prediction using CNTK-GAP is equivalent to prediction using CNTK without GAP but with full translation data augmentation  (same as rotational invariance) with wrap-around at the boundary (same as circular convolution). However, \Cref{th:conv} is a necessary result, in that, it shows rotational invariance is recovered in the dual view as well. 

$\bullet$ The expression in \Cref{th:conv} becomes $\frac{\bcnn}{{\din}^2} \cdot \sum_{r=0}^{\din-1}\sum_{i=1}^{\din} \textbf{overlap}(i,x,rot(x',r))$  for a \textbf{constant $\mathbf{1}$ input}. The key novel insight is that the rotational invariance is not lost and $\textbf{overlap}$ matrix measures the correlation of the paths which in turn depends on the correlation of the gates.

$\bullet$ \textbf{Destroying structure by permuting the layers} does not destroy the rotational invariance in \Cref{th:conv}. This is because, due to circular convolutions all the internal variables of the network rotate as the input rotates. Permuting the layers only affects the ordering of the layers, and does not affect the fact that the gates rotate if the input rotates, and correlation in the gates is not lost.

\subsubsection{Residual Networks With Skip Connections (ResNet): Ensemble Of Kernels}
We consider a ResNet with `$(b+2)$' blocks and `$b$' skip connections between the blocks. Each block is a fully connected (FC) network of depth `$\dblock$' and width `$w$'. There are $2^b$ many sub-FCNs within this ResNet (see \Cref{def:subfcdnn}).
Note that the blocks being fully connected is for expository purposes, and the result continue to hold for any kind of block.

\begin{definition}\label{def:subfcdnn}[Sub FCNs]
Let $2^{[b]}$ denote the power set of $[b]$ and let $\J\in 2^{[b]}$ denote any subset of $[b]$. Define the`$\J^{th}$' sub-FCN of the ResNet to be the fully connected network obtained by (i) including  $\text{block}_{j},\forall j\in \J$  and (ii) ignoring $\text{block}_{j},\forall j\notin \J$. 
\end{definition}
\begin{theorem}\label{th:res} 
Let $\text{NPK}^{\texttt{FC}}_{\J}$ be the NPK of the $\J^{th}$ sub-FCN, and $\bfc^{\J}$ (see Appendix for expansion of $\bfc^{\J}$) be the associated constant. Under \Cref{assmp:main}, we have:
\begin{align*}
\text{NTK}^{\texttt{RES}}\rightarrow \sum_{\J\in 2^{[b]}}  \bfc^{\J} \text{NPK}^{\texttt{FC}}_{\J}, \,\, \text{as}\,\,  w\rightarrow\infty
\end{align*}
\end{theorem}

$\bullet$ \textbf{Ensemble.} To the best of our knowledge, this is the first theoretical result to show that ResNets have an ensemble structure, where  each kernel in the ensemble, i.e., $\text{NPK}^{\texttt{FC}}_{\J}$ corresponds to one of the $2^b$ sub-architectures (see \Cref{def:subfcdnn}). The ensemble behaviour of ResNet and  presence of $2^b$ architectures was observed by \cite{veit2016residual}, however without any concrete theoretical formalism. 

$\bullet$ Effect of \textbf{constant $\mathbf{1}$ input} is as before for kernels $\text{NPK}^{\texttt{FC}}_{\J}$ and translates to the ensemble $\text{NTK}^{\texttt{RES}}$. 

$\bullet$ \textbf{Destroying structure.} The ResNet inherits the invariances of the block level kernel. In addition, the ensemble structure allows to even remove layers.  \cite{veit2016residual} showed empirically that removing single layers from ResNets at test time does not noticeably affect their performance, and yet removing a layer from architecture such as VGG leads to a dramatic loss in performance. \Cref{th:res} can be seen to provide a theoretical justification for this empirical result. In other words, due to the ensemble structure a ResNet is capable of dealing with failure of components. While failure of component itself does not occur unless one makes them fail purposefully as done in \citep{veit2016residual},  the insight is that even if one or many of the kernels in the ensemble are corrupt and the good ones can compensate.

\subsection{Numerical Experiments}\label{sec:dlgn}
\begin{figure}[!b]
\centering
\begin{minipage}{0.11\columnwidth}
\centering
\resizebox{0.99\columnwidth}{!}{
\begin{tikzpicture}

%%%%%%%%%%%%%%%%%%%%%%%%%%%%%%%%%%%%%%%%%%%%%%%%%%%%%%%%%%%%%%%%%
%%%%%%%%%%%%%%%%%%%%%%%%%%%%%%
\node []  (fntext)at (0.5,2.55) {C4GAP};

\node [] (output) at (1,-1.25){$\hat{y}(x)$};

\node [] (fc) at (1,-0.25){\tiny{$FC$}};
\draw [-stealth,thick]  (fc.south) -- (output.north);
\node [] (gap) at (1,0.5){\tiny{GAP}};
\draw [-stealth,thick]  (gap.south) -- (fc.north);

\node [] (relu4) at (1,1.25){\tiny{ReLU}};
\draw [-stealth,thick]  (relu4.south) -- (gap.north);
\node [] (c4) at (1,2.0){\tiny{$C_4$}};
\draw [-stealth,thick]  (c4.south) -- (relu4.north);

\node [] (relu3) at (0,2.0){\tiny{ReLU}};
\draw [-stealth,thick]  (relu3.east) -- (c4.west);
\node [] (c3) at (0,1.25){\tiny{$C_3$}};
\draw [-stealth,thick]  (c3.north) -- (relu3.south);

\node [] (relu2) at (0,0.5){\tiny{ReLU}};
\draw [-stealth,thick]  (relu2.north) -- (c3.south);
\node [] (c2) at (0,-0.25){\tiny{$C_2$}};
\draw [-stealth,thick]  (c2.north) -- (relu2.south);

\node [] (relu1) at (0,-1.0){\tiny{ReLU}};
\draw [-stealth,thick]  (relu1.north) -- (c2.south);
\node [] (c1) at (0,-1.75){\tiny{$C_1$}};
\draw [-stealth,thick]  (c1.north) -- (relu1.south);

\node [] (x) at (0,-2.5){$x$};
\draw [-stealth,thick ]  (x.north) -- (c1.south);

%%%%%%%%%%%%%%%%%%%%%%%%%%%%%%%%%%%%%%%%%%%%%%%%%%%%%%%%%%%%%%%%%

\end{tikzpicture}
}
\end{minipage}
\begin{minipage}{0.40\columnwidth}
\centering
\resizebox{0.99\columnwidth}{!}{
\begin{tikzpicture}

%%%%%%%%%%%%%%%%%%%%%%%%%%%%%%%%%%%%%%%%%%%%%%%%%%%%%%%%%%%%%%%%%
\node []  (fntext)at (1.75+3,2.25) {C4GAP-DGN};

%\node []  (output) at (7.5,1.5) {$\hat{y}(x)$};

\node [] (dgn-f-c4) at (6.25,1.5){\tiny{$C^{\text{f}}_4$}};

\node [rotate=-90] (dgn-relu-3) at (5.5,1.5){\tiny{ReLU}};
\node [] (dgn-f-c3) at (4.75,1.5){\tiny{$C^{\text{f}}_3$}};
\draw [-stealth,thick]   (dgn-f-c3.east) -- (dgn-relu-3.south);
\draw [-stealth,thick]   (dgn-relu-3.north) -- (dgn-f-c4.west);

\node [rotate=-90] (dgn-relu-2) at (4.0,1.5){\tiny{ReLU}};
\node [] (dgn-f-c2) at (3.25,1.5){\tiny{$C^{\text{f}}_2$}};
\draw [-stealth,thick]   (dgn-f-c2.east) -- (dgn-relu-2.south);
\draw [-stealth,thick]   (dgn-relu-2.north) -- (dgn-f-c3.west);

\node [rotate=-90] (dgn-relu-1) at (2.5,1.5){\tiny{ReLU}};
\node [] (dgn-f-c1) at (1.75,1.5){\tiny{$C^{\text{f}}_1$}};
\draw [-stealth,thick]   (dgn-f-c1.east) -- (dgn-relu-1.south);
\draw [-stealth,thick]   (dgn-relu-1.north) -- (dgn-f-c2.west);

\node [] (dgn-f-input) at (1.75,0){$x^{\text{f}}$};
\draw [-stealth,thick]   (dgn-f-input.north) -- (dgn-f-c1.south);

\node []  (dgn-output) at (7.75,-0.5) {$\hat{y}(x)$};
\node [] (dgn-smax) at (7.75,-1.5){\tiny{FC}};
\draw [-stealth,thick]   (dgn-smax.north)--(dgn-output.south);

\node [rotate=-90] (dgn-gap) at (7.75,-2.5){\tiny{GAP}};
\draw [-stealth,thick]   (dgn-gap.west)--(dgn-smax.south);

\node [rotate=-90] (dgn-galu-4) at (7,-2.5){\tiny{GaLU}};
\draw [-stealth,thick]   (dgn-galu-4.north) -- (dgn-gap.south);

\node [] (dgn-v-c4) at (6.25,-2.5){\tiny{$C^{\text{v}}_4$}};
\draw [-stealth,thick]   (dgn-v-c4.east) -- (dgn-galu-4.south);

\node [rotate=-90] (dgn-galu-3) at (5.5,-2.5){\tiny{GaLU}};
\node [] (dgn-v-c3) at (4.75,-2.5){\tiny{$C^{\text{v}}_3$}};
\draw [-stealth,thick]   (dgn-v-c3.east) -- (dgn-galu-3.south);
\draw [-stealth,thick]   (dgn-galu-3.north) -- (dgn-v-c4.west);

\node [rotate=-90] (dgn-galu-2) at (4,-2.5){\tiny{GaLU}};
\node [] (dgn-v-c2) at (3.25,-2.5){\tiny{$C^{\text{v}}_2$}};
\draw [-stealth,thick]   (dgn-v-c2.east) -- (dgn-galu-2.south);
\draw [-stealth,thick]   (dgn-galu-2.north) -- (dgn-v-c3.west);

\node [rotate=-90] (dgn-galu-1) at (2.5,-2.5){\tiny{GaLU}};
\node [] (dgn-v-c1) at (1.75,-2.5){\tiny{$C^{\text{v}}_1$}};

\draw [-stealth,thick]   (dgn-v-c1.east) -- (dgn-galu-1.south);
\draw [-stealth,thick]   (dgn-galu-1.north) -- (dgn-v-c2.west);

\node [] (dgn-input) at (1.75,-1){$x^{\text{v}}$};
\draw [-stealth,thick]   (dgn-input.south) -- (dgn-v-c1.north);

\node[] (dgn-gating-1-up) at (2.5,0.5){\tiny{$G_{1}$}};
\draw [-stealth,thick]   (dgn-f-c1.east) to[out=-90,in=90] (dgn-gating-1-up.north);

\node[] (dgn-gating-2-up) at (4,0.5){\tiny{$G_{2}$}};
\draw [-stealth,thick]   (dgn-f-c2.east) to[out=-90,in=90] (dgn-gating-2-up.north);

\node[] (dgn-gating-3-up) at (5.5,0.5){\tiny{$G_{3}$}};
\draw [-stealth,thick]   (dgn-f-c3.east) to[out=-90,in=90] (dgn-gating-3-up.north);

\node[] (dgn-gating-4-up) at (7,0.5){\tiny{$G_{4}$}};
\draw [-stealth,thick]   (dgn-f-c4.east) to[out=-90,in=90] (dgn-gating-4-up.north);

\node[] (dgn-gating-1) at (2.5,-1.5){\tiny{$G_{i_1}$}};
\draw [-stealth,thick]   (dgn-gating-1.south) -- (dgn-galu-1.west);

\node[] (dgn-gating-2) at (4,-1.5){\tiny{$G_{i_2}$}};
\draw [-stealth,thick]   (dgn-gating-2.south) -- (dgn-galu-2.west);

\node[] (dgn-gating-3) at (5.5,-1.5){\tiny{$G_{i_3}$}};
\draw [-stealth,thick]   (dgn-gating-3.south) -- (dgn-galu-3.west);

\node[] (dgn-gating-4) at (7,-1.5){\tiny{$G_{i_4}$}};
\draw [-stealth,thick]   (dgn-gating-4.south) -- (dgn-galu-4.west);

\node[] (permutation) at (4.75,-0.5){\small{Layer Permutation}};

\draw [-]   (dgn-gating-1-up.south) -- (permutation.north);
\draw [-]   (dgn-gating-4-up.south) -- (permutation.north);
\draw [-]   (dgn-gating-2-up.south) -- (permutation.north);
\draw [-]   (dgn-gating-3-up.south) -- (permutation.north);

\draw [-]  (permutation.south) --  (dgn-gating-1.north)  ;
\draw [-]  (permutation.south) --  (dgn-gating-2.north)  ;
\draw [-]  (permutation.south) --  (dgn-gating-3.north)  ;
\draw [-]  (permutation.south) --  (dgn-gating-4.north)  ;
%%%%%%%%%%%%%%%%%%%%%%%%%%%%%%

%%%%%%%%%%%%%%%%%%%%%%%%%%%%%%%%%%%%%%%%%%%%%%%%%%%%%%%%%%%%%%%%%

\end{tikzpicture}
}
\end{minipage}
\begin{minipage}{0.40\columnwidth}
\centering
\resizebox{0.99\columnwidth}{!}{
\begin{tikzpicture}

%%%%%%%%%%%%%%%%%%%%%%%%%%%%%%%%%%%%%%%%%%%%%%%%%%%%%%%%%%%%%%%%%
%%%%%%%%%%%%%%%%%%%%%%%%%%%%%%
\node []  (fntext)at (-8.5+3,2.25) {C4GAP-DLGN};

%\node []  (output) at (7.5,1.5) {$\hat{y}(x)$};

\node [] (dgn1-f-c4) at (-4,1.5){\tiny{$C^{\text{f}}_4$}};
\node [] (dgn1-f-c3) at (-5.5,1.5){\tiny{$C^{\text{f}}_3$}};
\node [] (dgn1-f-c2) at (-7,1.5){\tiny{$C^{\text{f}}_2$}};
\node [] (dgn1-f-c1) at (-8.5,1.5){\tiny{$C^{\text{f}}_1$}};
\node [] (dgn1-input-f) at (-8.5,0){$x^{\text{f}}$};
\draw [-stealth,thick]   (dgn1-f-c3.east) -- (dgn1-f-c4.west);
\draw [-stealth,thick]   (dgn1-f-c2.east) -- (dgn1-f-c3.west);
\draw [-stealth,thick]   (dgn1-f-c1.east) -- (dgn1-f-c2.west);
\draw [-stealth,thick]   (dgn1-input-f.north) -- (dgn1-f-c1.south);

\node []  (dgn1-output) at (-2.5,-0.5) {$\hat{y}(x)$};

\node [] (dgn1-smax) at (-2.5,-1.5){\tiny{FC}};
\draw [-stealth,thick]   (dgn1-smax.north)--(dgn1-output.south);

\node [rotate=-90] (dgn1-gap) at (-2.5,-2.5){\tiny{GAP}};
\draw [-stealth,thick]   (dgn1-gap.west)--(dgn1-smax.south);

\node [rotate=-90] (dgn1-galu-4) at (-3.25,-2.5){\tiny{GaLU}};
\draw [-stealth,thick]   (dgn1-galu-4.north)--(dgn1-gap.south);

\node [] (dgn1-v-c4) at (-4,-2.5){\tiny{$C^{\text{v}}_4$}};
\draw [-stealth,thick]   (dgn1-v-c4.east) -- (dgn1-galu-4.south);

\node [rotate=-90] (dgn1-galu-3) at (-4.75,-2.5){\tiny{GaLU}};
\draw [-stealth,thick]   (dgn1-galu-3.north) -- (dgn1-v-c4.west);

\node [] (dgn1-v-c3) at (-5.5,-2.5){\tiny{$C^{\text{v}}_3$}};
\draw [-stealth,thick]   (dgn1-v-c3.east) -- (dgn1-galu-3.south);

\node [rotate=-90] (dgn1-galu-2) at (-6.25,-2.5){\tiny{GaLU}};
\draw [-stealth,thick]   (dgn1-galu-2.north) -- (dgn1-v-c3.west);

\node [] (dgn1-v-c2) at (-7,-2.5){\tiny{$C^{\text{v}}_2$}};
\draw [-stealth,thick]   (dgn1-v-c2.east) -- (dgn1-galu-2.south);

\node [rotate=-90] (dgn1-galu-1) at (-7.75,-2.5){\tiny{GaLU}};
\draw [-stealth,thick]   (dgn1-galu-1.north) -- (dgn1-v-c2.west);

\node [] (dgn1-v-c1) at (-8.5,-2.5){\tiny{$C^{\text{v}}_1$}};
\draw [-stealth,thick]   (dgn1-v-c1.east) -- (dgn1-galu-1.south);

\node [] (dgn1-v-input) at (-8.5,-1){$x^{\text{v}}$};

\draw [-stealth,thick]   (dgn1-v-input.south) -- (dgn1-v-c1.north);

\node[] (dgn1-gating-4-up) at (-3.25,0.5){\tiny{$G_{4}$}};
\draw [-stealth,thick]   (dgn1-f-c4.east) to[out=-90,in=90] (dgn1-gating-4-up.north);

\node[] (dgn1-gating-3-up) at (-4.75,0.5){\tiny{$G_{3}$}};
\draw [-stealth,thick]   (dgn1-f-c3.east) to[out=-90,in=90] (dgn1-gating-3-up.north);

\node[] (dgn1-gating-2-up) at (-6.25,0.5){\tiny{$G_{2}$}};
\draw [-stealth,thick]   (dgn1-f-c2.east) to[out=-90,in=90] (dgn1-gating-2-up.north);

\node[] (dgn1-gating-1-up) at (-7.75,0.5){\tiny{$G_{1}$}};
\draw [-stealth,thick]   (dgn1-f-c1.east) to[out=-90,in=90] (dgn1-gating-1-up.north);

\node[] (dgn1-gating-4) at (-3.25,-1.5){\tiny{$G_{i_4}$}};
\draw [-stealth,thick]   (dgn1-gating-4.south) -- (dgn1-galu-4.west);

\node[] (dgn1-gating-3) at (-4.75,-1.5){\tiny{$G_{i_3}$}};
\draw [-stealth,thick]   (dgn1-gating-3.south) -- (dgn1-galu-3.west);

\node[] (dgn1-gating-2) at (-6.25,-1.5){\tiny{$G_{i_2}$}};
\draw [-stealth,thick]   (dgn1-gating-2.south) -- (dgn1-galu-2.west);

\node[] (dgn1-gating-1) at (-7.75,-1.5){\tiny{$G_{i_1}$}};
\draw [-stealth,thick]   (dgn1-gating-1.south) -- (dgn1-galu-1.west);

\node[] (permutation1) at (-5.5,-0.5){\small{Layer Permutation}};

\draw [-]   (dgn1-gating-1-up.south) -- (permutation1.north);
\draw [-]   (dgn1-gating-4-up.south) -- (permutation1.north);
\draw [-]   (dgn1-gating-2-up.south) -- (permutation1.north);
\draw [-]   (dgn1-gating-3-up.south) -- (permutation1.north);

\draw [-]  (permutation1.south) --  (dgn1-gating-1.north)  ;
\draw [-]  (permutation1.south) --  (dgn1-gating-2.north)  ;
\draw [-]  (permutation1.south) --  (dgn1-gating-3.north)  ;
\draw [-]  (permutation1.south) --  (dgn1-gating-4.north)  ;

%%%%%%%%%%%%%%%%%%%%%%%%%%%%%%%%%%%%%%%%%%%%%%%%%%%%%%%%%%%%%%%%%

\end{tikzpicture}
}
\end{minipage}
\centering
\resizebox{.90\columnwidth}{!}{
\begin{tabular}{cccccccc}
\toprule
\multicolumn{8}{c}{Table I}\\
\toprule
Dataset 					& Permute	&C4GAP 				&DGN$(x,x)$ 			&DGN$(x,\mathbf{1})$ 	&DLGN$(x,x)$ 			&DLGN$(x,\mathbf{1})$ 	&$\frac{\text{DLGN}(x,\mathbf{1})}{\text{DNN}}$\\\midrule
\multirow{2}{*}{CIFAR10}		& No			&80.5\tiny{$\pm$0.4} 	&77.4\tiny{$\pm$0.3} 	& 77.5\tiny{$\pm$0.2} 	&75.4\tiny{$\pm$0.3} 	&75.4\tiny{$\pm$0.2}		&$93.66$	\\
						& Yes 		&-- 					&77.3\tiny{$\pm$0.5} 	&77.9\tiny{$\pm$0.6}		&75.9\tiny{$\pm$0.5} 	&76.0\tiny{$\pm$0.5}		&$94.40$	\\\midrule
\multirow{2}{*}{CIFAR100}		& No  		&51.8\tiny{$\pm$0.4} 	&47.4\tiny{$\pm$0.2}		&47.3\tiny{$\pm$0.3} 	&47.4\tiny{$\pm$0.1} 	&48.0\tiny{$\pm$0.2}		&$92.66$\\
						& Yes 		& -- 					&48.4\tiny{$\pm$0.8} 	&49.2\tiny{$\pm$0.9} 	&47.5\tiny{$\pm$1.0} 	&48.4\tiny{$\pm$0.9}		&$93.43$\\
\bottomrule
\end{tabular}
}
\resizebox{0.9\columnwidth}{!}{
\begin{tabular}{cccccccccc}
\toprule
\multicolumn{8}{c}{Table II}\\
\toprule
Dataset 					& Model 		&DNN 				&DGN$(x,x)$ 			&DGN$(x,\mathbf{1})$ 	&DLGN$(x,x)$ 			&DLGN$(x,\mathbf{1})$ 	&$\frac{\text{DLGN}(x,\mathbf{1})}{\text{DNN}}$\\\midrule		
\multirow{2}{*}{CIFAR10}		&VGG16 		&93.6\tiny{$\pm$0.2} 	& 93.0\tiny{$\pm$0.1}  	&93.0\tiny{$\pm$0.1}   	&87.0\tiny{$\pm$0.1}		&87.0\tiny{$\pm$0.2}		&$92.94$	\\
						&ResNet110 	&94.0\tiny{$\pm$0.2} 	& 93.3\tiny{$\pm$0.2} 	&93.2\tiny{$\pm$0.1} 	&87.9\tiny{$\pm$0.2}   	&87.8\tiny{$\pm$0.1} 	&$93.40$	\\\midrule
\multirow{2}{*}{CIFAR100}		&VGG16	 	&73.4\tiny{$\pm$0.3}  	&70.3\tiny{$\pm$0.1} 	&70.5\tiny{$\pm$0.2} 	&61.5\tiny{$\pm$0.2}		&61.5\tiny{$\pm$0.1}		&$\mathbf{83.78}$	\\
						&ResNet110 	&72.7\tiny{$\pm$0.2}		&70.8\tiny{$\pm$0.2} 	&70.8\tiny{$\pm$0.4}		&62.3\tiny{$\pm$0.2} 	&62.7\tiny{$\pm$0.3} 	&$86.24$	\\
\bottomrule
\end{tabular}
}
\caption{\small{Here the gates $G_1, G_2, G_3, G_4$ are generated by the feature network and are permuted as $G_{i_1},G_{i_2},G_{i_3},G_{i_4}$ before applying to the value network. $C_1,C_2,C_3,C_4$ have $128$ filters each. \textbf{Table I and II:} All columns (except the last) show the $\%$ test accuracy on CIFAR-10 and CIFAR-100, and $\%$ of DNN performance recovered by DLGN is in the last column. \textbf{Table I:} For each dataset, the top row has results for vanilla models without permutations (the results are averaged over $5$ runs) and the bottom row has results of $4!-1=23$ permutations (except the identity) for each model (the results are averaged over the $23$ permutations). }}
\label{fig:c4gap}
\end{figure}

\Cref{sec:analysis}  presented theoretical results which demystified the layer-by-layer view in value network, in this section we will verify these theoretical results in experiments. We then show that DLGN recovers major part of performance of state-of-the-art DNNs on CIFAR-10 and CIFAR-100.

\textbf{Setup Details.} We consider $3$ DNN architectures, C4GAP, VGG-16 and Resnet-110, and their DGN and DLGN counterparts. Here C4GAP is a simple model (achieves about $80\%$ accuracy on CIFAR-10), mainly used to verify the theoretical insights in \Cref{sec:analysis}. VGG-16 and Resnet-110 are chosen for their state-of-the-art performance on CIFAR-10 and CIFAR-100. All models are trained using off-the-shelf optimisers (for more details, see \Cref{sec:expdetails}). The DGN and DLGN are trained from scratch, i.e., both the feature and value network are initialised at random and trained. In DGN and DLGN, we use soft gating (see \Cref{fig:dgn}) so that gradient flows through the feature network and the gates are learnt (we chose $\beta=10$).  In what follows, we use the notation DGN$(\xf,\xv)$ and DLGN$(\xf,\xv)$ where $\xf$ and $\xv$ denote the input to the value and feature networks respectively. For instance, DGN$(x,x)$ will mean that both the value and feature network of the DGN is provided the image as input, and DLGN$(x,\mathbf{1})$ will mean that the feature network is given with the image as input and the value network is given a constant $\mathbf{1}$ as input. 

\textbf{Disentangling Value Network.} We show that destroying the layer-by-layer structure via permutations and providing a constant $\mathbf{1}$ input do not degrade performance. Since our aim here is not state-of-the-art performance, we use C4GAP with $4$ convolutional layers which achieves only about $80\%$ test accuracy on CIFAR-10, however, enables us to run all the $4!=24$ layer permutations. The C4GAP, DGN and DLGN with layer permutations are shown in \Cref{fig:c4gap}. Once a permutation is chosen, it is fixed during both training and testing. The results in Table I of \Cref{fig:c4gap} show that there is no significant difference in performance between DGN$(x,x)$ vs DGN$(x,\mathbf{1})$, and DLGN$(x,x)$ vs DLGN$(x,\mathbf{1})$, i.e., constant $\mathbf{1}$ input does not hurt. Also, there is no significant difference between the models without permutations and the models with permutations.  These counter intuitive and surprising results are difficult to explain using the commonly held `sophisticated features are learnt layer-by-layer' view. However, neither the permutations or the constant $\mathbf{1}$ input destroys the correlation in the gates, and are not expected to degrade performance as per the insights in \Cref{sec:analysis}.

\textbf{DLGN Performance.} For this we choose VGG-16 and Resnet-110. The results in Table II of \Cref{fig:c4gap} show that the DLGN recovers more than $83.5\%$ (i.e., $83.78\%$ in the worst case) of the performance of the state-of-the-art DNN. 
While entanglement in the DNNs enable their improved performance,  the `disentangled and interpretable'  computations in the DLGN recovers most part of the performance. 

\subsection{Is DLGN a Universal Spectral Approximator}
Motivated by the success of DLGN, we further break it down into DLGN-\emph{Shallow Features} (DLGN-SF), wherein, the feature network is a collection of shallow single matrix multiplications. We compare a shallow DNN with ReLU called C1GAP against DLGN-SF of C4GAP (see \Cref{fig:shallow}) and VGG-16 (see Appendix). The results are in \Cref{fig:shallow}, based on which we observe the following:

$\bullet$ \textbf{Power of depth in value network and lifting to dual space.} Both C1GAP and C4GAP-DLGN-SF were trained with identical batch size, optimiser and learning rate (chosen to be the best for C1GAP). The performance of C1GAP at $200$ epochs is $\sim10\%$ lower than that of C4GAP-DLGN-SF. After $2000$ epochs of training and ensembling 16 such C1GAP's as C1GAP-\texttt{16-ENS} closes the gap within $\sim 3\%$ on C4GAP-DLGN-SF. Yet, a deeper architecture VGG-16-DLGN-SF is $\sim10\%$ better than C1GAP-\texttt{16-ENS}.  Note that both VGG-16-DLGN-SF and C1GAP-\texttt{16-ENS} have gates for $16$ layers produced in a shallow manner. While in a C1GAP-\texttt{16-ENS}, `16' C1GAPs are ensembled, in VGG-16-DLGN-SF these gates for 16 layers are used as gating signals to turn `on/off' the GaLUs laid depth-wise as 16 layers of the value network, which helps to lift the computations to the dual space. Thus, using the gates to \textbf{lift} (instead of ensembling) the computations to the dual space in the value network is playing a critical role, investigating which is an important future work.

$\bullet$ \textbf{Power of depth in feature network.} By comparing CIFAR-100 performance of VGG-16-DLGN-SF in \Cref{fig:shallow} and that of VGG-16-DLGN in \Cref{fig:c4gap}, we see $\sim 6\%$ improvement if we have a deep linear network instead of many shallow linear networks as the feature network. This implies depth helps even if the feature network is entirely linear, investigating which is an important future work. 

$\bullet$ \textbf{DGN vs DLGN} In Table I and II of \Cref{fig:c4gap}, the difference between DGN and DNN is minimal (about $3\%$), however, the difference between DLGN and DNN is significantly large. Thus, it is important to understand the role of the ReLUs in the feature network of DGN. It is interesting to know whether this is simpler than understanding the DNN with ReLUs itself.

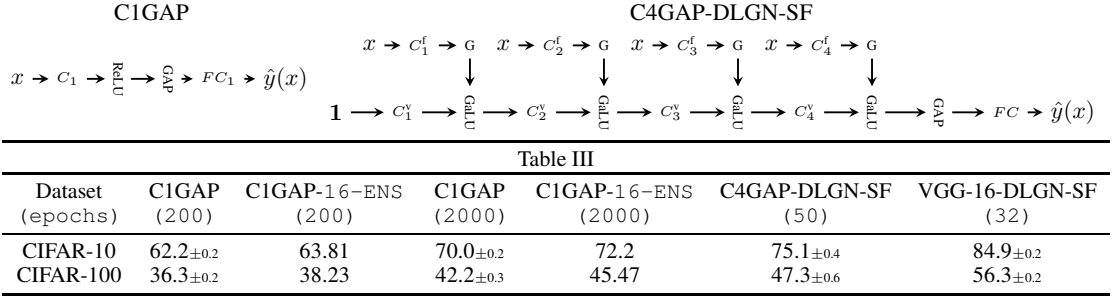
\begin{figure}
\centering
\resizebox{0.9\columnwidth}{!}{
\begin{tikzpicture}
%%%%%%%%%%%%%%%%%%%%%%%%%%%%%%%%%%%%%%%%%%%%%%%%%%%%%%%%%%%%%%%%%
\node []  (fntext)at (0.5+5.25,1.0) {C4GAP-DLGN-SF};

%\node []  (output) at (7.5,1.5) {$\hat{y}(x)$};

\node [] (dgn-f-input) at (0.5,1.5-1){$x$};
\node [] (dgn-f-c4) at (8-0.75,1.5-1){\tiny{$C^{\text{f}}_4$}};
\node [] (dgn-f-x-4) at (8-0.75-0.75,1.5-1){$x$};
\draw [-stealth,thick]   (dgn-f-x-4.east) -- (dgn-f-c4.west);

\node [] (dgn-f-c3) at (6-0.75,1.5-1){\tiny{$C^{\text{f}}_3$}};
\node [] (dgn-f-x-3) at (6-0.75-0.75,1.5-1){$x$};
\draw [-stealth,thick]   (dgn-f-x-3.east) -- (dgn-f-c3.west);

\node [] (dgn-f-c2) at (3.5-0.75+0.5,1.5-1){\tiny{$C^{\text{f}}_2$}};
\node [] (dgn-f-x-2) at (3.5-0.75+0.5-0.75,1.5-1){$x$};
\draw [-stealth,thick]   (dgn-f-x-2.east) -- (dgn-f-c2.west);

\node [] (dgn-f-c1) at (2-0.75,1.5-1){\tiny{$C^{\text{f}}_1$}};

\draw [-stealth,thick]   (dgn-f-input.east) -- (dgn-f-c1.west);

\node []  (dgn-output) at (11,-0.5) {$\hat{y}(x)$};

\node [] (dgn-smax) at (10,-0.5){\tiny{$FC$}};
\draw [-stealth,thick]   (dgn-smax.east) -- (dgn-output.west);
;
\node [rotate=-90] (dgn-gap) at (9,-0.5){\tiny{GAP}};
\draw [-stealth,thick]   (dgn-gap.north) -- (dgn-smax.west);

\node [rotate=-90] (dgn-galu-4) at (8,-0.5){\tiny{GaLU}};

\draw [-stealth,thick]   (dgn-galu-4.north) -- (dgn-gap.south);
\node [] (dgn-v-c4) at (7,-0.5){\tiny{$C^{\text{v}}_4$}};
\draw [-stealth,thick]   (dgn-v-c4.east) -- (dgn-galu-4.south);

\node [rotate=-90] (dgn-galu-3) at (6,-0.5){\tiny{GaLU}};
\draw [-stealth,thick]   (dgn-galu-3.north) -- (dgn-v-c4.west);

\node [] (dgn-v-c3) at (5,-0.5){\tiny{$C^{\text{v}}_3$}};
\draw [-stealth,thick]   (dgn-v-c3.east) -- (dgn-galu-3.south);

\node [rotate=-90] (dgn-galu-2) at (4,-0.5){\tiny{GaLU}};
\draw [-stealth,thick]   (dgn-galu-2.north) -- (dgn-v-c3.west);

\node [] (dgn-v-c2) at (3,-0.5){\tiny{$C^{\text{v}}_2$}};
\draw [-stealth,thick]   (dgn-v-c2.east) -- (dgn-galu-2.south);

\node [rotate=-90] (dgn-galu-1) at (2.0,-0.5){\tiny{GaLU}};
\draw [-stealth,thick]   (dgn-galu-1.north) -- (dgn-v-c2.west);

\node [] (dgn-v-c1) at (1,-0.5){\tiny{$C^{\text{v}}_1$}};
\draw [-stealth,thick]   (dgn-v-c1.east) -- (dgn-galu-1.south);

\node [] (dgn-v-input) at (0.0,-0.5){$\mathbf{1}$};

\draw [-stealth,thick]   (dgn-v-input.east) -- (dgn-v-c1.west);

\node[] (dgn-gating-1) at (2,0.5){\tiny{G}};
\draw [-stealth,thick]   (dgn-f-c1.east)-- (dgn-gating-1.west);

\node[] (dgn-gating-2) at (3.5+0.5,0.5){\tiny{G}};
\draw [-stealth,thick]   (dgn-f-c2.east)-- (dgn-gating-2.west);

\node[] (dgn-gating-3) at (6,0.5){\tiny{G}};
\draw [-stealth,thick]   (dgn-f-c3.east) -- (dgn-gating-3.west);

\node[] (dgn-gating-4) at (8,0.5){\tiny{G}};
\draw [-stealth,thick]   (dgn-f-c4.east) -- (dgn-gating-4.west);

\draw [-stealth,thick]   (dgn-gating-1.south) -- (dgn-galu-1.west);

\draw [-stealth,thick]   (dgn-gating-2.south) -- (dgn-galu-2.west);

\draw [-stealth,thick]   (dgn-gating-3.south) -- (dgn-galu-3.west);

\draw [-stealth,thick]   (dgn-gating-4.south) -- (dgn-galu-4.west);

\node []  (fntext)at (-4.75+2,1.0) {C1GAP};

\node []  (c1gap-output) at (-0.75,0) {$\hat{y}(x)$};

\node []  (c1gap-fc) at (-1.75,0) {\tiny{$FC_1$}};

\node [rotate=-90]  (c1gap-gap) at (-2.5,0) {\tiny{GAP}};

\node [rotate=-90]  (c1gap-relu) at (-3.25,0) {\tiny{ReLU}};

\node []  (c1gap-c1) at (-4,0) {\tiny{$C_1$}};

\node []  (c1gap-x) at (-4.75,0) {$x$};

\draw [-stealth,thick]		(c1gap-x.east) 	-- (c1gap-c1.west);

\draw [-stealth,thick]    	(c1gap-c1.east) -- (c1gap-relu.south) ;

\draw [-stealth,thick]     	(c1gap-relu.north) -- (c1gap-gap.south) ;

\draw [-stealth,thick]     	(c1gap-gap.north) -- (c1gap-fc.west) ;

\draw [-stealth,thick]     	(c1gap-fc.east) -- (c1gap-output.west) ;
	
\end{tikzpicture}
}
\resizebox{0.9\columnwidth}{!}{
\begin{tabular}{ccccccc}
\toprule
\multicolumn{7}{c}{Table III}\\
\toprule
Dataset 			&C1GAP 				&C1GAP-\texttt{16-ENS}		&C1GAP 				&C1GAP-\texttt{16-ENS}		&C4GAP-DLGN-SF	&VGG-16-DLGN-SF\\
\texttt{(epochs)} 		&\texttt{(200)}	 		&\texttt{(200)}			&\texttt{(2000)} 		&\texttt{(2000)}				&\texttt{(50)}					&\texttt{(32)}\\\midrule		
CIFAR-10			&62.2\tiny{$\pm$0.2}		&63.81					&70.0\tiny{$\pm$0.2}		&72.2					&75.1\tiny{$\pm$0.4}				&84.9\tiny{$\pm$0.2}\\ 
CIFAR-100		&36.3\tiny{$\pm$0.2}		&38.23					&42.2\tiny{$\pm$0.3}		&45.47					&47.3\tiny{$\pm$0.6}				&56.3\tiny{$\pm$0.2}\\
\bottomrule
\end{tabular}
}
\caption{\small{C1GAP and C4GAP have width = 512 to make them comparable to VGG-16 whose maximum width is 512. The ensemble size is 16 to match the 16 layers of VGG-16. Note that C4GAP in \Cref{fig:c4gap} has width=128. The last two columns show only DLGN-SF$(x,\mathbf{1})$. We observed the performance of DLGN-SF$(x,x)$ to be $\sim 2\%$ lesser and left it out from Table III for sake brevity. All results except for \texttt{ENS} are averaged over $5$ runs.}}
\label{fig:shallow}
\end{figure}

$\bullet$ \textbf{Is DLGN a Universal Spectral Approximator?} The value network realises the NPK which in general is an ensemble (assuming skip connections). The NPK is based on the gates whose pre-activations are generated linearly. It is interesting to ask whether the DLGN via its feature network learns the right linear transformations to extract the relevant spectral features (to tigger the gates) and via its value network learns  the ensembling of kernels (based on gates) in a dataset dependent manner.

\section{Conclusion}
Entanglement of the non-linear and the linear operation in each layer of a DNN makes them uninterpretable. This paper proposed a novel DLGN which disentangled the computations in a DNN with ReLUs into two mathematically interpretable linearities, the `primal' linearity from the input to the pre-activations that trigger the gates, and the `dual' linearity in the path space. DLGN recovers more than $83.5\%$ of performance of state-of-the-art DNNs on CIFAR-10 and CIFAR-100. Based on this success of DLGN, the paper concluded by asking `Is DLGN a universal spectral approximator?'.

%Entanglement of the non-linear and the linear operation in each layer of a DNN makes them uninterpretable. This paper proposed a novel DLGN which disentangled and rearranged the computations in a DNN with ReLUs in an interpretable manner. 
%The DLGN has two mathematically interpretable linearities, the `primal' linearity from the input to the pre-activations that trigger the gates, and the `dual' linearity in the path space which is characterised by the neural path kernel. 
%It was shown that the disentangled and interpretable computations in a DLGN recover more than $83.5\%$ of performance of state-of-the-art DNNs on CIFAR-10 and CIFAR-100. Based on this success of DLGN, the paper concluded by asking `Is DLGN a universal spectral approximator?'.

\bibliographystyle{unsrtnat}
\bibliography{refs} 

\begin{thebibliography}{22}
\providecommand{\natexlab}[1]{#1}
\providecommand{\url}[1]{\texttt{#1}}
\expandafter\ifx\csname urlstyle\endcsname\relax
  \providecommand{\doi}[1]{doi: #1}\else
  \providecommand{\doi}{doi: \begingroup \urlstyle{rm}\Url}\fi

\bibitem[Jacot et~al.(2018)Jacot, Gabriel, and Hongler]{ntk}
Arthur Jacot, Franck Gabriel, and Cl{\'e}ment Hongler.
\newblock Neural tangent kernel: Convergence and generalization in neural
  networks.
\newblock In \emph{Advances in neural information processing systems}, pages
  8571--8580, 2018.

\bibitem[Arora et~al.(2019)Arora, Du, Hu, Li, Salakhutdinov, and
  Wang]{arora2019exact}
Sanjeev Arora, Simon~S Du, Wei Hu, Zhiyuan Li, Russ~R Salakhutdinov, and
  Ruosong Wang.
\newblock On exact computation with an infinitely wide neural net.
\newblock In \emph{Advances in Neural Information Processing Systems}, pages
  8139--8148, 2019.

\bibitem[Cao and Gu(2019)]{cao2019generalization}
Yuan Cao and Quanquan Gu.
\newblock Generalization bounds of stochastic gradient descent for wide and
  deep neural networks.
\newblock In \emph{Advances in Neural Information Processing Systems}, pages
  10835--10845, 2019.

\bibitem[Lakshminarayanan and Singh(2020)]{npk}
Chandrashekar Lakshminarayanan and Amit~Vikram Singh.
\newblock Neural path features and neural path kernel: Understanding the role
  of gates in deep learning.
\newblock \emph{Advances in Neural Information Processing Systems}, 33, 2020.

\bibitem[Fiat et~al.(2019)Fiat, Malach, and Shalev{-}Shwartz]{sss}
Jonathan Fiat, Eran Malach, and Shai Shalev{-}Shwartz.
\newblock Decoupling gating from linearity.
\newblock \emph{CoRR}, abs/1906.05032, 2019.
\newblock URL \url{http://arxiv.org/abs/1906.05032}.

\bibitem[Fan and Wang(2020)]{spectra}
Zhou Fan and Zhichao Wang.
\newblock Spectra of the conjugate kernel and neural tangent kernel for
  linear-width neural networks.
\newblock \emph{arXiv preprint arXiv:2005.11879}, 2020.

\bibitem[Geifman et~al.(2020)Geifman, Yadav, Kasten, Galun, Jacobs, and
  Basri]{laplace}
Amnon Geifman, Abhay Yadav, Yoni Kasten, Meirav Galun, David Jacobs, and Ronen
  Basri.
\newblock On the similarity between the laplace and neural tangent kernels.
\newblock \emph{arXiv preprint arXiv:2007.01580}, 2020.

\bibitem[Liu et~al.(2020)Liu, Zhu, and Belkin]{belkin}
Chaoyue Liu, Libin Zhu, and Mikhail Belkin.
\newblock On the linearity of large non-linear models: when and why the tangent
  kernel is constant.
\newblock \emph{Advances in Neural Information Processing Systems}, 33, 2020.

\bibitem[Chen et~al.(2020)Chen, Cao, Gu, and Zhang]{genntk}
Zixiang Chen, Yuan Cao, Quanquan Gu, and Tong Zhang.
\newblock A generalized neural tangent kernel analysis for two-layer neural
  networks.
\newblock \emph{Advances in Neural Information Processing Systems}, 33, 2020.

\bibitem[Xiao et~al.(2020)Xiao, Pennington, and Schoenholz]{disentangling}
Lechao Xiao, Jeffrey Pennington, and Samuel Schoenholz.
\newblock Disentangling trainability and generalization in deep neural
  networks.
\newblock In \emph{International Conference on Machine Learning}, pages
  10462--10472. PMLR, 2020.

\bibitem[Novak et~al.(2018)Novak, Xiao, Lee, Bahri, Yang, Hron, Abolafia,
  Pennington, and Sohl-Dickstein]{convgp}
Roman Novak, Lechao Xiao, Jaehoon Lee, Yasaman Bahri, Greg Yang, Jiri Hron,
  Daniel~A Abolafia, Jeffrey Pennington, and Jascha Sohl-Dickstein.
\newblock Bayesian deep convolutional networks with many channels are gaussian
  processes.
\newblock \emph{arXiv preprint arXiv:1810.05148}, 2018.

\bibitem[Lee et~al.(2017)Lee, Bahri, Novak, Schoenholz, Pennington, and
  Sohl-Dickstein]{fcgp}
Jaehoon Lee, Yasaman Bahri, Roman Novak, Samuel~S Schoenholz, Jeffrey
  Pennington, and Jascha Sohl-Dickstein.
\newblock Deep neural networks as gaussian processes.
\newblock \emph{arXiv preprint arXiv:1711.00165}, 2017.

\bibitem[Lee et~al.(2020{\natexlab{a}})Lee, Schoenholz, Pennington, Adlam,
  Xiao, Novak, and Sohl-Dickstein]{lee2020finite}
Jaehoon Lee, Samuel~S Schoenholz, Jeffrey Pennington, Ben Adlam, Lechao Xiao,
  Roman Novak, and Jascha Sohl-Dickstein.
\newblock Finite versus infinite neural networks: an empirical study.
\newblock \emph{arXiv preprint arXiv:2007.15801}, 2020{\natexlab{a}}.

\bibitem[Balestriero et~al.(2018)]{balestriero2018spline}
Randall Balestriero et~al.
\newblock A spline theory of deep learning.
\newblock In \emph{International Conference on Machine Learning}, pages
  374--383, 2018.

\bibitem[Balestriero and Baraniuk(2018)]{balestriero2018hard}
Randall Balestriero and Richard~G Baraniuk.
\newblock From hard to soft: Understanding deep network nonlinearities via
  vector quantization and statistical inference.
\newblock \emph{arXiv preprint arXiv:1810.09274}, 2018.

\bibitem[Srivastava et~al.(2015)Srivastava, Greff, and Schmidhuber]{highway}
Rupesh~Kumar Srivastava, Klaus Greff, and J{\"u}rgen Schmidhuber.
\newblock Highway networks.
\newblock \emph{arXiv preprint arXiv:1505.00387}, 2015.

\bibitem[Ramachandran et~al.(2018)Ramachandran, Zoph, and Le]{swish}
Prajit Ramachandran, Barret Zoph, and Quoc~V. Le.
\newblock Searching for activation functions, 2018.
\newblock URL \url{https://openreview.net/forum?id=SkBYYyZRZ}.

\bibitem[Lee et~al.(2020{\natexlab{b}})Lee, Schoenholz, Pennington, Adlam,
  Xiao, Novak, and Sohl-Dickstein]{finitevsinfinite}
Jaehoon Lee, Samuel~S Schoenholz, Jeffrey Pennington, Ben Adlam, Lechao Xiao,
  Roman Novak, and Jascha Sohl-Dickstein.
\newblock Finite versus infinite neural networks: an empirical study.
\newblock \emph{arXiv preprint arXiv:2007.15801}, 2020{\natexlab{b}}.

\bibitem[Zhang et~al.(2016)Zhang, Bengio, Hardt, Recht, and Vinyals]{ben}
Chiyuan Zhang, Samy Bengio, Moritz Hardt, Benjamin Recht, and Oriol Vinyals.
\newblock Understanding deep learning requires rethinking generalization.
\newblock \emph{arXiv preprint arXiv:1611.03530}, 2016.

\bibitem[Li et~al.(2019)Li, Wang, Yu, Du, Hu, Salakhutdinov, and
  Arora]{li2019enhanced}
Zhiyuan Li, Ruosong Wang, Dingli Yu, Simon~S Du, Wei Hu, Ruslan Salakhutdinov,
  and Sanjeev Arora.
\newblock Enhanced convolutional neural tangent kernels.
\newblock \emph{arXiv preprint arXiv:1911.00809}, 2019.

\bibitem[Veit et~al.(2016)Veit, Wilber, and Belongie]{veit2016residual}
Andreas Veit, Michael Wilber, and Serge Belongie.
\newblock Residual networks behave like ensembles of relatively shallow
  networks.
\newblock \emph{arXiv preprint arXiv:1605.06431}, 2016.

\bibitem[Kingma and Ba(2014)]{adam}
Diederik~P. Kingma and Jimmy Ba.
\newblock Adam: A method for stochastic optimization, 2014.

\end{thebibliography}
\appendix

\section{Convolution With Global Average Pooling}\label{sec:conv}
In this section, we define NPFs and NPV in the presence of convolution with pooling. This requires three key steps (i) treating pooling layers like gates/masks (see \Cref{def:pooling}) (ii) bundling together the paths that share the same path value
(due to weight sharing in convolutions, see \Cref{def:bundle}),  and (iii) re-defining the NPF and NPV for bundles (see \Cref{def:convnps}). Weight sharing due to convolutions and pooling makes the NPK rotationally invariant \Cref{lm:cnnnpk}. We begin by describing the architecture.

\textbf{Architecture:} We consider (for sake of brevity) a $1$-dimensional\footnote{The results follow in a direct manner to any form of circular convolutions.} convolutional neural network with circular convolutions, with $\dc$ convolutional layers ($l=1,\ldots,\dc$), followed by a \emph{global-average-pooling} layer ($l=\dc+1$) and $\dfc$ ($l=\dc+2,\ldots,\dc+\dfc+1$) fully connected  layers. The convolutional window size is $\wconv<\din$, the number of filters per convolutional layer as well as the width of the FC is $w$. 

\textbf{Indexing:} Here $\iin/\iout$ are the indices (taking values in $[w]$) of the input/output filters. $\icin$ denotes the indices of the convolutional window taking values in $[\wconv]$. $\ifout$ denotes the indices (taking values in $[\din]$, the dimension of input features) of individual nodes in a given output filter. The weights of layers $l\in[\dc]$ are denoted by $\Theta(\icin,\iin,\iout,l)$ and for layers $l\in[\dfc]+\dc$ are denoted by $\Theta(\iin,\iout,l)$. The pre-activations, gating and hidden unit outputs are denoted by $q_{x,\Theta}(\ifout,\iout,l)$,  $G_{x,\Theta}(\ifout,\iout,l)$, and $z_{x,\Theta}(\ifout,\iout,l)$ for layers $l=1,\ldots, \dc$.

\begin{definition}[Circular Convolution]
For $x\in\R^{\din}$, $i\in[\din]$ and $r\in\{0,\ldots,\din-1\}$, define :

(i) $i\oplus r = i+r$, for $i+r \leq \din$ and $i\oplus r =i+r-\din$, for $i+r>\din$.

(ii) $rot(x,r)(i)=x(i\oplus r), i\in[\din]$.

(iii) $q_{x,\Theta}(\ifout,\iout,l)=\sum_{\icin,\iin}\Theta(\icin,\iin,\iout,l)\cdot z_{x,\Theta}(\ifout\oplus (\icin-1),\iin,l-1)$. 
\end{definition}
\begin{definition}[Pooling]\label{def:pooling}
Let $G^{\text{pool}}_{x,\Theta}(\ifout,\iout,\dc+1)$ denote the pooling mask, then we have
\centerline{
$z_{x,\Theta}(\iout, \dc+1) =\sum_{\ifout} z_{x,\Theta}(\ifout,\iout,\dc)\cdot G^{\text{pool}}_{x,\Theta}(\ifout,\iout,\dc+1),$
}
where in the case of \emph{global-average-pooling} $G^{\text{pool}}_{x,\Theta}(\ifout,\iout,\dc+1)=\frac{1}{\din},\forall \iout\in[w], \ifout\in[\din]$.
\end{definition}
\FloatBarrier
\begin{table}[!h]
\centering
\resizebox{1.0\columnwidth}{!}{
\begin{tabular}{|c l lll|}\hline
Input Layer&: &$z_{x,\Theta}(\cdot,1,0)$ &$=$ &$x$ \\\hline
\multicolumn{5}{l}{\quad }\\
\multicolumn{5}{l}{\quad \quad \quad \quad \quad \quad \quad \quad \quad \quad \quad \quad \quad \quad \quad \quad Convolutional Layers, $l\in[\dc]$}\\\hline
%\multicolumn{5}{l}{\quad }\\\hline
Pre-Activation&: & $q_{x,\Theta}(\ifout,\iout,l)$& $=$ & $\sum_{\icin,\iin}\Theta(\icin,\iin,\iout,l)\cdot z_{x,\Theta}(\ifout\oplus (\icin-1),\iin,l-1)$\\
Gating Values&: &$G_{x,\Theta}(\ifout,\iout,l)$& $=$ & $\mathbf{1}_{\{q_{x,\Theta}(\ifout,\iout,l)>0\}}$\\
Hidden Unit Output&: &$z_{x,\Theta}(\ifout,\iout,l)$ & $=$ & $q_{x,\Theta}(\ifout,\iout,l)\cdot G_{x,\Theta}(\ifout,\iout,l)$\\\hline
\multicolumn{5}{l}{\quad }\\
\multicolumn{5}{l}{\quad \quad \quad \quad \quad \quad \quad \quad \quad \quad \quad \quad \quad \quad \quad \quad GAP Layer, $l=\dc+1$}\\\hline
%HUO&: &${z}_{x,\Theta}(\iout,l)$ & $=$ & $\frac{1}{\din}\sum_{i\in [\din]} z_{x,\Theta}(i,\iout,l-1)$\\\hline\hline
Hidden Unit Output&: &$z_{x,\Theta}(\iout, \dc+1)$ & $=$ &$\sum_{\ifout} z_{x,\Theta}(\ifout,\iout,\dc)\cdot G^{\text{pool}}_{x,\Theta}(\ifout,\iout,\dc+1)$\\\hline
\multicolumn{5}{l}{\quad }\\
\multicolumn{5}{l}{\quad \quad \quad \quad \quad \quad \quad \quad \quad \quad \quad \quad \quad \quad \quad \quad Fully Connected Layers, $l\in[\dfc]+(\dc+1)$}\\\hline
Pre-Activation&: & $q_{x,\Theta}(\iout,l)$& $=$ & $\sum_{\iin}\Theta(\iin,\iout,l) \cdot z_{x,\Theta}(\iin,l-1) $\\
Gating Values&: &$G_{x,\Theta}(\iout,l)$& $=$ & $\mathbf{1}_{\{(q_{x,\Theta}(\iout,l))>0\}}$\\
Hidden Unit Output&: &$z_{x,\Theta}(\iout,l)$ & $=$ & $q_{x,\Theta}(\iout,l)\cdot G_{x,\Theta}(\iout,l)$\\
Final Output&: & $\hat{y}_{\Theta}(x)$ & $=$ & $\sum_{\iin}\Theta(\iin,\iout, d)\cdot z_{x,\Theta}(\iin,d-1)$\\\hline
\end{tabular}
}
\caption{Shows the information flow in the convolutional architecture described at the beginning of \Cref{sec:conv}.}
\label{tb:cconv}
\end{table}

\subsection{Neural Path Features, Neural Path Value}

\begin{proposition}
The total number of paths in a CNN is given by  $\Pcnn=\din(\wconv w)^{\dc}w^{(\dfc-1)}$.
\end{proposition}

\begin{notation}[Index Maps]
The ranges of index maps $\Ifeat_l$,  $\Iconv_l$, $\I_l$ are $[\din]$, $[\wconv]$ and $[w]$ respectively. 
\end{notation}

\begin{definition}[Bundle Paths of Sharing Weights]\label{def:bundle}
Let $\hat{P}^{\text{cnn}}=\frac{\Pcnn}{\din}$, and $\{B_1,\ldots, B_{\hat{P}^{\text{cnn}}}\}$ be a collection of sets such that $\forall i,j\in [\hat{P}^{\text{cnn}}], i\neq j$ we have $B_i\cap B_j=\emptyset$ and $\cup_{i=1}^{\hat{P}^{\text{cnn}}}B_i =[\Pcnn]$. Further,  if paths $p,p' \in B_i$, then $\Iconv_l(p)=\Iconv_l(p'), \forall l=1,\ldots, \dc$ and $\I_l(p)=\I_l(p'), \forall l=0,\ldots, \dc$.
\end{definition}

\begin{proposition}\label{prop:bundle}
There are exactly $\din$ paths in a bundle.
\end{proposition}

\begin{definition}\label{def:convnps} Let $x\in\R^{\din}$ be the input to the CNN. For this input, 
\begin{tabular}{rlp{12cm}}
$A_{\Theta}(x,p)$&$\eqdef$&$\left(\Pi_{l=1}^{\dc+1} G_{x,\Theta}(\Ifeat_l(p),\I_l(p),l)\right)\cdot\left(\Pi_{l=\dc+2}^{\dc+\dfc+1} G_{x,\Theta}(\I_l(p),l)\right)$\\
$\phi_{x,\Theta}(\hat{p})$&$\eqdef$&$ \sum_{\hat{p}\in B_{\hat{p}}}x(\Ifeat_0(p))A_{\Theta}(x,p)$\\
$v_{\Theta}(B_{\hat{p}})$&$\eqdef$&$ \left(\Pi_{l=1}^{\dc} \Theta(\Iconv_{l}(p),\I_{l-1}(p),\I_{l}(p),l)\right) \cdot\left( \Pi_{l=\dc+2}^{\dc+\dfc+1} \Theta(\I_{l-1}(p),\I_l(p),l)\right)$ 
\end{tabular}
\begin{center}
\begin{tabular}{|c|c|}\hline
NPF &$\phi_{x,\Theta}\eqdef (\phi_{x,\Theta}(B_{\hat{p}}),\hat{p}\in [\hat{P}^{\text{cnn}}])\in\R^{\hat{P}^{\text{cnn}}}$\\\hline
NPV& $v_{\Theta}\eqdef (v_{\Theta}(B_{\hat{p}}),\hat{p}\in [\hat{P}^{\text{cnn}}])\in\R^{\hat{P}^{\text{cnn}}}$\\\hline
\end{tabular}
\end{center}
\end{definition}

\subsection{Rotational Invariant Kernel}
\begin{lemma}\label{lm:cnnnpk}
\begin{align*}
\text{NPK}^{\texttt{CONV}}_{\Theta}(x,x')&=\sum_{r=0}^{\din-1} \ip{x,rot(x',r)}_{\textbf{overlap}_{\Theta}(\cdot, x,rot(x',r))}\\&=\sum_{r=0}^{\din-1} \ip{rot(x,r),x'}_{\textbf{overlap}_{\Theta}(\cdot, rot(x,r),x')}
\end{align*}
\end{lemma}

\begin{proof}
For the CNN architecture considered in this paper, each bundle has exactly $\din$ number of paths, each one corresponding to a distinct input node. For a bundle $b_{\hat{p}}$, let $b_{\hat{p}}(i),i\in[\din]$ denote the path starting from input node $i$.
\begin{align*}
&\sum_{\hat{p}\in [\hat{P}]} \Bigg(\sum_{i,i'\in[\din]} x(i) x'(i') A_{\Theta}\left(x,b_{\hat{p}}(i)\right) A_{\Theta}\left(x',b_{\hat{p}}(i')\right) \Bigg)\\
=&\sum_{\hat{p}\in [\hat{P}]}\Bigg(\sum_{i\in[\din],i'=i\oplus r, r\in\{0,\ldots,\din-1\}} x(i) x'(i\oplus r) A_{\Theta}\left(x,b_{\hat{p}}(i)\right) A_{\Theta}\left(x',b_{\hat{p}}(i\oplus r)\right)\Bigg)\\
=&\sum_{\hat{p}\in [\hat{P}]}\Bigg(\sum_{i\in[\din], r\in\{0,\ldots,\din-1\}} x(i) rot(x',r)(i) A_{\Theta}\left(x,b_{\hat{p}}(i)\right) A_{\Theta}\left(rot(x',r),b_{\hat{p}}(i)\right)\Bigg)\\
=&\sum_{r=0}^{\din-1} \Bigg(\sum_{i\in[\din]} x(i) rot(x',r)(i) \sum_{\hat{p}\in [\hat{P}]}  A_{\Theta}\left(x,b_{\hat{p}}(i)\right) A_{\Theta}\left(rot(x',r),b_{\hat{p}}(i)\right)\Bigg)\\
=&\sum_{r=0}^{\din-1}\Bigg(\sum_{i\in[\din]} x(i) rot(x',r)(i) \textbf{overlap}_{\Theta}(i,x,rot(x',r))\Bigg)\\
=&\sum_{r=0}^{\din-1} \ip{x,rot(x',r)}_{\textbf{overlap}_{\Theta}(\cdot,x,rot(x',r))}
\end{align*}
\end{proof}

In what follows we re-state \Cref{th:conv}.

\begin{theorem} Let $\sigcnn=\frac{\cscale}{\sqrt{w\wconv}}$ for the convolutional layers and $\sigfc=\frac{\cscale}{\sqrt{w}}$ for FC layers. Under \Cref{assmp:main}, as $w\rightarrow\infty$, with  $\bcnn = \ \left(\dconv \sigcnn^{2(\dconv-1)}\sigfc^{2\dfc}+\dfc \sigcnn^{2\dconv}\sigfc^{2(\dfc-1)}\right)$ we have:
\begin{align*}
&\text{NTK}^{\texttt{CONV}}_{\Tdgn_0}\rightarrow\quad \frac{\bcnn}{{\din}^2} \cdot \text{NPK}^{\texttt{CONV}}_{\Tf_0}
\end{align*}
\end{theorem}

\begin{proof}
Follows from Theorem~5.1 in [\citenum{npk}].
\end{proof}

\section{Residual Networks with Skip connections}

As a consequence of the skip connections, within the ResNet architecture there are $2^b$ sub-FC networks (see \Cref{def:subfcdnn}). The total number of paths $\Pres$ in the ResNet is equal to the summation of the paths in these $2^b$ sub-FC networks (see \Cref{prop:resnetpath}). Now, The neural path features and the neural path value are $\Pres$ dimensional quantities, obtained as the concatenation of the NPFs and NPV of the $2^b$ sub-FC networks. 

\begin{proposition}\label{prop:resnetpath}
The total number of paths in the ResNet is  $\Pres = \din \cdot\sum_{i=0}^b \binom{b}{i} w^{(i+2)\dblock-1}$.
\end{proposition}

\begin{lemma}[Sum of Product Kernel]\label{lm:sumofproduct}
Let $\text{NPK}^{\texttt{RES}}_{\Theta}$ be the NPK of the ResNet, and $\text{NPK}^{\J}_{\Theta}$ be the NPK of the sub-FCNs within the ResNet obtained by ignoring those skip connections in the set $\J$. Then, \begin{align*}\text{NPK}^{\texttt{RES}}_{\Theta}=\sum_{\J\in 2^{[b]}}\text{NPK}^{\J}_{\Theta}\end{align*}
%\begin{align*}
%\end{align*}
\end{lemma}
\begin{proof}
Proof is complete by noting that the NPF of the ResNet is a concatenation of the NPFs of the $2^b$ distinct sub-FC-DNNs within the ResNet architecture.
\end{proof}

We re-state \Cref{th:res}
\begin{theorem} Let $\sigma=\frac{\cscale}{\sqrt{w}}$. Under \Cref{assmp:main}, as $w\rightarrow\infty$,  for $\bres^{\J} = (|\J| +2)\cdot\dblock\cdot \sigma^{2\big( (|\J|+2)\dblock-1\big)}$,
\begin{align*}
\text{NTK}^{\texttt{RES}}_{\Tdgn_0}\rightarrow \sum_{\J\in 2^{[b]}}  \bres^{\J} \text{NPK}^{\J}_{\Tf_0}
\end{align*}
\end{theorem}

\begin{proof}
Follows from Theorem~5.1 in [\citenum{npk}].
\end{proof}

\section{Numerical Experiments}\label{sec:expdetails}
We now list the details related to the numerical experiments which have been left out in the main body of the paper.

 $\bullet$ \textbf{Computational Resource.} The numerical experiments were run in Nvidia-RTX 2080 TI GPUs and Tesla V100 GPUs.

$\bullet$ All the models in Table I of \Cref{fig:c4gap} we used Adam \citep{adam} with learning rate of $3\times 10^{-4}$, and batch size of 32.

$\bullet$ In \Cref{sec:dlgn}, the codes for experiments based on VGG-16 and Resnet-110  were refactored from following repository: ``https://github.com/gahaalt/resnets-in-tensorflow2".

$\bullet$ For VGG-16-DLGN in \Cref{fig:c4gap} and DLGN-SF in \Cref{fig:shallow}, the $\max$-pooling were replaced by \emph{average} pooling so as to ensure that the feature network is entirely linear. For the comparison to be fair, we replaced the $\max$ pooling in VGG-16 reported in \Cref{fig:c4gap} by \emph{average} pooling. Batch normalisation layers were retained in VGG-16, VGG-16-DLGN and VGG-16-DLGN-SF (all three are shown in \Cref{fig:vggnets}).

$\bullet$ For VGG-16-DLGN in \Cref{fig:c4gap} and DLGN-SF in \Cref{fig:shallow}, the $\max$-pooling were replaced by \emph{average} pooling so as to ensure that the feature network is entirely linear. For the comparison to be fair, we replaced the $\max$ pooling in VGG-16 reported in \Cref{fig:c4gap}. Batch normalisation layers were retained in VGG-16, VGG-16-DLGN and VGG-16-DLGN-SF.

$\bullet$ All the VGG-16, Resnet-100 (and their DGN/DLGN) models in Table II of \Cref{fig:c4gap} we used \emph{SGD} optimiser with momentum $0.9$ and the following learning rate schedule (as suggested in ``https://github.com/gahaalt/resnets-in-tensorflow2") : for iterations $[0, 400)$ learning rate was $0.01$,  for iterations $[400, 32000)$ the learning rate was $ 0.1$, for iterations $[32000, 48000)$ the learning rate was $0.01$, for iterations $[48000, 64000)$ the learning rate was $0.001$. The batch size was $128$. The models were trained till $32$ epochs.

$\bullet$ The VGG-16-DLGN-SF in Table III of \Cref{fig:shallow} uses the same optimiser, batch size and learning rate schedule as the models in Table II of \Cref{fig:c4gap} as explained in the previous point.

$\bullet$ For C1GAP and C4GAP in Table III of \Cref{fig:shallow}, we used Adam \citep{adam} with learning rate of $10^{-3}$, and batch size of 32. This learning rate is best among the set $\{10^{-1},10^{-2},10^{-3}, 3\times 10^{-4}\}$ for C1GAP.

$\bullet$ Models Used in \Cref{fig:c4gap} and \Cref{fig:shallow} are shown below.

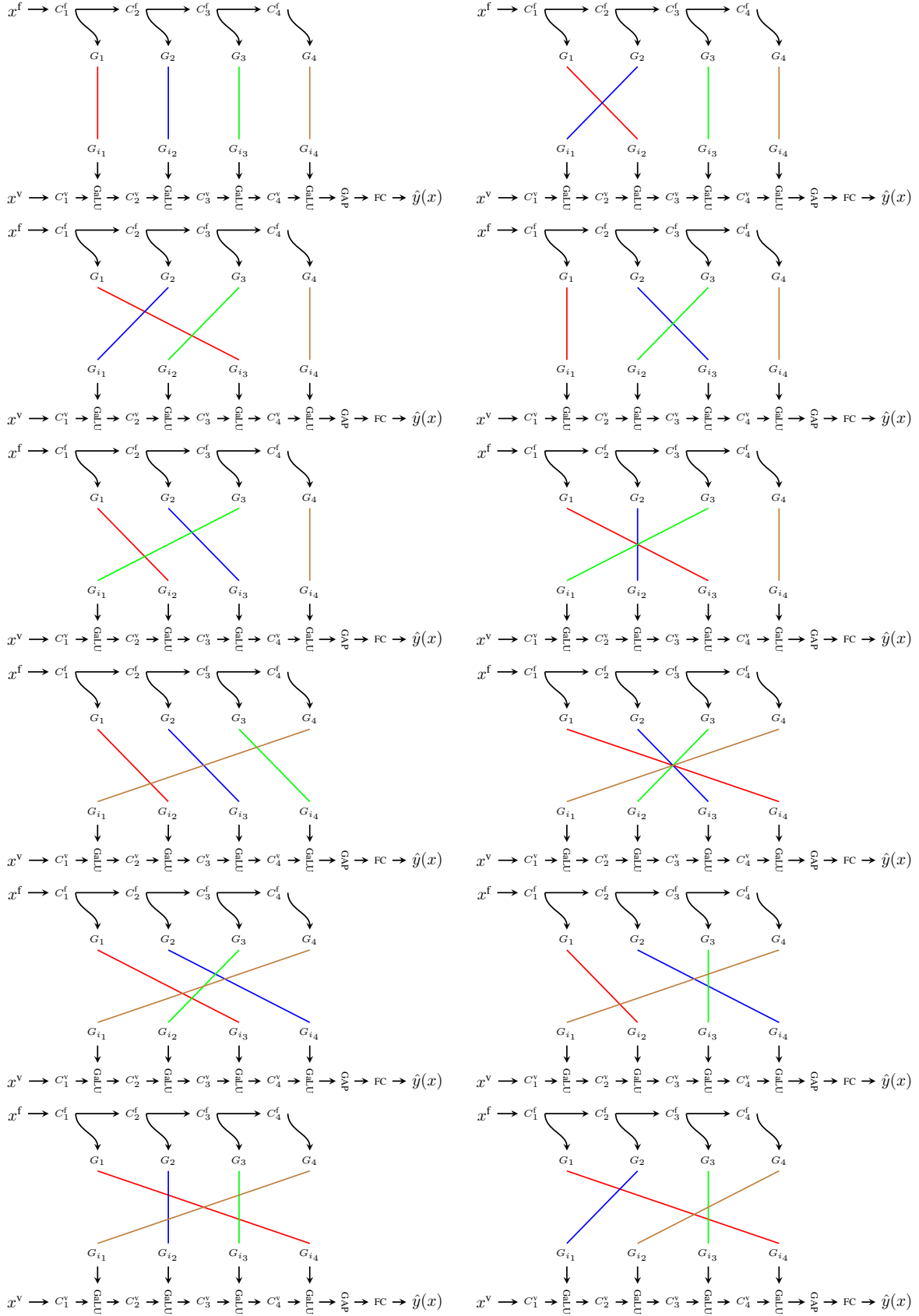
\begin{figure}
\centering
\begin{minipage}{0.48\columnwidth}
\centering
\resizebox{0.99\columnwidth}{!}{
\begin{tikzpicture}
%\node []  (fntext)at (-4.625,-3.5) {CNN-GAP-DLGN};

%\node []  (output) at (7.5,1.5) {$\hat{y}(x)$};

\node [] (dgn1-f-c4) at (-3.5,1.5){\tiny{$C^{\text{f}}_4$}};
\node [] (dgn1-f-c3) at (-5,1.5){\tiny{$C^{\text{f}}_3$}};
\node [] (dgn1-f-c2) at (-6.5,1.5){\tiny{$C^{\text{f}}_2$}};
\node [] (dgn1-f-c1) at (-8,1.5){\tiny{$C^{\text{f}}_1$}};
\node [] (dgn1-input-f) at (-9,1.5){$x^{\text{f}}$};
\draw [-stealth,thick]   (dgn1-f-c3.east) -- (dgn1-f-c4.west);
\draw [-stealth,thick]   (dgn1-f-c2.east) -- (dgn1-f-c3.west);
\draw [-stealth,thick]   (dgn1-f-c1.east) -- (dgn1-f-c2.west);
\draw [-stealth,thick]   (dgn1-input-f.east) -- (dgn1-f-c1.west);

\node []  (dgn1-output) at (-0.25,-2.5) {$\hat{y}(x)$};

\node [] (dgn1-smax) at (-1.25,-2.5){\tiny{FC}};
\draw [-stealth,thick]   (dgn1-smax.east)--(dgn1-output.west);

\node [rotate=-90] (dgn1-gap) at (-2,-2.5){\tiny{GAP}};
\draw [-stealth,thick]   (dgn1-gap.north)--(dgn1-smax.west);

\node [rotate=-90] (dgn1-galu-4) at (-2.75,-2.5){\tiny{GaLU}};
\draw [-stealth,thick]   (dgn1-galu-4.north)--(dgn1-gap.south);

\node [] (dgn1-v-c4) at (-3.5,-2.5){\tiny{$C^{\text{v}}_4$}};
\draw [-stealth,thick]   (dgn1-v-c4.east) -- (dgn1-galu-4.south);

\node [rotate=-90] (dgn1-galu-3) at (-4.25,-2.5){\tiny{GaLU}};
\draw [-stealth,thick]   (dgn1-galu-3.north) -- (dgn1-v-c4.west);

\node [] (dgn1-v-c3) at (-5,-2.5){\tiny{$C^{\text{v}}_3$}};
\draw [-stealth,thick]   (dgn1-v-c3.east) -- (dgn1-galu-3.south);

\node [rotate=-90] (dgn1-galu-2) at (-5.75,-2.5){\tiny{GaLU}};
\draw [-stealth,thick]   (dgn1-galu-2.north) -- (dgn1-v-c3.west);

\node [] (dgn1-v-c2) at (-6.5,-2.5){\tiny{$C^{\text{v}}_2$}};
\draw [-stealth,thick]   (dgn1-v-c2.east) -- (dgn1-galu-2.south);

\node [rotate=-90] (dgn1-galu-1) at (-7.25,-2.5){\tiny{GaLU}};
\draw [-stealth,thick]   (dgn1-galu-1.north) -- (dgn1-v-c2.west);

\node [] (dgn1-v-c1) at (-8,-2.5){\tiny{$C^{\text{v}}_1$}};
\draw [-stealth,thick]   (dgn1-v-c1.east) -- (dgn1-galu-1.south);

\node [] (dgn1-v-input) at (-9,-2.5){$x^{\text{v}}$};

\draw [-stealth,thick]   (dgn1-v-input.east) -- (dgn1-v-c1.west);

\node[] (dgn1-gating-4-up) at (-2.75,0.5){\tiny{$G_{4}$}};
\draw [-stealth,thick]   (dgn1-f-c4.east) to[out=-90,in=90] (dgn1-gating-4-up.north);

\node[] (dgn1-gating-3-up) at (-4.25,0.5){\tiny{$G_{3}$}};
\draw [-stealth,thick]   (dgn1-f-c3.east) to[out=-90,in=90] (dgn1-gating-3-up.north);

\node[] (dgn1-gating-2-up) at (-5.75,0.5){\tiny{$G_{2}$}};
\draw [-stealth,thick]   (dgn1-f-c2.east) to[out=-90,in=90] (dgn1-gating-2-up.north);

\node[] (dgn1-gating-1-up) at (-7.25,0.5){\tiny{$G_{1}$}};
\draw [-stealth,thick]   (dgn1-f-c1.east) to[out=-90,in=90] (dgn1-gating-1-up.north);

\node[] (dgn1-gating-4) at (-2.75,-1.5){\tiny{$G_{i_4}$}};
\draw [-stealth,thick]   (dgn1-gating-4.south) -- (dgn1-galu-4.west);

\node[] (dgn1-gating-3) at (-4.25,-1.5){\tiny{$G_{i_3}$}};
\draw [-stealth,thick]   (dgn1-gating-3.south) -- (dgn1-galu-3.west);

\node[] (dgn1-gating-2) at (-5.75,-1.5){\tiny{$G_{i_2}$}};
\draw [-stealth,thick]   (dgn1-gating-2.south) -- (dgn1-galu-2.west);

\node[] (dgn1-gating-1) at (-7.25,-1.5){\tiny{$G_{i_1}$}};
\draw [-stealth,thick]   (dgn1-gating-1.south) -- (dgn1-galu-1.west);

\draw [-,thick,color=red]   (dgn1-gating-1-up.south) --(dgn1-gating-1.north);

\draw [-,thick,color=blue]   (dgn1-gating-2-up.south) --(dgn1-gating-2.north);

\draw [-,thick,color=green]   (dgn1-gating-3-up.south) --(dgn1-gating-3.north);

\draw [-,thick,color=brown]   (dgn1-gating-4-up.south) --(dgn1-gating-4.north);

%%%%%%%%%%%%%%%%%%%%%%%%%%%%%%%%%%%%%%%%%%%%%%%%%%%%%%%%%%%%%%%%%
	
\end{tikzpicture}
}
\end{minipage}
\begin{minipage}{0.48\columnwidth}
\centering
\resizebox{0.99\columnwidth}{!}{
\begin{tikzpicture}
%\node []  (fntext)at (-4.625,-3.5) {CNN-GAP-DLGN};

%\node []  (output) at (7.5,1.5) {$\hat{y}(x)$};

\node [] (dgn1-f-c4) at (-3.5,1.5){\tiny{$C^{\text{f}}_4$}};
\node [] (dgn1-f-c3) at (-5,1.5){\tiny{$C^{\text{f}}_3$}};
\node [] (dgn1-f-c2) at (-6.5,1.5){\tiny{$C^{\text{f}}_2$}};
\node [] (dgn1-f-c1) at (-8,1.5){\tiny{$C^{\text{f}}_1$}};
\node [] (dgn1-input-f) at (-9,1.5){$x^{\text{f}}$};
\draw [-stealth,thick]   (dgn1-f-c3.east) -- (dgn1-f-c4.west);
\draw [-stealth,thick]   (dgn1-f-c2.east) -- (dgn1-f-c3.west);
\draw [-stealth,thick]   (dgn1-f-c1.east) -- (dgn1-f-c2.west);
\draw [-stealth,thick]   (dgn1-input-f.east) -- (dgn1-f-c1.west);

\node []  (dgn1-output) at (-0.25,-2.5) {$\hat{y}(x)$};

\node [] (dgn1-smax) at (-1.25,-2.5){\tiny{FC}};
\draw [-stealth,thick]   (dgn1-smax.east)--(dgn1-output.west);

\node [rotate=-90] (dgn1-gap) at (-2,-2.5){\tiny{GAP}};
\draw [-stealth,thick]   (dgn1-gap.north)--(dgn1-smax.west);

\node [rotate=-90] (dgn1-galu-4) at (-2.75,-2.5){\tiny{GaLU}};
\draw [-stealth,thick]   (dgn1-galu-4.north)--(dgn1-gap.south);

\node [] (dgn1-v-c4) at (-3.5,-2.5){\tiny{$C^{\text{v}}_4$}};
\draw [-stealth,thick]   (dgn1-v-c4.east) -- (dgn1-galu-4.south);

\node [rotate=-90] (dgn1-galu-3) at (-4.25,-2.5){\tiny{GaLU}};
\draw [-stealth,thick]   (dgn1-galu-3.north) -- (dgn1-v-c4.west);

\node [] (dgn1-v-c3) at (-5,-2.5){\tiny{$C^{\text{v}}_3$}};
\draw [-stealth,thick]   (dgn1-v-c3.east) -- (dgn1-galu-3.south);

\node [rotate=-90] (dgn1-galu-2) at (-5.75,-2.5){\tiny{GaLU}};
\draw [-stealth,thick]   (dgn1-galu-2.north) -- (dgn1-v-c3.west);

\node [] (dgn1-v-c2) at (-6.5,-2.5){\tiny{$C^{\text{v}}_2$}};
\draw [-stealth,thick]   (dgn1-v-c2.east) -- (dgn1-galu-2.south);

\node [rotate=-90] (dgn1-galu-1) at (-7.25,-2.5){\tiny{GaLU}};
\draw [-stealth,thick]   (dgn1-galu-1.north) -- (dgn1-v-c2.west);

\node [] (dgn1-v-c1) at (-8,-2.5){\tiny{$C^{\text{v}}_1$}};
\draw [-stealth,thick]   (dgn1-v-c1.east) -- (dgn1-galu-1.south);

\node [] (dgn1-v-input) at (-9,-2.5){$x^{\text{v}}$};

\draw [-stealth,thick]   (dgn1-v-input.east) -- (dgn1-v-c1.west);

\node[] (dgn1-gating-4-up) at (-2.75,0.5){\tiny{$G_{4}$}};
\draw [-stealth,thick]   (dgn1-f-c4.east) to[out=-90,in=90] (dgn1-gating-4-up.north);

\node[] (dgn1-gating-3-up) at (-4.25,0.5){\tiny{$G_{3}$}};
\draw [-stealth,thick]   (dgn1-f-c3.east) to[out=-90,in=90] (dgn1-gating-3-up.north);

\node[] (dgn1-gating-2-up) at (-5.75,0.5){\tiny{$G_{2}$}};
\draw [-stealth,thick]   (dgn1-f-c2.east) to[out=-90,in=90] (dgn1-gating-2-up.north);

\node[] (dgn1-gating-1-up) at (-7.25,0.5){\tiny{$G_{1}$}};
\draw [-stealth,thick]   (dgn1-f-c1.east) to[out=-90,in=90] (dgn1-gating-1-up.north);

\node[] (dgn1-gating-4) at (-2.75,-1.5){\tiny{$G_{i_4}$}};
\draw [-stealth,thick]   (dgn1-gating-4.south) -- (dgn1-galu-4.west);

\node[] (dgn1-gating-3) at (-4.25,-1.5){\tiny{$G_{i_3}$}};
\draw [-stealth,thick]   (dgn1-gating-3.south) -- (dgn1-galu-3.west);

\node[] (dgn1-gating-2) at (-5.75,-1.5){\tiny{$G_{i_2}$}};
\draw [-stealth,thick]   (dgn1-gating-2.south) -- (dgn1-galu-2.west);

\node[] (dgn1-gating-1) at (-7.25,-1.5){\tiny{$G_{i_1}$}};
\draw [-stealth,thick]   (dgn1-gating-1.south) -- (dgn1-galu-1.west);

\draw [-,thick,color=red]   (dgn1-gating-1-up.south) --(dgn1-gating-2.north);

\draw [-,thick,color=blue]   (dgn1-gating-2-up.south) --(dgn1-gating-1.north);

\draw [-,thick,color=green]   (dgn1-gating-3-up.south) --(dgn1-gating-3.north);

\draw [-,thick,color=brown]   (dgn1-gating-4-up.south) --(dgn1-gating-4.north);

%%%%%%%%%%%%%%%%%%%%%%%%%%%%%%%%%%%%%%%%%%%%%%%%%%%%%%%%%%%%%%%%%
	
\end{tikzpicture}
}
\end{minipage}

\begin{minipage}{0.48\columnwidth}
\centering
\resizebox{0.99\columnwidth}{!}{
\begin{tikzpicture}
%\node []  (fntext)at (-4.625,-3.5) {CNN-GAP-DLGN};

%\node []  (output) at (7.5,1.5) {$\hat{y}(x)$};

\node [] (dgn1-f-c4) at (-3.5,1.5){\tiny{$C^{\text{f}}_4$}};
\node [] (dgn1-f-c3) at (-5,1.5){\tiny{$C^{\text{f}}_3$}};
\node [] (dgn1-f-c2) at (-6.5,1.5){\tiny{$C^{\text{f}}_2$}};
\node [] (dgn1-f-c1) at (-8,1.5){\tiny{$C^{\text{f}}_1$}};
\node [] (dgn1-input-f) at (-9,1.5){$x^{\text{f}}$};
\draw [-stealth,thick]   (dgn1-f-c3.east) -- (dgn1-f-c4.west);
\draw [-stealth,thick]   (dgn1-f-c2.east) -- (dgn1-f-c3.west);
\draw [-stealth,thick]   (dgn1-f-c1.east) -- (dgn1-f-c2.west);
\draw [-stealth,thick]   (dgn1-input-f.east) -- (dgn1-f-c1.west);

\node []  (dgn1-output) at (-0.25,-2.5) {$\hat{y}(x)$};

\node [] (dgn1-smax) at (-1.25,-2.5){\tiny{FC}};
\draw [-stealth,thick]   (dgn1-smax.east)--(dgn1-output.west);

\node [rotate=-90] (dgn1-gap) at (-2,-2.5){\tiny{GAP}};
\draw [-stealth,thick]   (dgn1-gap.north)--(dgn1-smax.west);

\node [rotate=-90] (dgn1-galu-4) at (-2.75,-2.5){\tiny{GaLU}};
\draw [-stealth,thick]   (dgn1-galu-4.north)--(dgn1-gap.south);

\node [] (dgn1-v-c4) at (-3.5,-2.5){\tiny{$C^{\text{v}}_4$}};
\draw [-stealth,thick]   (dgn1-v-c4.east) -- (dgn1-galu-4.south);

\node [rotate=-90] (dgn1-galu-3) at (-4.25,-2.5){\tiny{GaLU}};
\draw [-stealth,thick]   (dgn1-galu-3.north) -- (dgn1-v-c4.west);

\node [] (dgn1-v-c3) at (-5,-2.5){\tiny{$C^{\text{v}}_3$}};
\draw [-stealth,thick]   (dgn1-v-c3.east) -- (dgn1-galu-3.south);

\node [rotate=-90] (dgn1-galu-2) at (-5.75,-2.5){\tiny{GaLU}};
\draw [-stealth,thick]   (dgn1-galu-2.north) -- (dgn1-v-c3.west);

\node [] (dgn1-v-c2) at (-6.5,-2.5){\tiny{$C^{\text{v}}_2$}};
\draw [-stealth,thick]   (dgn1-v-c2.east) -- (dgn1-galu-2.south);

\node [rotate=-90] (dgn1-galu-1) at (-7.25,-2.5){\tiny{GaLU}};
\draw [-stealth,thick]   (dgn1-galu-1.north) -- (dgn1-v-c2.west);

\node [] (dgn1-v-c1) at (-8,-2.5){\tiny{$C^{\text{v}}_1$}};
\draw [-stealth,thick]   (dgn1-v-c1.east) -- (dgn1-galu-1.south);

\node [] (dgn1-v-input) at (-9,-2.5){$x^{\text{v}}$};

\draw [-stealth,thick]   (dgn1-v-input.east) -- (dgn1-v-c1.west);

\node[] (dgn1-gating-4-up) at (-2.75,0.5){\tiny{$G_{4}$}};
\draw [-stealth,thick]   (dgn1-f-c4.east) to[out=-90,in=90] (dgn1-gating-4-up.north);

\node[] (dgn1-gating-3-up) at (-4.25,0.5){\tiny{$G_{3}$}};
\draw [-stealth,thick]   (dgn1-f-c3.east) to[out=-90,in=90] (dgn1-gating-3-up.north);

\node[] (dgn1-gating-2-up) at (-5.75,0.5){\tiny{$G_{2}$}};
\draw [-stealth,thick]   (dgn1-f-c2.east) to[out=-90,in=90] (dgn1-gating-2-up.north);

\node[] (dgn1-gating-1-up) at (-7.25,0.5){\tiny{$G_{1}$}};
\draw [-stealth,thick]   (dgn1-f-c1.east) to[out=-90,in=90] (dgn1-gating-1-up.north);

\node[] (dgn1-gating-4) at (-2.75,-1.5){\tiny{$G_{i_4}$}};
\draw [-stealth,thick]   (dgn1-gating-4.south) -- (dgn1-galu-4.west);

\node[] (dgn1-gating-3) at (-4.25,-1.5){\tiny{$G_{i_3}$}};
\draw [-stealth,thick]   (dgn1-gating-3.south) -- (dgn1-galu-3.west);

\node[] (dgn1-gating-2) at (-5.75,-1.5){\tiny{$G_{i_2}$}};
\draw [-stealth,thick]   (dgn1-gating-2.south) -- (dgn1-galu-2.west);

\node[] (dgn1-gating-1) at (-7.25,-1.5){\tiny{$G_{i_1}$}};
\draw [-stealth,thick]   (dgn1-gating-1.south) -- (dgn1-galu-1.west);

\draw [-,thick,color=red]   (dgn1-gating-1-up.south) --(dgn1-gating-3.north);

\draw [-,thick,color=blue]   (dgn1-gating-2-up.south) --(dgn1-gating-1.north);

\draw [-,thick,color=green]   (dgn1-gating-3-up.south) --(dgn1-gating-2.north);

\draw [-,thick,color=brown]   (dgn1-gating-4-up.south) --(dgn1-gating-4.north);

%%%%%%%%%%%%%%%%%%%%%%%%%%%%%%%%%%%%%%%%%%%%%%%%%%%%%%%%%%%%%%%%%
	
\end{tikzpicture}
}
\end{minipage}
\begin{minipage}{0.48\columnwidth}
\centering
\resizebox{0.99\columnwidth}{!}{
\begin{tikzpicture}
%\node []  (fntext)at (-4.625,-3.5) {CNN-GAP-DLGN};

%\node []  (output) at (7.5,1.5) {$\hat{y}(x)$};

\node [] (dgn1-f-c4) at (-3.5,1.5){\tiny{$C^{\text{f}}_4$}};
\node [] (dgn1-f-c3) at (-5,1.5){\tiny{$C^{\text{f}}_3$}};
\node [] (dgn1-f-c2) at (-6.5,1.5){\tiny{$C^{\text{f}}_2$}};
\node [] (dgn1-f-c1) at (-8,1.5){\tiny{$C^{\text{f}}_1$}};
\node [] (dgn1-input-f) at (-9,1.5){$x^{\text{f}}$};
\draw [-stealth,thick]   (dgn1-f-c3.east) -- (dgn1-f-c4.west);
\draw [-stealth,thick]   (dgn1-f-c2.east) -- (dgn1-f-c3.west);
\draw [-stealth,thick]   (dgn1-f-c1.east) -- (dgn1-f-c2.west);
\draw [-stealth,thick]   (dgn1-input-f.east) -- (dgn1-f-c1.west);

\node []  (dgn1-output) at (-0.25,-2.5) {$\hat{y}(x)$};

\node [] (dgn1-smax) at (-1.25,-2.5){\tiny{FC}};
\draw [-stealth,thick]   (dgn1-smax.east)--(dgn1-output.west);

\node [rotate=-90] (dgn1-gap) at (-2,-2.5){\tiny{GAP}};
\draw [-stealth,thick]   (dgn1-gap.north)--(dgn1-smax.west);

\node [rotate=-90] (dgn1-galu-4) at (-2.75,-2.5){\tiny{GaLU}};
\draw [-stealth,thick]   (dgn1-galu-4.north)--(dgn1-gap.south);

\node [] (dgn1-v-c4) at (-3.5,-2.5){\tiny{$C^{\text{v}}_4$}};
\draw [-stealth,thick]   (dgn1-v-c4.east) -- (dgn1-galu-4.south);

\node [rotate=-90] (dgn1-galu-3) at (-4.25,-2.5){\tiny{GaLU}};
\draw [-stealth,thick]   (dgn1-galu-3.north) -- (dgn1-v-c4.west);

\node [] (dgn1-v-c3) at (-5,-2.5){\tiny{$C^{\text{v}}_3$}};
\draw [-stealth,thick]   (dgn1-v-c3.east) -- (dgn1-galu-3.south);

\node [rotate=-90] (dgn1-galu-2) at (-5.75,-2.5){\tiny{GaLU}};
\draw [-stealth,thick]   (dgn1-galu-2.north) -- (dgn1-v-c3.west);

\node [] (dgn1-v-c2) at (-6.5,-2.5){\tiny{$C^{\text{v}}_2$}};
\draw [-stealth,thick]   (dgn1-v-c2.east) -- (dgn1-galu-2.south);

\node [rotate=-90] (dgn1-galu-1) at (-7.25,-2.5){\tiny{GaLU}};
\draw [-stealth,thick]   (dgn1-galu-1.north) -- (dgn1-v-c2.west);

\node [] (dgn1-v-c1) at (-8,-2.5){\tiny{$C^{\text{v}}_1$}};
\draw [-stealth,thick]   (dgn1-v-c1.east) -- (dgn1-galu-1.south);

\node [] (dgn1-v-input) at (-9,-2.5){$x^{\text{v}}$};

\draw [-stealth,thick]   (dgn1-v-input.east) -- (dgn1-v-c1.west);

\node[] (dgn1-gating-4-up) at (-2.75,0.5){\tiny{$G_{4}$}};
\draw [-stealth,thick]   (dgn1-f-c4.east) to[out=-90,in=90] (dgn1-gating-4-up.north);

\node[] (dgn1-gating-3-up) at (-4.25,0.5){\tiny{$G_{3}$}};
\draw [-stealth,thick]   (dgn1-f-c3.east) to[out=-90,in=90] (dgn1-gating-3-up.north);

\node[] (dgn1-gating-2-up) at (-5.75,0.5){\tiny{$G_{2}$}};
\draw [-stealth,thick]   (dgn1-f-c2.east) to[out=-90,in=90] (dgn1-gating-2-up.north);

\node[] (dgn1-gating-1-up) at (-7.25,0.5){\tiny{$G_{1}$}};
\draw [-stealth,thick]   (dgn1-f-c1.east) to[out=-90,in=90] (dgn1-gating-1-up.north);

\node[] (dgn1-gating-4) at (-2.75,-1.5){\tiny{$G_{i_4}$}};
\draw [-stealth,thick]   (dgn1-gating-4.south) -- (dgn1-galu-4.west);

\node[] (dgn1-gating-3) at (-4.25,-1.5){\tiny{$G_{i_3}$}};
\draw [-stealth,thick]   (dgn1-gating-3.south) -- (dgn1-galu-3.west);

\node[] (dgn1-gating-2) at (-5.75,-1.5){\tiny{$G_{i_2}$}};
\draw [-stealth,thick]   (dgn1-gating-2.south) -- (dgn1-galu-2.west);

\node[] (dgn1-gating-1) at (-7.25,-1.5){\tiny{$G_{i_1}$}};
\draw [-stealth,thick]   (dgn1-gating-1.south) -- (dgn1-galu-1.west);

\draw [-,thick,color=red]   (dgn1-gating-1-up.south) --(dgn1-gating-1.north);

\draw [-,thick,color=blue]   (dgn1-gating-2-up.south) --(dgn1-gating-3.north);

\draw [-,thick,color=green]   (dgn1-gating-3-up.south) --(dgn1-gating-2.north);

\draw [-,thick,color=brown]   (dgn1-gating-4-up.south) --(dgn1-gating-4.north);

%%%%%%%%%%%%%%%%%%%%%%%%%%%%%%%%%%%%%%%%%%%%%%%%%%%%%%%%%%%%%%%%%
	
\end{tikzpicture}
}
\end{minipage}

\begin{minipage}{0.48\columnwidth}
\centering
\resizebox{0.99\columnwidth}{!}{
\begin{tikzpicture}
%\node []  (fntext)at (-4.625,-3.5) {CNN-GAP-DLGN};

%\node []  (output) at (7.5,1.5) {$\hat{y}(x)$};

\node [] (dgn1-f-c4) at (-3.5,1.5){\tiny{$C^{\text{f}}_4$}};
\node [] (dgn1-f-c3) at (-5,1.5){\tiny{$C^{\text{f}}_3$}};
\node [] (dgn1-f-c2) at (-6.5,1.5){\tiny{$C^{\text{f}}_2$}};
\node [] (dgn1-f-c1) at (-8,1.5){\tiny{$C^{\text{f}}_1$}};
\node [] (dgn1-input-f) at (-9,1.5){$x^{\text{f}}$};
\draw [-stealth,thick]   (dgn1-f-c3.east) -- (dgn1-f-c4.west);
\draw [-stealth,thick]   (dgn1-f-c2.east) -- (dgn1-f-c3.west);
\draw [-stealth,thick]   (dgn1-f-c1.east) -- (dgn1-f-c2.west);
\draw [-stealth,thick]   (dgn1-input-f.east) -- (dgn1-f-c1.west);

\node []  (dgn1-output) at (-0.25,-2.5) {$\hat{y}(x)$};

\node [] (dgn1-smax) at (-1.25,-2.5){\tiny{FC}};
\draw [-stealth,thick]   (dgn1-smax.east)--(dgn1-output.west);

\node [rotate=-90] (dgn1-gap) at (-2,-2.5){\tiny{GAP}};
\draw [-stealth,thick]   (dgn1-gap.north)--(dgn1-smax.west);

\node [rotate=-90] (dgn1-galu-4) at (-2.75,-2.5){\tiny{GaLU}};
\draw [-stealth,thick]   (dgn1-galu-4.north)--(dgn1-gap.south);

\node [] (dgn1-v-c4) at (-3.5,-2.5){\tiny{$C^{\text{v}}_4$}};
\draw [-stealth,thick]   (dgn1-v-c4.east) -- (dgn1-galu-4.south);

\node [rotate=-90] (dgn1-galu-3) at (-4.25,-2.5){\tiny{GaLU}};
\draw [-stealth,thick]   (dgn1-galu-3.north) -- (dgn1-v-c4.west);

\node [] (dgn1-v-c3) at (-5,-2.5){\tiny{$C^{\text{v}}_3$}};
\draw [-stealth,thick]   (dgn1-v-c3.east) -- (dgn1-galu-3.south);

\node [rotate=-90] (dgn1-galu-2) at (-5.75,-2.5){\tiny{GaLU}};
\draw [-stealth,thick]   (dgn1-galu-2.north) -- (dgn1-v-c3.west);

\node [] (dgn1-v-c2) at (-6.5,-2.5){\tiny{$C^{\text{v}}_2$}};
\draw [-stealth,thick]   (dgn1-v-c2.east) -- (dgn1-galu-2.south);

\node [rotate=-90] (dgn1-galu-1) at (-7.25,-2.5){\tiny{GaLU}};
\draw [-stealth,thick]   (dgn1-galu-1.north) -- (dgn1-v-c2.west);

\node [] (dgn1-v-c1) at (-8,-2.5){\tiny{$C^{\text{v}}_1$}};
\draw [-stealth,thick]   (dgn1-v-c1.east) -- (dgn1-galu-1.south);

\node [] (dgn1-v-input) at (-9,-2.5){$x^{\text{v}}$};

\draw [-stealth,thick]   (dgn1-v-input.east) -- (dgn1-v-c1.west);

\node[] (dgn1-gating-4-up) at (-2.75,0.5){\tiny{$G_{4}$}};
\draw [-stealth,thick]   (dgn1-f-c4.east) to[out=-90,in=90] (dgn1-gating-4-up.north);

\node[] (dgn1-gating-3-up) at (-4.25,0.5){\tiny{$G_{3}$}};
\draw [-stealth,thick]   (dgn1-f-c3.east) to[out=-90,in=90] (dgn1-gating-3-up.north);

\node[] (dgn1-gating-2-up) at (-5.75,0.5){\tiny{$G_{2}$}};
\draw [-stealth,thick]   (dgn1-f-c2.east) to[out=-90,in=90] (dgn1-gating-2-up.north);

\node[] (dgn1-gating-1-up) at (-7.25,0.5){\tiny{$G_{1}$}};
\draw [-stealth,thick]   (dgn1-f-c1.east) to[out=-90,in=90] (dgn1-gating-1-up.north);

\node[] (dgn1-gating-4) at (-2.75,-1.5){\tiny{$G_{i_4}$}};
\draw [-stealth,thick]   (dgn1-gating-4.south) -- (dgn1-galu-4.west);

\node[] (dgn1-gating-3) at (-4.25,-1.5){\tiny{$G_{i_3}$}};
\draw [-stealth,thick]   (dgn1-gating-3.south) -- (dgn1-galu-3.west);

\node[] (dgn1-gating-2) at (-5.75,-1.5){\tiny{$G_{i_2}$}};
\draw [-stealth,thick]   (dgn1-gating-2.south) -- (dgn1-galu-2.west);

\node[] (dgn1-gating-1) at (-7.25,-1.5){\tiny{$G_{i_1}$}};
\draw [-stealth,thick]   (dgn1-gating-1.south) -- (dgn1-galu-1.west);

\draw [-,thick,color=red]   (dgn1-gating-1-up.south) --(dgn1-gating-2.north);

\draw [-,thick,color=blue]   (dgn1-gating-2-up.south) --(dgn1-gating-3.north);

\draw [-,thick,color=green]   (dgn1-gating-3-up.south) --(dgn1-gating-1.north);

\draw [-,thick,color=brown]   (dgn1-gating-4-up.south) --(dgn1-gating-4.north);

%%%%%%%%%%%%%%%%%%%%%%%%%%%%%%%%%%%%%%%%%%%%%%%%%%%%%%%%%%%%%%%%%
	
\end{tikzpicture}
}
\end{minipage}
\begin{minipage}{0.48\columnwidth}
\centering
\resizebox{0.99\columnwidth}{!}{
\begin{tikzpicture}
%\node []  (fntext)at (-4.625,-3.5) {CNN-GAP-DLGN};

%\node []  (output) at (7.5,1.5) {$\hat{y}(x)$};

\node [] (dgn1-f-c4) at (-3.5,1.5){\tiny{$C^{\text{f}}_4$}};
\node [] (dgn1-f-c3) at (-5,1.5){\tiny{$C^{\text{f}}_3$}};
\node [] (dgn1-f-c2) at (-6.5,1.5){\tiny{$C^{\text{f}}_2$}};
\node [] (dgn1-f-c1) at (-8,1.5){\tiny{$C^{\text{f}}_1$}};
\node [] (dgn1-input-f) at (-9,1.5){$x^{\text{f}}$};
\draw [-stealth,thick]   (dgn1-f-c3.east) -- (dgn1-f-c4.west);
\draw [-stealth,thick]   (dgn1-f-c2.east) -- (dgn1-f-c3.west);
\draw [-stealth,thick]   (dgn1-f-c1.east) -- (dgn1-f-c2.west);
\draw [-stealth,thick]   (dgn1-input-f.east) -- (dgn1-f-c1.west);

\node []  (dgn1-output) at (-0.25,-2.5) {$\hat{y}(x)$};

\node [] (dgn1-smax) at (-1.25,-2.5){\tiny{FC}};
\draw [-stealth,thick]   (dgn1-smax.east)--(dgn1-output.west);

\node [rotate=-90] (dgn1-gap) at (-2,-2.5){\tiny{GAP}};
\draw [-stealth,thick]   (dgn1-gap.north)--(dgn1-smax.west);

\node [rotate=-90] (dgn1-galu-4) at (-2.75,-2.5){\tiny{GaLU}};
\draw [-stealth,thick]   (dgn1-galu-4.north)--(dgn1-gap.south);

\node [] (dgn1-v-c4) at (-3.5,-2.5){\tiny{$C^{\text{v}}_4$}};
\draw [-stealth,thick]   (dgn1-v-c4.east) -- (dgn1-galu-4.south);

\node [rotate=-90] (dgn1-galu-3) at (-4.25,-2.5){\tiny{GaLU}};
\draw [-stealth,thick]   (dgn1-galu-3.north) -- (dgn1-v-c4.west);

\node [] (dgn1-v-c3) at (-5,-2.5){\tiny{$C^{\text{v}}_3$}};
\draw [-stealth,thick]   (dgn1-v-c3.east) -- (dgn1-galu-3.south);

\node [rotate=-90] (dgn1-galu-2) at (-5.75,-2.5){\tiny{GaLU}};
\draw [-stealth,thick]   (dgn1-galu-2.north) -- (dgn1-v-c3.west);

\node [] (dgn1-v-c2) at (-6.5,-2.5){\tiny{$C^{\text{v}}_2$}};
\draw [-stealth,thick]   (dgn1-v-c2.east) -- (dgn1-galu-2.south);

\node [rotate=-90] (dgn1-galu-1) at (-7.25,-2.5){\tiny{GaLU}};
\draw [-stealth,thick]   (dgn1-galu-1.north) -- (dgn1-v-c2.west);

\node [] (dgn1-v-c1) at (-8,-2.5){\tiny{$C^{\text{v}}_1$}};
\draw [-stealth,thick]   (dgn1-v-c1.east) -- (dgn1-galu-1.south);

\node [] (dgn1-v-input) at (-9,-2.5){$x^{\text{v}}$};

\draw [-stealth,thick]   (dgn1-v-input.east) -- (dgn1-v-c1.west);

\node[] (dgn1-gating-4-up) at (-2.75,0.5){\tiny{$G_{4}$}};
\draw [-stealth,thick]   (dgn1-f-c4.east) to[out=-90,in=90] (dgn1-gating-4-up.north);

\node[] (dgn1-gating-3-up) at (-4.25,0.5){\tiny{$G_{3}$}};
\draw [-stealth,thick]   (dgn1-f-c3.east) to[out=-90,in=90] (dgn1-gating-3-up.north);

\node[] (dgn1-gating-2-up) at (-5.75,0.5){\tiny{$G_{2}$}};
\draw [-stealth,thick]   (dgn1-f-c2.east) to[out=-90,in=90] (dgn1-gating-2-up.north);

\node[] (dgn1-gating-1-up) at (-7.25,0.5){\tiny{$G_{1}$}};
\draw [-stealth,thick]   (dgn1-f-c1.east) to[out=-90,in=90] (dgn1-gating-1-up.north);

\node[] (dgn1-gating-4) at (-2.75,-1.5){\tiny{$G_{i_4}$}};
\draw [-stealth,thick]   (dgn1-gating-4.south) -- (dgn1-galu-4.west);

\node[] (dgn1-gating-3) at (-4.25,-1.5){\tiny{$G_{i_3}$}};
\draw [-stealth,thick]   (dgn1-gating-3.south) -- (dgn1-galu-3.west);

\node[] (dgn1-gating-2) at (-5.75,-1.5){\tiny{$G_{i_2}$}};
\draw [-stealth,thick]   (dgn1-gating-2.south) -- (dgn1-galu-2.west);

\node[] (dgn1-gating-1) at (-7.25,-1.5){\tiny{$G_{i_1}$}};
\draw [-stealth,thick]   (dgn1-gating-1.south) -- (dgn1-galu-1.west);

\draw [-,thick,color=red]   (dgn1-gating-1-up.south) --(dgn1-gating-3.north);

\draw [-,thick,color=blue]   (dgn1-gating-2-up.south) --(dgn1-gating-2.north);

\draw [-,thick,color=green]   (dgn1-gating-3-up.south) --(dgn1-gating-1.north);

\draw [-,thick,color=brown]   (dgn1-gating-4-up.south) --(dgn1-gating-4.north);

%%%%%%%%%%%%%%%%%%%%%%%%%%%%%%%%%%%%%%%%%%%%%%%%%%%%%%%%%%%%%%%%%
	
\end{tikzpicture}
}
\end{minipage}

\begin{minipage}{0.48\columnwidth}
\centering
\resizebox{0.99\columnwidth}{!}{
\begin{tikzpicture}
%\node []  (fntext)at (-4.625,-3.5) {CNN-GAP-DLGN};

%\node []  (output) at (7.5,1.5) {$\hat{y}(x)$};

\node [] (dgn1-f-c4) at (-3.5,1.5){\tiny{$C^{\text{f}}_4$}};
\node [] (dgn1-f-c3) at (-5,1.5){\tiny{$C^{\text{f}}_3$}};
\node [] (dgn1-f-c2) at (-6.5,1.5){\tiny{$C^{\text{f}}_2$}};
\node [] (dgn1-f-c1) at (-8,1.5){\tiny{$C^{\text{f}}_1$}};
\node [] (dgn1-input-f) at (-9,1.5){$x^{\text{f}}$};
\draw [-stealth,thick]   (dgn1-f-c3.east) -- (dgn1-f-c4.west);
\draw [-stealth,thick]   (dgn1-f-c2.east) -- (dgn1-f-c3.west);
\draw [-stealth,thick]   (dgn1-f-c1.east) -- (dgn1-f-c2.west);
\draw [-stealth,thick]   (dgn1-input-f.east) -- (dgn1-f-c1.west);

\node []  (dgn1-output) at (-0.25,-2.5) {$\hat{y}(x)$};

\node [] (dgn1-smax) at (-1.25,-2.5){\tiny{FC}};
\draw [-stealth,thick]   (dgn1-smax.east)--(dgn1-output.west);

\node [rotate=-90] (dgn1-gap) at (-2,-2.5){\tiny{GAP}};
\draw [-stealth,thick]   (dgn1-gap.north)--(dgn1-smax.west);

\node [rotate=-90] (dgn1-galu-4) at (-2.75,-2.5){\tiny{GaLU}};
\draw [-stealth,thick]   (dgn1-galu-4.north)--(dgn1-gap.south);

\node [] (dgn1-v-c4) at (-3.5,-2.5){\tiny{$C^{\text{v}}_4$}};
\draw [-stealth,thick]   (dgn1-v-c4.east) -- (dgn1-galu-4.south);

\node [rotate=-90] (dgn1-galu-3) at (-4.25,-2.5){\tiny{GaLU}};
\draw [-stealth,thick]   (dgn1-galu-3.north) -- (dgn1-v-c4.west);

\node [] (dgn1-v-c3) at (-5,-2.5){\tiny{$C^{\text{v}}_3$}};
\draw [-stealth,thick]   (dgn1-v-c3.east) -- (dgn1-galu-3.south);

\node [rotate=-90] (dgn1-galu-2) at (-5.75,-2.5){\tiny{GaLU}};
\draw [-stealth,thick]   (dgn1-galu-2.north) -- (dgn1-v-c3.west);

\node [] (dgn1-v-c2) at (-6.5,-2.5){\tiny{$C^{\text{v}}_2$}};
\draw [-stealth,thick]   (dgn1-v-c2.east) -- (dgn1-galu-2.south);

\node [rotate=-90] (dgn1-galu-1) at (-7.25,-2.5){\tiny{GaLU}};
\draw [-stealth,thick]   (dgn1-galu-1.north) -- (dgn1-v-c2.west);

\node [] (dgn1-v-c1) at (-8,-2.5){\tiny{$C^{\text{v}}_1$}};
\draw [-stealth,thick]   (dgn1-v-c1.east) -- (dgn1-galu-1.south);

\node [] (dgn1-v-input) at (-9,-2.5){$x^{\text{v}}$};

\draw [-stealth,thick]   (dgn1-v-input.east) -- (dgn1-v-c1.west);

\node[] (dgn1-gating-4-up) at (-2.75,0.5){\tiny{$G_{4}$}};
\draw [-stealth,thick]   (dgn1-f-c4.east) to[out=-90,in=90] (dgn1-gating-4-up.north);

\node[] (dgn1-gating-3-up) at (-4.25,0.5){\tiny{$G_{3}$}};
\draw [-stealth,thick]   (dgn1-f-c3.east) to[out=-90,in=90] (dgn1-gating-3-up.north);

\node[] (dgn1-gating-2-up) at (-5.75,0.5){\tiny{$G_{2}$}};
\draw [-stealth,thick]   (dgn1-f-c2.east) to[out=-90,in=90] (dgn1-gating-2-up.north);

\node[] (dgn1-gating-1-up) at (-7.25,0.5){\tiny{$G_{1}$}};
\draw [-stealth,thick]   (dgn1-f-c1.east) to[out=-90,in=90] (dgn1-gating-1-up.north);

\node[] (dgn1-gating-4) at (-2.75,-1.5){\tiny{$G_{i_4}$}};
\draw [-stealth,thick]   (dgn1-gating-4.south) -- (dgn1-galu-4.west);

\node[] (dgn1-gating-3) at (-4.25,-1.5){\tiny{$G_{i_3}$}};
\draw [-stealth,thick]   (dgn1-gating-3.south) -- (dgn1-galu-3.west);

\node[] (dgn1-gating-2) at (-5.75,-1.5){\tiny{$G_{i_2}$}};
\draw [-stealth,thick]   (dgn1-gating-2.south) -- (dgn1-galu-2.west);

\node[] (dgn1-gating-1) at (-7.25,-1.5){\tiny{$G_{i_1}$}};
\draw [-stealth,thick]   (dgn1-gating-1.south) -- (dgn1-galu-1.west);

\draw [-,thick,color=red]   (dgn1-gating-1-up.south) --(dgn1-gating-2.north);

\draw [-,thick,color=blue]   (dgn1-gating-2-up.south) --(dgn1-gating-3.north);

\draw [-,thick,color=green]   (dgn1-gating-3-up.south) --(dgn1-gating-4.north);

\draw [-,thick,color=brown]   (dgn1-gating-4-up.south) --(dgn1-gating-1.north);

%%%%%%%%%%%%%%%%%%%%%%%%%%%%%%%%%%%%%%%%%%%%%%%%%%%%%%%%%%%%%%%%%
	
\end{tikzpicture}
}
\end{minipage}
\begin{minipage}{0.48\columnwidth}
\centering
\resizebox{0.99\columnwidth}{!}{
\begin{tikzpicture}
%\node []  (fntext)at (-4.625,-3.5) {CNN-GAP-DLGN};

%\node []  (output) at (7.5,1.5) {$\hat{y}(x)$};

\node [] (dgn1-f-c4) at (-3.5,1.5){\tiny{$C^{\text{f}}_4$}};
\node [] (dgn1-f-c3) at (-5,1.5){\tiny{$C^{\text{f}}_3$}};
\node [] (dgn1-f-c2) at (-6.5,1.5){\tiny{$C^{\text{f}}_2$}};
\node [] (dgn1-f-c1) at (-8,1.5){\tiny{$C^{\text{f}}_1$}};
\node [] (dgn1-input-f) at (-9,1.5){$x^{\text{f}}$};
\draw [-stealth,thick]   (dgn1-f-c3.east) -- (dgn1-f-c4.west);
\draw [-stealth,thick]   (dgn1-f-c2.east) -- (dgn1-f-c3.west);
\draw [-stealth,thick]   (dgn1-f-c1.east) -- (dgn1-f-c2.west);
\draw [-stealth,thick]   (dgn1-input-f.east) -- (dgn1-f-c1.west);

\node []  (dgn1-output) at (-0.25,-2.5) {$\hat{y}(x)$};

\node [] (dgn1-smax) at (-1.25,-2.5){\tiny{FC}};
\draw [-stealth,thick]   (dgn1-smax.east)--(dgn1-output.west);

\node [rotate=-90] (dgn1-gap) at (-2,-2.5){\tiny{GAP}};
\draw [-stealth,thick]   (dgn1-gap.north)--(dgn1-smax.west);

\node [rotate=-90] (dgn1-galu-4) at (-2.75,-2.5){\tiny{GaLU}};
\draw [-stealth,thick]   (dgn1-galu-4.north)--(dgn1-gap.south);

\node [] (dgn1-v-c4) at (-3.5,-2.5){\tiny{$C^{\text{v}}_4$}};
\draw [-stealth,thick]   (dgn1-v-c4.east) -- (dgn1-galu-4.south);

\node [rotate=-90] (dgn1-galu-3) at (-4.25,-2.5){\tiny{GaLU}};
\draw [-stealth,thick]   (dgn1-galu-3.north) -- (dgn1-v-c4.west);

\node [] (dgn1-v-c3) at (-5,-2.5){\tiny{$C^{\text{v}}_3$}};
\draw [-stealth,thick]   (dgn1-v-c3.east) -- (dgn1-galu-3.south);

\node [rotate=-90] (dgn1-galu-2) at (-5.75,-2.5){\tiny{GaLU}};
\draw [-stealth,thick]   (dgn1-galu-2.north) -- (dgn1-v-c3.west);

\node [] (dgn1-v-c2) at (-6.5,-2.5){\tiny{$C^{\text{v}}_2$}};
\draw [-stealth,thick]   (dgn1-v-c2.east) -- (dgn1-galu-2.south);

\node [rotate=-90] (dgn1-galu-1) at (-7.25,-2.5){\tiny{GaLU}};
\draw [-stealth,thick]   (dgn1-galu-1.north) -- (dgn1-v-c2.west);

\node [] (dgn1-v-c1) at (-8,-2.5){\tiny{$C^{\text{v}}_1$}};
\draw [-stealth,thick]   (dgn1-v-c1.east) -- (dgn1-galu-1.south);

\node [] (dgn1-v-input) at (-9,-2.5){$x^{\text{v}}$};

\draw [-stealth,thick]   (dgn1-v-input.east) -- (dgn1-v-c1.west);

\node[] (dgn1-gating-4-up) at (-2.75,0.5){\tiny{$G_{4}$}};
\draw [-stealth,thick]   (dgn1-f-c4.east) to[out=-90,in=90] (dgn1-gating-4-up.north);

\node[] (dgn1-gating-3-up) at (-4.25,0.5){\tiny{$G_{3}$}};
\draw [-stealth,thick]   (dgn1-f-c3.east) to[out=-90,in=90] (dgn1-gating-3-up.north);

\node[] (dgn1-gating-2-up) at (-5.75,0.5){\tiny{$G_{2}$}};
\draw [-stealth,thick]   (dgn1-f-c2.east) to[out=-90,in=90] (dgn1-gating-2-up.north);

\node[] (dgn1-gating-1-up) at (-7.25,0.5){\tiny{$G_{1}$}};
\draw [-stealth,thick]   (dgn1-f-c1.east) to[out=-90,in=90] (dgn1-gating-1-up.north);

\node[] (dgn1-gating-4) at (-2.75,-1.5){\tiny{$G_{i_4}$}};
\draw [-stealth,thick]   (dgn1-gating-4.south) -- (dgn1-galu-4.west);

\node[] (dgn1-gating-3) at (-4.25,-1.5){\tiny{$G_{i_3}$}};
\draw [-stealth,thick]   (dgn1-gating-3.south) -- (dgn1-galu-3.west);

\node[] (dgn1-gating-2) at (-5.75,-1.5){\tiny{$G_{i_2}$}};
\draw [-stealth,thick]   (dgn1-gating-2.south) -- (dgn1-galu-2.west);

\node[] (dgn1-gating-1) at (-7.25,-1.5){\tiny{$G_{i_1}$}};
\draw [-stealth,thick]   (dgn1-gating-1.south) -- (dgn1-galu-1.west);

\draw [-,thick,color=red]   (dgn1-gating-1-up.south) --(dgn1-gating-4.north);

\draw [-,thick,color=blue]   (dgn1-gating-2-up.south) --(dgn1-gating-3.north);

\draw [-,thick,color=green]   (dgn1-gating-3-up.south) --(dgn1-gating-2.north);

\draw [-,thick,color=brown]   (dgn1-gating-4-up.south) --(dgn1-gating-1.north);

%%%%%%%%%%%%%%%%%%%%%%%%%%%%%%%%%%%%%%%%%%%%%%%%%%%%%%%%%%%%%%%%%
	
\end{tikzpicture}
}
\end{minipage}

\begin{minipage}{0.48\columnwidth}
\centering
\resizebox{0.99\columnwidth}{!}{
\begin{tikzpicture}
%\node []  (fntext)at (-4.625,-3.5) {CNN-GAP-DLGN};

%\node []  (output) at (7.5,1.5) {$\hat{y}(x)$};

\node [] (dgn1-f-c4) at (-3.5,1.5){\tiny{$C^{\text{f}}_4$}};
\node [] (dgn1-f-c3) at (-5,1.5){\tiny{$C^{\text{f}}_3$}};
\node [] (dgn1-f-c2) at (-6.5,1.5){\tiny{$C^{\text{f}}_2$}};
\node [] (dgn1-f-c1) at (-8,1.5){\tiny{$C^{\text{f}}_1$}};
\node [] (dgn1-input-f) at (-9,1.5){$x^{\text{f}}$};
\draw [-stealth,thick]   (dgn1-f-c3.east) -- (dgn1-f-c4.west);
\draw [-stealth,thick]   (dgn1-f-c2.east) -- (dgn1-f-c3.west);
\draw [-stealth,thick]   (dgn1-f-c1.east) -- (dgn1-f-c2.west);
\draw [-stealth,thick]   (dgn1-input-f.east) -- (dgn1-f-c1.west);

\node []  (dgn1-output) at (-0.25,-2.5) {$\hat{y}(x)$};

\node [] (dgn1-smax) at (-1.25,-2.5){\tiny{FC}};
\draw [-stealth,thick]   (dgn1-smax.east)--(dgn1-output.west);

\node [rotate=-90] (dgn1-gap) at (-2,-2.5){\tiny{GAP}};
\draw [-stealth,thick]   (dgn1-gap.north)--(dgn1-smax.west);

\node [rotate=-90] (dgn1-galu-4) at (-2.75,-2.5){\tiny{GaLU}};
\draw [-stealth,thick]   (dgn1-galu-4.north)--(dgn1-gap.south);

\node [] (dgn1-v-c4) at (-3.5,-2.5){\tiny{$C^{\text{v}}_4$}};
\draw [-stealth,thick]   (dgn1-v-c4.east) -- (dgn1-galu-4.south);

\node [rotate=-90] (dgn1-galu-3) at (-4.25,-2.5){\tiny{GaLU}};
\draw [-stealth,thick]   (dgn1-galu-3.north) -- (dgn1-v-c4.west);

\node [] (dgn1-v-c3) at (-5,-2.5){\tiny{$C^{\text{v}}_3$}};
\draw [-stealth,thick]   (dgn1-v-c3.east) -- (dgn1-galu-3.south);

\node [rotate=-90] (dgn1-galu-2) at (-5.75,-2.5){\tiny{GaLU}};
\draw [-stealth,thick]   (dgn1-galu-2.north) -- (dgn1-v-c3.west);

\node [] (dgn1-v-c2) at (-6.5,-2.5){\tiny{$C^{\text{v}}_2$}};
\draw [-stealth,thick]   (dgn1-v-c2.east) -- (dgn1-galu-2.south);

\node [rotate=-90] (dgn1-galu-1) at (-7.25,-2.5){\tiny{GaLU}};
\draw [-stealth,thick]   (dgn1-galu-1.north) -- (dgn1-v-c2.west);

\node [] (dgn1-v-c1) at (-8,-2.5){\tiny{$C^{\text{v}}_1$}};
\draw [-stealth,thick]   (dgn1-v-c1.east) -- (dgn1-galu-1.south);

\node [] (dgn1-v-input) at (-9,-2.5){$x^{\text{v}}$};

\draw [-stealth,thick]   (dgn1-v-input.east) -- (dgn1-v-c1.west);

\node[] (dgn1-gating-4-up) at (-2.75,0.5){\tiny{$G_{4}$}};
\draw [-stealth,thick]   (dgn1-f-c4.east) to[out=-90,in=90] (dgn1-gating-4-up.north);

\node[] (dgn1-gating-3-up) at (-4.25,0.5){\tiny{$G_{3}$}};
\draw [-stealth,thick]   (dgn1-f-c3.east) to[out=-90,in=90] (dgn1-gating-3-up.north);

\node[] (dgn1-gating-2-up) at (-5.75,0.5){\tiny{$G_{2}$}};
\draw [-stealth,thick]   (dgn1-f-c2.east) to[out=-90,in=90] (dgn1-gating-2-up.north);

\node[] (dgn1-gating-1-up) at (-7.25,0.5){\tiny{$G_{1}$}};
\draw [-stealth,thick]   (dgn1-f-c1.east) to[out=-90,in=90] (dgn1-gating-1-up.north);

\node[] (dgn1-gating-4) at (-2.75,-1.5){\tiny{$G_{i_4}$}};
\draw [-stealth,thick]   (dgn1-gating-4.south) -- (dgn1-galu-4.west);

\node[] (dgn1-gating-3) at (-4.25,-1.5){\tiny{$G_{i_3}$}};
\draw [-stealth,thick]   (dgn1-gating-3.south) -- (dgn1-galu-3.west);

\node[] (dgn1-gating-2) at (-5.75,-1.5){\tiny{$G_{i_2}$}};
\draw [-stealth,thick]   (dgn1-gating-2.south) -- (dgn1-galu-2.west);

\node[] (dgn1-gating-1) at (-7.25,-1.5){\tiny{$G_{i_1}$}};
\draw [-stealth,thick]   (dgn1-gating-1.south) -- (dgn1-galu-1.west);

\draw [-,thick,color=red]   (dgn1-gating-1-up.south) --(dgn1-gating-3.north);

\draw [-,thick,color=blue]   (dgn1-gating-2-up.south) --(dgn1-gating-4.north);

\draw [-,thick,color=green]   (dgn1-gating-3-up.south) --(dgn1-gating-2.north);

\draw [-,thick,color=brown]   (dgn1-gating-4-up.south) --(dgn1-gating-1.north);

%%%%%%%%%%%%%%%%%%%%%%%%%%%%%%%%%%%%%%%%%%%%%%%%%%%%%%%%%%%%%%%%%
	
\end{tikzpicture}
}
\end{minipage}
\begin{minipage}{0.48\columnwidth}
\centering
\resizebox{0.99\columnwidth}{!}{
\begin{tikzpicture}
%\node []  (fntext)at (-4.625,-3.5) {CNN-GAP-DLGN};

%\node []  (output) at (7.5,1.5) {$\hat{y}(x)$};

\node [] (dgn1-f-c4) at (-3.5,1.5){\tiny{$C^{\text{f}}_4$}};
\node [] (dgn1-f-c3) at (-5,1.5){\tiny{$C^{\text{f}}_3$}};
\node [] (dgn1-f-c2) at (-6.5,1.5){\tiny{$C^{\text{f}}_2$}};
\node [] (dgn1-f-c1) at (-8,1.5){\tiny{$C^{\text{f}}_1$}};
\node [] (dgn1-input-f) at (-9,1.5){$x^{\text{f}}$};
\draw [-stealth,thick]   (dgn1-f-c3.east) -- (dgn1-f-c4.west);
\draw [-stealth,thick]   (dgn1-f-c2.east) -- (dgn1-f-c3.west);
\draw [-stealth,thick]   (dgn1-f-c1.east) -- (dgn1-f-c2.west);
\draw [-stealth,thick]   (dgn1-input-f.east) -- (dgn1-f-c1.west);

\node []  (dgn1-output) at (-0.25,-2.5) {$\hat{y}(x)$};

\node [] (dgn1-smax) at (-1.25,-2.5){\tiny{FC}};
\draw [-stealth,thick]   (dgn1-smax.east)--(dgn1-output.west);

\node [rotate=-90] (dgn1-gap) at (-2,-2.5){\tiny{GAP}};
\draw [-stealth,thick]   (dgn1-gap.north)--(dgn1-smax.west);

\node [rotate=-90] (dgn1-galu-4) at (-2.75,-2.5){\tiny{GaLU}};
\draw [-stealth,thick]   (dgn1-galu-4.north)--(dgn1-gap.south);

\node [] (dgn1-v-c4) at (-3.5,-2.5){\tiny{$C^{\text{v}}_4$}};
\draw [-stealth,thick]   (dgn1-v-c4.east) -- (dgn1-galu-4.south);

\node [rotate=-90] (dgn1-galu-3) at (-4.25,-2.5){\tiny{GaLU}};
\draw [-stealth,thick]   (dgn1-galu-3.north) -- (dgn1-v-c4.west);

\node [] (dgn1-v-c3) at (-5,-2.5){\tiny{$C^{\text{v}}_3$}};
\draw [-stealth,thick]   (dgn1-v-c3.east) -- (dgn1-galu-3.south);

\node [rotate=-90] (dgn1-galu-2) at (-5.75,-2.5){\tiny{GaLU}};
\draw [-stealth,thick]   (dgn1-galu-2.north) -- (dgn1-v-c3.west);

\node [] (dgn1-v-c2) at (-6.5,-2.5){\tiny{$C^{\text{v}}_2$}};
\draw [-stealth,thick]   (dgn1-v-c2.east) -- (dgn1-galu-2.south);

\node [rotate=-90] (dgn1-galu-1) at (-7.25,-2.5){\tiny{GaLU}};
\draw [-stealth,thick]   (dgn1-galu-1.north) -- (dgn1-v-c2.west);

\node [] (dgn1-v-c1) at (-8,-2.5){\tiny{$C^{\text{v}}_1$}};
\draw [-stealth,thick]   (dgn1-v-c1.east) -- (dgn1-galu-1.south);

\node [] (dgn1-v-input) at (-9,-2.5){$x^{\text{v}}$};

\draw [-stealth,thick]   (dgn1-v-input.east) -- (dgn1-v-c1.west);

\node[] (dgn1-gating-4-up) at (-2.75,0.5){\tiny{$G_{4}$}};
\draw [-stealth,thick]   (dgn1-f-c4.east) to[out=-90,in=90] (dgn1-gating-4-up.north);

\node[] (dgn1-gating-3-up) at (-4.25,0.5){\tiny{$G_{3}$}};
\draw [-stealth,thick]   (dgn1-f-c3.east) to[out=-90,in=90] (dgn1-gating-3-up.north);

\node[] (dgn1-gating-2-up) at (-5.75,0.5){\tiny{$G_{2}$}};
\draw [-stealth,thick]   (dgn1-f-c2.east) to[out=-90,in=90] (dgn1-gating-2-up.north);

\node[] (dgn1-gating-1-up) at (-7.25,0.5){\tiny{$G_{1}$}};
\draw [-stealth,thick]   (dgn1-f-c1.east) to[out=-90,in=90] (dgn1-gating-1-up.north);

\node[] (dgn1-gating-4) at (-2.75,-1.5){\tiny{$G_{i_4}$}};
\draw [-stealth,thick]   (dgn1-gating-4.south) -- (dgn1-galu-4.west);

\node[] (dgn1-gating-3) at (-4.25,-1.5){\tiny{$G_{i_3}$}};
\draw [-stealth,thick]   (dgn1-gating-3.south) -- (dgn1-galu-3.west);

\node[] (dgn1-gating-2) at (-5.75,-1.5){\tiny{$G_{i_2}$}};
\draw [-stealth,thick]   (dgn1-gating-2.south) -- (dgn1-galu-2.west);

\node[] (dgn1-gating-1) at (-7.25,-1.5){\tiny{$G_{i_1}$}};
\draw [-stealth,thick]   (dgn1-gating-1.south) -- (dgn1-galu-1.west);

\draw [-,thick,color=red]   (dgn1-gating-1-up.south) --(dgn1-gating-2.north);

\draw [-,thick,color=blue]   (dgn1-gating-2-up.south) --(dgn1-gating-4.north);

\draw [-,thick,color=green]   (dgn1-gating-3-up.south) --(dgn1-gating-3.north);

\draw [-,thick,color=brown]   (dgn1-gating-4-up.south) --(dgn1-gating-1.north);

%%%%%%%%%%%%%%%%%%%%%%%%%%%%%%%%%%%%%%%%%%%%%%%%%%%%%%%%%%%%%%%%%
	
\end{tikzpicture}
}
\end{minipage}

\begin{minipage}{0.48\columnwidth}
\centering
\resizebox{0.99\columnwidth}{!}{
\begin{tikzpicture}
%\node []  (fntext)at (-4.625,-3.5) {CNN-GAP-DLGN};

%\node []  (output) at (7.5,1.5) {$\hat{y}(x)$};

\node [] (dgn1-f-c4) at (-3.5,1.5){\tiny{$C^{\text{f}}_4$}};
\node [] (dgn1-f-c3) at (-5,1.5){\tiny{$C^{\text{f}}_3$}};
\node [] (dgn1-f-c2) at (-6.5,1.5){\tiny{$C^{\text{f}}_2$}};
\node [] (dgn1-f-c1) at (-8,1.5){\tiny{$C^{\text{f}}_1$}};
\node [] (dgn1-input-f) at (-9,1.5){$x^{\text{f}}$};
\draw [-stealth,thick]   (dgn1-f-c3.east) -- (dgn1-f-c4.west);
\draw [-stealth,thick]   (dgn1-f-c2.east) -- (dgn1-f-c3.west);
\draw [-stealth,thick]   (dgn1-f-c1.east) -- (dgn1-f-c2.west);
\draw [-stealth,thick]   (dgn1-input-f.east) -- (dgn1-f-c1.west);

\node []  (dgn1-output) at (-0.25,-2.5) {$\hat{y}(x)$};

\node [] (dgn1-smax) at (-1.25,-2.5){\tiny{FC}};
\draw [-stealth,thick]   (dgn1-smax.east)--(dgn1-output.west);

\node [rotate=-90] (dgn1-gap) at (-2,-2.5){\tiny{GAP}};
\draw [-stealth,thick]   (dgn1-gap.north)--(dgn1-smax.west);

\node [rotate=-90] (dgn1-galu-4) at (-2.75,-2.5){\tiny{GaLU}};
\draw [-stealth,thick]   (dgn1-galu-4.north)--(dgn1-gap.south);

\node [] (dgn1-v-c4) at (-3.5,-2.5){\tiny{$C^{\text{v}}_4$}};
\draw [-stealth,thick]   (dgn1-v-c4.east) -- (dgn1-galu-4.south);

\node [rotate=-90] (dgn1-galu-3) at (-4.25,-2.5){\tiny{GaLU}};
\draw [-stealth,thick]   (dgn1-galu-3.north) -- (dgn1-v-c4.west);

\node [] (dgn1-v-c3) at (-5,-2.5){\tiny{$C^{\text{v}}_3$}};
\draw [-stealth,thick]   (dgn1-v-c3.east) -- (dgn1-galu-3.south);

\node [rotate=-90] (dgn1-galu-2) at (-5.75,-2.5){\tiny{GaLU}};
\draw [-stealth,thick]   (dgn1-galu-2.north) -- (dgn1-v-c3.west);

\node [] (dgn1-v-c2) at (-6.5,-2.5){\tiny{$C^{\text{v}}_2$}};
\draw [-stealth,thick]   (dgn1-v-c2.east) -- (dgn1-galu-2.south);

\node [rotate=-90] (dgn1-galu-1) at (-7.25,-2.5){\tiny{GaLU}};
\draw [-stealth,thick]   (dgn1-galu-1.north) -- (dgn1-v-c2.west);

\node [] (dgn1-v-c1) at (-8,-2.5){\tiny{$C^{\text{v}}_1$}};
\draw [-stealth,thick]   (dgn1-v-c1.east) -- (dgn1-galu-1.south);

\node [] (dgn1-v-input) at (-9,-2.5){$x^{\text{v}}$};

\draw [-stealth,thick]   (dgn1-v-input.east) -- (dgn1-v-c1.west);

\node[] (dgn1-gating-4-up) at (-2.75,0.5){\tiny{$G_{4}$}};
\draw [-stealth,thick]   (dgn1-f-c4.east) to[out=-90,in=90] (dgn1-gating-4-up.north);

\node[] (dgn1-gating-3-up) at (-4.25,0.5){\tiny{$G_{3}$}};
\draw [-stealth,thick]   (dgn1-f-c3.east) to[out=-90,in=90] (dgn1-gating-3-up.north);

\node[] (dgn1-gating-2-up) at (-5.75,0.5){\tiny{$G_{2}$}};
\draw [-stealth,thick]   (dgn1-f-c2.east) to[out=-90,in=90] (dgn1-gating-2-up.north);

\node[] (dgn1-gating-1-up) at (-7.25,0.5){\tiny{$G_{1}$}};
\draw [-stealth,thick]   (dgn1-f-c1.east) to[out=-90,in=90] (dgn1-gating-1-up.north);

\node[] (dgn1-gating-4) at (-2.75,-1.5){\tiny{$G_{i_4}$}};
\draw [-stealth,thick]   (dgn1-gating-4.south) -- (dgn1-galu-4.west);

\node[] (dgn1-gating-3) at (-4.25,-1.5){\tiny{$G_{i_3}$}};
\draw [-stealth,thick]   (dgn1-gating-3.south) -- (dgn1-galu-3.west);

\node[] (dgn1-gating-2) at (-5.75,-1.5){\tiny{$G_{i_2}$}};
\draw [-stealth,thick]   (dgn1-gating-2.south) -- (dgn1-galu-2.west);

\node[] (dgn1-gating-1) at (-7.25,-1.5){\tiny{$G_{i_1}$}};
\draw [-stealth,thick]   (dgn1-gating-1.south) -- (dgn1-galu-1.west);

\draw [-,thick,color=red]   (dgn1-gating-1-up.south) --(dgn1-gating-4.north);

\draw [-,thick,color=blue]   (dgn1-gating-2-up.south) --(dgn1-gating-2.north);

\draw [-,thick,color=green]   (dgn1-gating-3-up.south) --(dgn1-gating-3.north);

\draw [-,thick,color=brown]   (dgn1-gating-4-up.south) --(dgn1-gating-1.north);

%%%%%%%%%%%%%%%%%%%%%%%%%%%%%%%%%%%%%%%%%%%%%%%%%%%%%%%%%%%%%%%%%
	
\end{tikzpicture}
}
\end{minipage}
\begin{minipage}{0.48\columnwidth}
\centering
\resizebox{0.99\columnwidth}{!}{
\begin{tikzpicture}
%\node []  (fntext)at (-4.625,-3.5) {CNN-GAP-DLGN};

%\node []  (output) at (7.5,1.5) {$\hat{y}(x)$};

\node [] (dgn1-f-c4) at (-3.5,1.5){\tiny{$C^{\text{f}}_4$}};
\node [] (dgn1-f-c3) at (-5,1.5){\tiny{$C^{\text{f}}_3$}};
\node [] (dgn1-f-c2) at (-6.5,1.5){\tiny{$C^{\text{f}}_2$}};
\node [] (dgn1-f-c1) at (-8,1.5){\tiny{$C^{\text{f}}_1$}};
\node [] (dgn1-input-f) at (-9,1.5){$x^{\text{f}}$};
\draw [-stealth,thick]   (dgn1-f-c3.east) -- (dgn1-f-c4.west);
\draw [-stealth,thick]   (dgn1-f-c2.east) -- (dgn1-f-c3.west);
\draw [-stealth,thick]   (dgn1-f-c1.east) -- (dgn1-f-c2.west);
\draw [-stealth,thick]   (dgn1-input-f.east) -- (dgn1-f-c1.west);

\node []  (dgn1-output) at (-0.25,-2.5) {$\hat{y}(x)$};

\node [] (dgn1-smax) at (-1.25,-2.5){\tiny{FC}};
\draw [-stealth,thick]   (dgn1-smax.east)--(dgn1-output.west);

\node [rotate=-90] (dgn1-gap) at (-2,-2.5){\tiny{GAP}};
\draw [-stealth,thick]   (dgn1-gap.north)--(dgn1-smax.west);

\node [rotate=-90] (dgn1-galu-4) at (-2.75,-2.5){\tiny{GaLU}};
\draw [-stealth,thick]   (dgn1-galu-4.north)--(dgn1-gap.south);

\node [] (dgn1-v-c4) at (-3.5,-2.5){\tiny{$C^{\text{v}}_4$}};
\draw [-stealth,thick]   (dgn1-v-c4.east) -- (dgn1-galu-4.south);

\node [rotate=-90] (dgn1-galu-3) at (-4.25,-2.5){\tiny{GaLU}};
\draw [-stealth,thick]   (dgn1-galu-3.north) -- (dgn1-v-c4.west);

\node [] (dgn1-v-c3) at (-5,-2.5){\tiny{$C^{\text{v}}_3$}};
\draw [-stealth,thick]   (dgn1-v-c3.east) -- (dgn1-galu-3.south);

\node [rotate=-90] (dgn1-galu-2) at (-5.75,-2.5){\tiny{GaLU}};
\draw [-stealth,thick]   (dgn1-galu-2.north) -- (dgn1-v-c3.west);

\node [] (dgn1-v-c2) at (-6.5,-2.5){\tiny{$C^{\text{v}}_2$}};
\draw [-stealth,thick]   (dgn1-v-c2.east) -- (dgn1-galu-2.south);

\node [rotate=-90] (dgn1-galu-1) at (-7.25,-2.5){\tiny{GaLU}};
\draw [-stealth,thick]   (dgn1-galu-1.north) -- (dgn1-v-c2.west);

\node [] (dgn1-v-c1) at (-8,-2.5){\tiny{$C^{\text{v}}_1$}};
\draw [-stealth,thick]   (dgn1-v-c1.east) -- (dgn1-galu-1.south);

\node [] (dgn1-v-input) at (-9,-2.5){$x^{\text{v}}$};

\draw [-stealth,thick]   (dgn1-v-input.east) -- (dgn1-v-c1.west);

\node[] (dgn1-gating-4-up) at (-2.75,0.5){\tiny{$G_{4}$}};
\draw [-stealth,thick]   (dgn1-f-c4.east) to[out=-90,in=90] (dgn1-gating-4-up.north);

\node[] (dgn1-gating-3-up) at (-4.25,0.5){\tiny{$G_{3}$}};
\draw [-stealth,thick]   (dgn1-f-c3.east) to[out=-90,in=90] (dgn1-gating-3-up.north);

\node[] (dgn1-gating-2-up) at (-5.75,0.5){\tiny{$G_{2}$}};
\draw [-stealth,thick]   (dgn1-f-c2.east) to[out=-90,in=90] (dgn1-gating-2-up.north);

\node[] (dgn1-gating-1-up) at (-7.25,0.5){\tiny{$G_{1}$}};
\draw [-stealth,thick]   (dgn1-f-c1.east) to[out=-90,in=90] (dgn1-gating-1-up.north);

\node[] (dgn1-gating-4) at (-2.75,-1.5){\tiny{$G_{i_4}$}};
\draw [-stealth,thick]   (dgn1-gating-4.south) -- (dgn1-galu-4.west);

\node[] (dgn1-gating-3) at (-4.25,-1.5){\tiny{$G_{i_3}$}};
\draw [-stealth,thick]   (dgn1-gating-3.south) -- (dgn1-galu-3.west);

\node[] (dgn1-gating-2) at (-5.75,-1.5){\tiny{$G_{i_2}$}};
\draw [-stealth,thick]   (dgn1-gating-2.south) -- (dgn1-galu-2.west);

\node[] (dgn1-gating-1) at (-7.25,-1.5){\tiny{$G_{i_1}$}};
\draw [-stealth,thick]   (dgn1-gating-1.south) -- (dgn1-galu-1.west);

\draw [-,thick,color=red]   (dgn1-gating-1-up.south) --(dgn1-gating-4.north);

\draw [-,thick,color=blue]   (dgn1-gating-2-up.south) --(dgn1-gating-1.north);

\draw [-,thick,color=green]   (dgn1-gating-3-up.south) --(dgn1-gating-3.north);

\draw [-,thick,color=brown]   (dgn1-gating-4-up.south) --(dgn1-gating-2.north);

%%%%%%%%%%%%%%%%%%%%%%%%%%%%%%%%%%%%%%%%%%%%%%%%%%%%%%%%%%%%%%%%%
	
\end{tikzpicture}
}
\end{minipage}

\caption{Shows the permutations $1-12$ C4GAP-DLGN in Table I of \Cref{fig:c4gap}. The top left is the identity permutation and is the vanilla model.}
\label{fig:perm1}
\end{figure}

\begin{figure}
\centering
\begin{minipage}{0.48\columnwidth}
\centering
\resizebox{0.99\columnwidth}{!}{
\begin{tikzpicture}
%\node []  (fntext)at (-4.625,-3.5) {CNN-GAP-DLGN};

%\node []  (output) at (7.5,1.5) {$\hat{y}(x)$};

\node [] (dgn1-f-c4) at (-3.5,1.5){\tiny{$C^{\text{f}}_4$}};
\node [] (dgn1-f-c3) at (-5,1.5){\tiny{$C^{\text{f}}_3$}};
\node [] (dgn1-f-c2) at (-6.5,1.5){\tiny{$C^{\text{f}}_2$}};
\node [] (dgn1-f-c1) at (-8,1.5){\tiny{$C^{\text{f}}_1$}};
\node [] (dgn1-input-f) at (-9,1.5){$x^{\text{f}}$};
\draw [-stealth,thick]   (dgn1-f-c3.east) -- (dgn1-f-c4.west);
\draw [-stealth,thick]   (dgn1-f-c2.east) -- (dgn1-f-c3.west);
\draw [-stealth,thick]   (dgn1-f-c1.east) -- (dgn1-f-c2.west);
\draw [-stealth,thick]   (dgn1-input-f.east) -- (dgn1-f-c1.west);

\node []  (dgn1-output) at (-0.25,-2.5) {$\hat{y}(x)$};

\node [] (dgn1-smax) at (-1.25,-2.5){\tiny{FC}};
\draw [-stealth,thick]   (dgn1-smax.east)--(dgn1-output.west);

\node [rotate=-90] (dgn1-gap) at (-2,-2.5){\tiny{GAP}};
\draw [-stealth,thick]   (dgn1-gap.north)--(dgn1-smax.west);

\node [rotate=-90] (dgn1-galu-4) at (-2.75,-2.5){\tiny{GaLU}};
\draw [-stealth,thick]   (dgn1-galu-4.north)--(dgn1-gap.south);

\node [] (dgn1-v-c4) at (-3.5,-2.5){\tiny{$C^{\text{v}}_4$}};
\draw [-stealth,thick]   (dgn1-v-c4.east) -- (dgn1-galu-4.south);

\node [rotate=-90] (dgn1-galu-3) at (-4.25,-2.5){\tiny{GaLU}};
\draw [-stealth,thick]   (dgn1-galu-3.north) -- (dgn1-v-c4.west);

\node [] (dgn1-v-c3) at (-5,-2.5){\tiny{$C^{\text{v}}_3$}};
\draw [-stealth,thick]   (dgn1-v-c3.east) -- (dgn1-galu-3.south);

\node [rotate=-90] (dgn1-galu-2) at (-5.75,-2.5){\tiny{GaLU}};
\draw [-stealth,thick]   (dgn1-galu-2.north) -- (dgn1-v-c3.west);

\node [] (dgn1-v-c2) at (-6.5,-2.5){\tiny{$C^{\text{v}}_2$}};
\draw [-stealth,thick]   (dgn1-v-c2.east) -- (dgn1-galu-2.south);

\node [rotate=-90] (dgn1-galu-1) at (-7.25,-2.5){\tiny{GaLU}};
\draw [-stealth,thick]   (dgn1-galu-1.north) -- (dgn1-v-c2.west);

\node [] (dgn1-v-c1) at (-8,-2.5){\tiny{$C^{\text{v}}_1$}};
\draw [-stealth,thick]   (dgn1-v-c1.east) -- (dgn1-galu-1.south);

\node [] (dgn1-v-input) at (-9,-2.5){$x^{\text{v}}$};

\draw [-stealth,thick]   (dgn1-v-input.east) -- (dgn1-v-c1.west);

\node[] (dgn1-gating-4-up) at (-2.75,0.5){\tiny{$G_{4}$}};
\draw [-stealth,thick]   (dgn1-f-c4.east) to[out=-90,in=90] (dgn1-gating-4-up.north);

\node[] (dgn1-gating-3-up) at (-4.25,0.5){\tiny{$G_{3}$}};
\draw [-stealth,thick]   (dgn1-f-c3.east) to[out=-90,in=90] (dgn1-gating-3-up.north);

\node[] (dgn1-gating-2-up) at (-5.75,0.5){\tiny{$G_{2}$}};
\draw [-stealth,thick]   (dgn1-f-c2.east) to[out=-90,in=90] (dgn1-gating-2-up.north);

\node[] (dgn1-gating-1-up) at (-7.25,0.5){\tiny{$G_{1}$}};
\draw [-stealth,thick]   (dgn1-f-c1.east) to[out=-90,in=90] (dgn1-gating-1-up.north);

\node[] (dgn1-gating-4) at (-2.75,-1.5){\tiny{$G_{i_4}$}};
\draw [-stealth,thick]   (dgn1-gating-4.south) -- (dgn1-galu-4.west);

\node[] (dgn1-gating-3) at (-4.25,-1.5){\tiny{$G_{i_3}$}};
\draw [-stealth,thick]   (dgn1-gating-3.south) -- (dgn1-galu-3.west);

\node[] (dgn1-gating-2) at (-5.75,-1.5){\tiny{$G_{i_2}$}};
\draw [-stealth,thick]   (dgn1-gating-2.south) -- (dgn1-galu-2.west);

\node[] (dgn1-gating-1) at (-7.25,-1.5){\tiny{$G_{i_1}$}};
\draw [-stealth,thick]   (dgn1-gating-1.south) -- (dgn1-galu-1.west);

\draw [-,thick,color=red]   (dgn1-gating-1-up.south) --(dgn1-gating-1.north);

\draw [-,thick,color=blue]   (dgn1-gating-2-up.south) --(dgn1-gating-4.north);

\draw [-,thick,color=green]   (dgn1-gating-3-up.south) --(dgn1-gating-3.north);

\draw [-,thick,color=brown]   (dgn1-gating-4-up.south) --(dgn1-gating-2.north);

%%%%%%%%%%%%%%%%%%%%%%%%%%%%%%%%%%%%%%%%%%%%%%%%%%%%%%%%%%%%%%%%%
	
\end{tikzpicture}
}
\end{minipage}
\begin{minipage}{0.48\columnwidth}
\centering
\resizebox{0.99\columnwidth}{!}{
\begin{tikzpicture}
%\node []  (fntext)at (-4.625,-3.5) {CNN-GAP-DLGN};

%\node []  (output) at (7.5,1.5) {$\hat{y}(x)$};

\node [] (dgn1-f-c4) at (-3.5,1.5){\tiny{$C^{\text{f}}_4$}};
\node [] (dgn1-f-c3) at (-5,1.5){\tiny{$C^{\text{f}}_3$}};
\node [] (dgn1-f-c2) at (-6.5,1.5){\tiny{$C^{\text{f}}_2$}};
\node [] (dgn1-f-c1) at (-8,1.5){\tiny{$C^{\text{f}}_1$}};
\node [] (dgn1-input-f) at (-9,1.5){$x^{\text{f}}$};
\draw [-stealth,thick]   (dgn1-f-c3.east) -- (dgn1-f-c4.west);
\draw [-stealth,thick]   (dgn1-f-c2.east) -- (dgn1-f-c3.west);
\draw [-stealth,thick]   (dgn1-f-c1.east) -- (dgn1-f-c2.west);
\draw [-stealth,thick]   (dgn1-input-f.east) -- (dgn1-f-c1.west);

\node []  (dgn1-output) at (-0.25,-2.5) {$\hat{y}(x)$};

\node [] (dgn1-smax) at (-1.25,-2.5){\tiny{FC}};
\draw [-stealth,thick]   (dgn1-smax.east)--(dgn1-output.west);

\node [rotate=-90] (dgn1-gap) at (-2,-2.5){\tiny{GAP}};
\draw [-stealth,thick]   (dgn1-gap.north)--(dgn1-smax.west);

\node [rotate=-90] (dgn1-galu-4) at (-2.75,-2.5){\tiny{GaLU}};
\draw [-stealth,thick]   (dgn1-galu-4.north)--(dgn1-gap.south);

\node [] (dgn1-v-c4) at (-3.5,-2.5){\tiny{$C^{\text{v}}_4$}};
\draw [-stealth,thick]   (dgn1-v-c4.east) -- (dgn1-galu-4.south);

\node [rotate=-90] (dgn1-galu-3) at (-4.25,-2.5){\tiny{GaLU}};
\draw [-stealth,thick]   (dgn1-galu-3.north) -- (dgn1-v-c4.west);

\node [] (dgn1-v-c3) at (-5,-2.5){\tiny{$C^{\text{v}}_3$}};
\draw [-stealth,thick]   (dgn1-v-c3.east) -- (dgn1-galu-3.south);

\node [rotate=-90] (dgn1-galu-2) at (-5.75,-2.5){\tiny{GaLU}};
\draw [-stealth,thick]   (dgn1-galu-2.north) -- (dgn1-v-c3.west);

\node [] (dgn1-v-c2) at (-6.5,-2.5){\tiny{$C^{\text{v}}_2$}};
\draw [-stealth,thick]   (dgn1-v-c2.east) -- (dgn1-galu-2.south);

\node [rotate=-90] (dgn1-galu-1) at (-7.25,-2.5){\tiny{GaLU}};
\draw [-stealth,thick]   (dgn1-galu-1.north) -- (dgn1-v-c2.west);

\node [] (dgn1-v-c1) at (-8,-2.5){\tiny{$C^{\text{v}}_1$}};
\draw [-stealth,thick]   (dgn1-v-c1.east) -- (dgn1-galu-1.south);

\node [] (dgn1-v-input) at (-9,-2.5){$x^{\text{v}}$};

\draw [-stealth,thick]   (dgn1-v-input.east) -- (dgn1-v-c1.west);

\node[] (dgn1-gating-4-up) at (-2.75,0.5){\tiny{$G_{4}$}};
\draw [-stealth,thick]   (dgn1-f-c4.east) to[out=-90,in=90] (dgn1-gating-4-up.north);

\node[] (dgn1-gating-3-up) at (-4.25,0.5){\tiny{$G_{3}$}};
\draw [-stealth,thick]   (dgn1-f-c3.east) to[out=-90,in=90] (dgn1-gating-3-up.north);

\node[] (dgn1-gating-2-up) at (-5.75,0.5){\tiny{$G_{2}$}};
\draw [-stealth,thick]   (dgn1-f-c2.east) to[out=-90,in=90] (dgn1-gating-2-up.north);

\node[] (dgn1-gating-1-up) at (-7.25,0.5){\tiny{$G_{1}$}};
\draw [-stealth,thick]   (dgn1-f-c1.east) to[out=-90,in=90] (dgn1-gating-1-up.north);

\node[] (dgn1-gating-4) at (-2.75,-1.5){\tiny{$G_{i_4}$}};
\draw [-stealth,thick]   (dgn1-gating-4.south) -- (dgn1-galu-4.west);

\node[] (dgn1-gating-3) at (-4.25,-1.5){\tiny{$G_{i_3}$}};
\draw [-stealth,thick]   (dgn1-gating-3.south) -- (dgn1-galu-3.west);

\node[] (dgn1-gating-2) at (-5.75,-1.5){\tiny{$G_{i_2}$}};
\draw [-stealth,thick]   (dgn1-gating-2.south) -- (dgn1-galu-2.west);

\node[] (dgn1-gating-1) at (-7.25,-1.5){\tiny{$G_{i_1}$}};
\draw [-stealth,thick]   (dgn1-gating-1.south) -- (dgn1-galu-1.west);

\draw [-,thick,color=red]   (dgn1-gating-1-up.south) --(dgn1-gating-3.north);

\draw [-,thick,color=blue]   (dgn1-gating-2-up.south) --(dgn1-gating-4.north);

\draw [-,thick,color=green]   (dgn1-gating-3-up.south) --(dgn1-gating-1.north);

\draw [-,thick,color=brown]   (dgn1-gating-4-up.south) --(dgn1-gating-2.north);

%%%%%%%%%%%%%%%%%%%%%%%%%%%%%%%%%%%%%%%%%%%%%%%%%%%%%%%%%%%%%%%%%
	
\end{tikzpicture}
}
\end{minipage}

\begin{minipage}{0.48\columnwidth}
\centering
\resizebox{0.99\columnwidth}{!}{
\begin{tikzpicture}
%\node []  (fntext)at (-4.625,-3.5) {CNN-GAP-DLGN};

%\node []  (output) at (7.5,1.5) {$\hat{y}(x)$};

\node [] (dgn1-f-c4) at (-3.5,1.5){\tiny{$C^{\text{f}}_4$}};
\node [] (dgn1-f-c3) at (-5,1.5){\tiny{$C^{\text{f}}_3$}};
\node [] (dgn1-f-c2) at (-6.5,1.5){\tiny{$C^{\text{f}}_2$}};
\node [] (dgn1-f-c1) at (-8,1.5){\tiny{$C^{\text{f}}_1$}};
\node [] (dgn1-input-f) at (-9,1.5){$x^{\text{f}}$};
\draw [-stealth,thick]   (dgn1-f-c3.east) -- (dgn1-f-c4.west);
\draw [-stealth,thick]   (dgn1-f-c2.east) -- (dgn1-f-c3.west);
\draw [-stealth,thick]   (dgn1-f-c1.east) -- (dgn1-f-c2.west);
\draw [-stealth,thick]   (dgn1-input-f.east) -- (dgn1-f-c1.west);

\node []  (dgn1-output) at (-0.25,-2.5) {$\hat{y}(x)$};

\node [] (dgn1-smax) at (-1.25,-2.5){\tiny{FC}};
\draw [-stealth,thick]   (dgn1-smax.east)--(dgn1-output.west);

\node [rotate=-90] (dgn1-gap) at (-2,-2.5){\tiny{GAP}};
\draw [-stealth,thick]   (dgn1-gap.north)--(dgn1-smax.west);

\node [rotate=-90] (dgn1-galu-4) at (-2.75,-2.5){\tiny{GaLU}};
\draw [-stealth,thick]   (dgn1-galu-4.north)--(dgn1-gap.south);

\node [] (dgn1-v-c4) at (-3.5,-2.5){\tiny{$C^{\text{v}}_4$}};
\draw [-stealth,thick]   (dgn1-v-c4.east) -- (dgn1-galu-4.south);

\node [rotate=-90] (dgn1-galu-3) at (-4.25,-2.5){\tiny{GaLU}};
\draw [-stealth,thick]   (dgn1-galu-3.north) -- (dgn1-v-c4.west);

\node [] (dgn1-v-c3) at (-5,-2.5){\tiny{$C^{\text{v}}_3$}};
\draw [-stealth,thick]   (dgn1-v-c3.east) -- (dgn1-galu-3.south);

\node [rotate=-90] (dgn1-galu-2) at (-5.75,-2.5){\tiny{GaLU}};
\draw [-stealth,thick]   (dgn1-galu-2.north) -- (dgn1-v-c3.west);

\node [] (dgn1-v-c2) at (-6.5,-2.5){\tiny{$C^{\text{v}}_2$}};
\draw [-stealth,thick]   (dgn1-v-c2.east) -- (dgn1-galu-2.south);

\node [rotate=-90] (dgn1-galu-1) at (-7.25,-2.5){\tiny{GaLU}};
\draw [-stealth,thick]   (dgn1-galu-1.north) -- (dgn1-v-c2.west);

\node [] (dgn1-v-c1) at (-8,-2.5){\tiny{$C^{\text{v}}_1$}};
\draw [-stealth,thick]   (dgn1-v-c1.east) -- (dgn1-galu-1.south);

\node [] (dgn1-v-input) at (-9,-2.5){$x^{\text{v}}$};

\draw [-stealth,thick]   (dgn1-v-input.east) -- (dgn1-v-c1.west);

\node[] (dgn1-gating-4-up) at (-2.75,0.5){\tiny{$G_{4}$}};
\draw [-stealth,thick]   (dgn1-f-c4.east) to[out=-90,in=90] (dgn1-gating-4-up.north);

\node[] (dgn1-gating-3-up) at (-4.25,0.5){\tiny{$G_{3}$}};
\draw [-stealth,thick]   (dgn1-f-c3.east) to[out=-90,in=90] (dgn1-gating-3-up.north);

\node[] (dgn1-gating-2-up) at (-5.75,0.5){\tiny{$G_{2}$}};
\draw [-stealth,thick]   (dgn1-f-c2.east) to[out=-90,in=90] (dgn1-gating-2-up.north);

\node[] (dgn1-gating-1-up) at (-7.25,0.5){\tiny{$G_{1}$}};
\draw [-stealth,thick]   (dgn1-f-c1.east) to[out=-90,in=90] (dgn1-gating-1-up.north);

\node[] (dgn1-gating-4) at (-2.75,-1.5){\tiny{$G_{i_4}$}};
\draw [-stealth,thick]   (dgn1-gating-4.south) -- (dgn1-galu-4.west);

\node[] (dgn1-gating-3) at (-4.25,-1.5){\tiny{$G_{i_3}$}};
\draw [-stealth,thick]   (dgn1-gating-3.south) -- (dgn1-galu-3.west);

\node[] (dgn1-gating-2) at (-5.75,-1.5){\tiny{$G_{i_2}$}};
\draw [-stealth,thick]   (dgn1-gating-2.south) -- (dgn1-galu-2.west);

\node[] (dgn1-gating-1) at (-7.25,-1.5){\tiny{$G_{i_1}$}};
\draw [-stealth,thick]   (dgn1-gating-1.south) -- (dgn1-galu-1.west);

\draw [-,thick,color=red]   (dgn1-gating-1-up.south) --(dgn1-gating-4.north);

\draw [-,thick,color=blue]   (dgn1-gating-2-up.south) --(dgn1-gating-3.north);

\draw [-,thick,color=green]   (dgn1-gating-3-up.south) --(dgn1-gating-1.north);

\draw [-,thick,color=brown]   (dgn1-gating-4-up.south) --(dgn1-gating-2.north);

%%%%%%%%%%%%%%%%%%%%%%%%%%%%%%%%%%%%%%%%%%%%%%%%%%%%%%%%%%%%%%%%%
	
\end{tikzpicture}
}
\end{minipage}
\begin{minipage}{0.48\columnwidth}
\centering
\resizebox{0.99\columnwidth}{!}{
\begin{tikzpicture}
%\node []  (fntext)at (-4.625,-3.5) {CNN-GAP-DLGN};

%\node []  (output) at (7.5,1.5) {$\hat{y}(x)$};

\node [] (dgn1-f-c4) at (-3.5,1.5){\tiny{$C^{\text{f}}_4$}};
\node [] (dgn1-f-c3) at (-5,1.5){\tiny{$C^{\text{f}}_3$}};
\node [] (dgn1-f-c2) at (-6.5,1.5){\tiny{$C^{\text{f}}_2$}};
\node [] (dgn1-f-c1) at (-8,1.5){\tiny{$C^{\text{f}}_1$}};
\node [] (dgn1-input-f) at (-9,1.5){$x^{\text{f}}$};
\draw [-stealth,thick]   (dgn1-f-c3.east) -- (dgn1-f-c4.west);
\draw [-stealth,thick]   (dgn1-f-c2.east) -- (dgn1-f-c3.west);
\draw [-stealth,thick]   (dgn1-f-c1.east) -- (dgn1-f-c2.west);
\draw [-stealth,thick]   (dgn1-input-f.east) -- (dgn1-f-c1.west);

\node []  (dgn1-output) at (-0.25,-2.5) {$\hat{y}(x)$};

\node [] (dgn1-smax) at (-1.25,-2.5){\tiny{FC}};
\draw [-stealth,thick]   (dgn1-smax.east)--(dgn1-output.west);

\node [rotate=-90] (dgn1-gap) at (-2,-2.5){\tiny{GAP}};
\draw [-stealth,thick]   (dgn1-gap.north)--(dgn1-smax.west);

\node [rotate=-90] (dgn1-galu-4) at (-2.75,-2.5){\tiny{GaLU}};
\draw [-stealth,thick]   (dgn1-galu-4.north)--(dgn1-gap.south);

\node [] (dgn1-v-c4) at (-3.5,-2.5){\tiny{$C^{\text{v}}_4$}};
\draw [-stealth,thick]   (dgn1-v-c4.east) -- (dgn1-galu-4.south);

\node [rotate=-90] (dgn1-galu-3) at (-4.25,-2.5){\tiny{GaLU}};
\draw [-stealth,thick]   (dgn1-galu-3.north) -- (dgn1-v-c4.west);

\node [] (dgn1-v-c3) at (-5,-2.5){\tiny{$C^{\text{v}}_3$}};
\draw [-stealth,thick]   (dgn1-v-c3.east) -- (dgn1-galu-3.south);

\node [rotate=-90] (dgn1-galu-2) at (-5.75,-2.5){\tiny{GaLU}};
\draw [-stealth,thick]   (dgn1-galu-2.north) -- (dgn1-v-c3.west);

\node [] (dgn1-v-c2) at (-6.5,-2.5){\tiny{$C^{\text{v}}_2$}};
\draw [-stealth,thick]   (dgn1-v-c2.east) -- (dgn1-galu-2.south);

\node [rotate=-90] (dgn1-galu-1) at (-7.25,-2.5){\tiny{GaLU}};
\draw [-stealth,thick]   (dgn1-galu-1.north) -- (dgn1-v-c2.west);

\node [] (dgn1-v-c1) at (-8,-2.5){\tiny{$C^{\text{v}}_1$}};
\draw [-stealth,thick]   (dgn1-v-c1.east) -- (dgn1-galu-1.south);

\node [] (dgn1-v-input) at (-9,-2.5){$x^{\text{v}}$};

\draw [-stealth,thick]   (dgn1-v-input.east) -- (dgn1-v-c1.west);

\node[] (dgn1-gating-4-up) at (-2.75,0.5){\tiny{$G_{4}$}};
\draw [-stealth,thick]   (dgn1-f-c4.east) to[out=-90,in=90] (dgn1-gating-4-up.north);

\node[] (dgn1-gating-3-up) at (-4.25,0.5){\tiny{$G_{3}$}};
\draw [-stealth,thick]   (dgn1-f-c3.east) to[out=-90,in=90] (dgn1-gating-3-up.north);

\node[] (dgn1-gating-2-up) at (-5.75,0.5){\tiny{$G_{2}$}};
\draw [-stealth,thick]   (dgn1-f-c2.east) to[out=-90,in=90] (dgn1-gating-2-up.north);

\node[] (dgn1-gating-1-up) at (-7.25,0.5){\tiny{$G_{1}$}};
\draw [-stealth,thick]   (dgn1-f-c1.east) to[out=-90,in=90] (dgn1-gating-1-up.north);

\node[] (dgn1-gating-4) at (-2.75,-1.5){\tiny{$G_{i_4}$}};
\draw [-stealth,thick]   (dgn1-gating-4.south) -- (dgn1-galu-4.west);

\node[] (dgn1-gating-3) at (-4.25,-1.5){\tiny{$G_{i_3}$}};
\draw [-stealth,thick]   (dgn1-gating-3.south) -- (dgn1-galu-3.west);

\node[] (dgn1-gating-2) at (-5.75,-1.5){\tiny{$G_{i_2}$}};
\draw [-stealth,thick]   (dgn1-gating-2.south) -- (dgn1-galu-2.west);

\node[] (dgn1-gating-1) at (-7.25,-1.5){\tiny{$G_{i_1}$}};
\draw [-stealth,thick]   (dgn1-gating-1.south) -- (dgn1-galu-1.west);

\draw [-,thick,color=red]   (dgn1-gating-1-up.south) --(dgn1-gating-1.north);

\draw [-,thick,color=blue]   (dgn1-gating-2-up.south) --(dgn1-gating-3.north);

\draw [-,thick,color=green]   (dgn1-gating-3-up.south) --(dgn1-gating-4.north);

\draw [-,thick,color=brown]   (dgn1-gating-4-up.south) --(dgn1-gating-2.north);

%%%%%%%%%%%%%%%%%%%%%%%%%%%%%%%%%%%%%%%%%%%%%%%%%%%%%%%%%%%%%%%%%
	
\end{tikzpicture}
}
\end{minipage}

\begin{minipage}{0.48\columnwidth}
\centering
\resizebox{0.99\columnwidth}{!}{
\begin{tikzpicture}
%\node []  (fntext)at (-4.625,-3.5) {CNN-GAP-DLGN};

%\node []  (output) at (7.5,1.5) {$\hat{y}(x)$};

\node [] (dgn1-f-c4) at (-3.5,1.5){\tiny{$C^{\text{f}}_4$}};
\node [] (dgn1-f-c3) at (-5,1.5){\tiny{$C^{\text{f}}_3$}};
\node [] (dgn1-f-c2) at (-6.5,1.5){\tiny{$C^{\text{f}}_2$}};
\node [] (dgn1-f-c1) at (-8,1.5){\tiny{$C^{\text{f}}_1$}};
\node [] (dgn1-input-f) at (-9,1.5){$x^{\text{f}}$};
\draw [-stealth,thick]   (dgn1-f-c3.east) -- (dgn1-f-c4.west);
\draw [-stealth,thick]   (dgn1-f-c2.east) -- (dgn1-f-c3.west);
\draw [-stealth,thick]   (dgn1-f-c1.east) -- (dgn1-f-c2.west);
\draw [-stealth,thick]   (dgn1-input-f.east) -- (dgn1-f-c1.west);

\node []  (dgn1-output) at (-0.25,-2.5) {$\hat{y}(x)$};

\node [] (dgn1-smax) at (-1.25,-2.5){\tiny{FC}};
\draw [-stealth,thick]   (dgn1-smax.east)--(dgn1-output.west);

\node [rotate=-90] (dgn1-gap) at (-2,-2.5){\tiny{GAP}};
\draw [-stealth,thick]   (dgn1-gap.north)--(dgn1-smax.west);

\node [rotate=-90] (dgn1-galu-4) at (-2.75,-2.5){\tiny{GaLU}};
\draw [-stealth,thick]   (dgn1-galu-4.north)--(dgn1-gap.south);

\node [] (dgn1-v-c4) at (-3.5,-2.5){\tiny{$C^{\text{v}}_4$}};
\draw [-stealth,thick]   (dgn1-v-c4.east) -- (dgn1-galu-4.south);

\node [rotate=-90] (dgn1-galu-3) at (-4.25,-2.5){\tiny{GaLU}};
\draw [-stealth,thick]   (dgn1-galu-3.north) -- (dgn1-v-c4.west);

\node [] (dgn1-v-c3) at (-5,-2.5){\tiny{$C^{\text{v}}_3$}};
\draw [-stealth,thick]   (dgn1-v-c3.east) -- (dgn1-galu-3.south);

\node [rotate=-90] (dgn1-galu-2) at (-5.75,-2.5){\tiny{GaLU}};
\draw [-stealth,thick]   (dgn1-galu-2.north) -- (dgn1-v-c3.west);

\node [] (dgn1-v-c2) at (-6.5,-2.5){\tiny{$C^{\text{v}}_2$}};
\draw [-stealth,thick]   (dgn1-v-c2.east) -- (dgn1-galu-2.south);

\node [rotate=-90] (dgn1-galu-1) at (-7.25,-2.5){\tiny{GaLU}};
\draw [-stealth,thick]   (dgn1-galu-1.north) -- (dgn1-v-c2.west);

\node [] (dgn1-v-c1) at (-8,-2.5){\tiny{$C^{\text{v}}_1$}};
\draw [-stealth,thick]   (dgn1-v-c1.east) -- (dgn1-galu-1.south);

\node [] (dgn1-v-input) at (-9,-2.5){$x^{\text{v}}$};

\draw [-stealth,thick]   (dgn1-v-input.east) -- (dgn1-v-c1.west);

\node[] (dgn1-gating-4-up) at (-2.75,0.5){\tiny{$G_{4}$}};
\draw [-stealth,thick]   (dgn1-f-c4.east) to[out=-90,in=90] (dgn1-gating-4-up.north);

\node[] (dgn1-gating-3-up) at (-4.25,0.5){\tiny{$G_{3}$}};
\draw [-stealth,thick]   (dgn1-f-c3.east) to[out=-90,in=90] (dgn1-gating-3-up.north);

\node[] (dgn1-gating-2-up) at (-5.75,0.5){\tiny{$G_{2}$}};
\draw [-stealth,thick]   (dgn1-f-c2.east) to[out=-90,in=90] (dgn1-gating-2-up.north);

\node[] (dgn1-gating-1-up) at (-7.25,0.5){\tiny{$G_{1}$}};
\draw [-stealth,thick]   (dgn1-f-c1.east) to[out=-90,in=90] (dgn1-gating-1-up.north);

\node[] (dgn1-gating-4) at (-2.75,-1.5){\tiny{$G_{i_4}$}};
\draw [-stealth,thick]   (dgn1-gating-4.south) -- (dgn1-galu-4.west);

\node[] (dgn1-gating-3) at (-4.25,-1.5){\tiny{$G_{i_3}$}};
\draw [-stealth,thick]   (dgn1-gating-3.south) -- (dgn1-galu-3.west);

\node[] (dgn1-gating-2) at (-5.75,-1.5){\tiny{$G_{i_2}$}};
\draw [-stealth,thick]   (dgn1-gating-2.south) -- (dgn1-galu-2.west);

\node[] (dgn1-gating-1) at (-7.25,-1.5){\tiny{$G_{i_1}$}};
\draw [-stealth,thick]   (dgn1-gating-1.south) -- (dgn1-galu-1.west);

\draw [-,thick,color=red]   (dgn1-gating-1-up.south) --(dgn1-gating-3.north);

\draw [-,thick,color=blue]   (dgn1-gating-2-up.south) --(dgn1-gating-1.north);

\draw [-,thick,color=green]   (dgn1-gating-3-up.south) --(dgn1-gating-4.north);

\draw [-,thick,color=brown]   (dgn1-gating-4-up.south) --(dgn1-gating-2.north);

%%%%%%%%%%%%%%%%%%%%%%%%%%%%%%%%%%%%%%%%%%%%%%%%%%%%%%%%%%%%%%%%%
	
\end{tikzpicture}
}
\end{minipage}
\begin{minipage}{0.48\columnwidth}
\centering
\resizebox{0.99\columnwidth}{!}{
\begin{tikzpicture}
%\node []  (fntext)at (-4.625,-3.5) {CNN-GAP-DLGN};

%\node []  (output) at (7.5,1.5) {$\hat{y}(x)$};

\node [] (dgn1-f-c4) at (-3.5,1.5){\tiny{$C^{\text{f}}_4$}};
\node [] (dgn1-f-c3) at (-5,1.5){\tiny{$C^{\text{f}}_3$}};
\node [] (dgn1-f-c2) at (-6.5,1.5){\tiny{$C^{\text{f}}_2$}};
\node [] (dgn1-f-c1) at (-8,1.5){\tiny{$C^{\text{f}}_1$}};
\node [] (dgn1-input-f) at (-9,1.5){$x^{\text{f}}$};
\draw [-stealth,thick]   (dgn1-f-c3.east) -- (dgn1-f-c4.west);
\draw [-stealth,thick]   (dgn1-f-c2.east) -- (dgn1-f-c3.west);
\draw [-stealth,thick]   (dgn1-f-c1.east) -- (dgn1-f-c2.west);
\draw [-stealth,thick]   (dgn1-input-f.east) -- (dgn1-f-c1.west);

\node []  (dgn1-output) at (-0.25,-2.5) {$\hat{y}(x)$};

\node [] (dgn1-smax) at (-1.25,-2.5){\tiny{FC}};
\draw [-stealth,thick]   (dgn1-smax.east)--(dgn1-output.west);

\node [rotate=-90] (dgn1-gap) at (-2,-2.5){\tiny{GAP}};
\draw [-stealth,thick]   (dgn1-gap.north)--(dgn1-smax.west);

\node [rotate=-90] (dgn1-galu-4) at (-2.75,-2.5){\tiny{GaLU}};
\draw [-stealth,thick]   (dgn1-galu-4.north)--(dgn1-gap.south);

\node [] (dgn1-v-c4) at (-3.5,-2.5){\tiny{$C^{\text{v}}_4$}};
\draw [-stealth,thick]   (dgn1-v-c4.east) -- (dgn1-galu-4.south);

\node [rotate=-90] (dgn1-galu-3) at (-4.25,-2.5){\tiny{GaLU}};
\draw [-stealth,thick]   (dgn1-galu-3.north) -- (dgn1-v-c4.west);

\node [] (dgn1-v-c3) at (-5,-2.5){\tiny{$C^{\text{v}}_3$}};
\draw [-stealth,thick]   (dgn1-v-c3.east) -- (dgn1-galu-3.south);

\node [rotate=-90] (dgn1-galu-2) at (-5.75,-2.5){\tiny{GaLU}};
\draw [-stealth,thick]   (dgn1-galu-2.north) -- (dgn1-v-c3.west);

\node [] (dgn1-v-c2) at (-6.5,-2.5){\tiny{$C^{\text{v}}_2$}};
\draw [-stealth,thick]   (dgn1-v-c2.east) -- (dgn1-galu-2.south);

\node [rotate=-90] (dgn1-galu-1) at (-7.25,-2.5){\tiny{GaLU}};
\draw [-stealth,thick]   (dgn1-galu-1.north) -- (dgn1-v-c2.west);

\node [] (dgn1-v-c1) at (-8,-2.5){\tiny{$C^{\text{v}}_1$}};
\draw [-stealth,thick]   (dgn1-v-c1.east) -- (dgn1-galu-1.south);

\node [] (dgn1-v-input) at (-9,-2.5){$x^{\text{v}}$};

\draw [-stealth,thick]   (dgn1-v-input.east) -- (dgn1-v-c1.west);

\node[] (dgn1-gating-4-up) at (-2.75,0.5){\tiny{$G_{4}$}};
\draw [-stealth,thick]   (dgn1-f-c4.east) to[out=-90,in=90] (dgn1-gating-4-up.north);

\node[] (dgn1-gating-3-up) at (-4.25,0.5){\tiny{$G_{3}$}};
\draw [-stealth,thick]   (dgn1-f-c3.east) to[out=-90,in=90] (dgn1-gating-3-up.north);

\node[] (dgn1-gating-2-up) at (-5.75,0.5){\tiny{$G_{2}$}};
\draw [-stealth,thick]   (dgn1-f-c2.east) to[out=-90,in=90] (dgn1-gating-2-up.north);

\node[] (dgn1-gating-1-up) at (-7.25,0.5){\tiny{$G_{1}$}};
\draw [-stealth,thick]   (dgn1-f-c1.east) to[out=-90,in=90] (dgn1-gating-1-up.north);

\node[] (dgn1-gating-4) at (-2.75,-1.5){\tiny{$G_{i_4}$}};
\draw [-stealth,thick]   (dgn1-gating-4.south) -- (dgn1-galu-4.west);

\node[] (dgn1-gating-3) at (-4.25,-1.5){\tiny{$G_{i_3}$}};
\draw [-stealth,thick]   (dgn1-gating-3.south) -- (dgn1-galu-3.west);

\node[] (dgn1-gating-2) at (-5.75,-1.5){\tiny{$G_{i_2}$}};
\draw [-stealth,thick]   (dgn1-gating-2.south) -- (dgn1-galu-2.west);

\node[] (dgn1-gating-1) at (-7.25,-1.5){\tiny{$G_{i_1}$}};
\draw [-stealth,thick]   (dgn1-gating-1.south) -- (dgn1-galu-1.west);

\draw [-,thick,color=red]   (dgn1-gating-1-up.south) --(dgn1-gating-2.north);

\draw [-,thick,color=blue]   (dgn1-gating-2-up.south) --(dgn1-gating-1.north);

\draw [-,thick,color=green]   (dgn1-gating-3-up.south) --(dgn1-gating-4.north);

\draw [-,thick,color=brown]   (dgn1-gating-4-up.south) --(dgn1-gating-3.north);

%%%%%%%%%%%%%%%%%%%%%%%%%%%%%%%%%%%%%%%%%%%%%%%%%%%%%%%%%%%%%%%%%
	
\end{tikzpicture}
}
\end{minipage}

\begin{minipage}{0.48\columnwidth}
\centering
\resizebox{0.99\columnwidth}{!}{
\begin{tikzpicture}
%\node []  (fntext)at (-4.625,-3.5) {CNN-GAP-DLGN};

%\node []  (output) at (7.5,1.5) {$\hat{y}(x)$};

\node [] (dgn1-f-c4) at (-3.5,1.5){\tiny{$C^{\text{f}}_4$}};
\node [] (dgn1-f-c3) at (-5,1.5){\tiny{$C^{\text{f}}_3$}};
\node [] (dgn1-f-c2) at (-6.5,1.5){\tiny{$C^{\text{f}}_2$}};
\node [] (dgn1-f-c1) at (-8,1.5){\tiny{$C^{\text{f}}_1$}};
\node [] (dgn1-input-f) at (-9,1.5){$x^{\text{f}}$};
\draw [-stealth,thick]   (dgn1-f-c3.east) -- (dgn1-f-c4.west);
\draw [-stealth,thick]   (dgn1-f-c2.east) -- (dgn1-f-c3.west);
\draw [-stealth,thick]   (dgn1-f-c1.east) -- (dgn1-f-c2.west);
\draw [-stealth,thick]   (dgn1-input-f.east) -- (dgn1-f-c1.west);

\node []  (dgn1-output) at (-0.25,-2.5) {$\hat{y}(x)$};

\node [] (dgn1-smax) at (-1.25,-2.5){\tiny{FC}};
\draw [-stealth,thick]   (dgn1-smax.east)--(dgn1-output.west);

\node [rotate=-90] (dgn1-gap) at (-2,-2.5){\tiny{GAP}};
\draw [-stealth,thick]   (dgn1-gap.north)--(dgn1-smax.west);

\node [rotate=-90] (dgn1-galu-4) at (-2.75,-2.5){\tiny{GaLU}};
\draw [-stealth,thick]   (dgn1-galu-4.north)--(dgn1-gap.south);

\node [] (dgn1-v-c4) at (-3.5,-2.5){\tiny{$C^{\text{v}}_4$}};
\draw [-stealth,thick]   (dgn1-v-c4.east) -- (dgn1-galu-4.south);

\node [rotate=-90] (dgn1-galu-3) at (-4.25,-2.5){\tiny{GaLU}};
\draw [-stealth,thick]   (dgn1-galu-3.north) -- (dgn1-v-c4.west);

\node [] (dgn1-v-c3) at (-5,-2.5){\tiny{$C^{\text{v}}_3$}};
\draw [-stealth,thick]   (dgn1-v-c3.east) -- (dgn1-galu-3.south);

\node [rotate=-90] (dgn1-galu-2) at (-5.75,-2.5){\tiny{GaLU}};
\draw [-stealth,thick]   (dgn1-galu-2.north) -- (dgn1-v-c3.west);

\node [] (dgn1-v-c2) at (-6.5,-2.5){\tiny{$C^{\text{v}}_2$}};
\draw [-stealth,thick]   (dgn1-v-c2.east) -- (dgn1-galu-2.south);

\node [rotate=-90] (dgn1-galu-1) at (-7.25,-2.5){\tiny{GaLU}};
\draw [-stealth,thick]   (dgn1-galu-1.north) -- (dgn1-v-c2.west);

\node [] (dgn1-v-c1) at (-8,-2.5){\tiny{$C^{\text{v}}_1$}};
\draw [-stealth,thick]   (dgn1-v-c1.east) -- (dgn1-galu-1.south);

\node [] (dgn1-v-input) at (-9,-2.5){$x^{\text{v}}$};

\draw [-stealth,thick]   (dgn1-v-input.east) -- (dgn1-v-c1.west);

\node[] (dgn1-gating-4-up) at (-2.75,0.5){\tiny{$G_{4}$}};
\draw [-stealth,thick]   (dgn1-f-c4.east) to[out=-90,in=90] (dgn1-gating-4-up.north);

\node[] (dgn1-gating-3-up) at (-4.25,0.5){\tiny{$G_{3}$}};
\draw [-stealth,thick]   (dgn1-f-c3.east) to[out=-90,in=90] (dgn1-gating-3-up.north);

\node[] (dgn1-gating-2-up) at (-5.75,0.5){\tiny{$G_{2}$}};
\draw [-stealth,thick]   (dgn1-f-c2.east) to[out=-90,in=90] (dgn1-gating-2-up.north);

\node[] (dgn1-gating-1-up) at (-7.25,0.5){\tiny{$G_{1}$}};
\draw [-stealth,thick]   (dgn1-f-c1.east) to[out=-90,in=90] (dgn1-gating-1-up.north);

\node[] (dgn1-gating-4) at (-2.75,-1.5){\tiny{$G_{i_4}$}};
\draw [-stealth,thick]   (dgn1-gating-4.south) -- (dgn1-galu-4.west);

\node[] (dgn1-gating-3) at (-4.25,-1.5){\tiny{$G_{i_3}$}};
\draw [-stealth,thick]   (dgn1-gating-3.south) -- (dgn1-galu-3.west);

\node[] (dgn1-gating-2) at (-5.75,-1.5){\tiny{$G_{i_2}$}};
\draw [-stealth,thick]   (dgn1-gating-2.south) -- (dgn1-galu-2.west);

\node[] (dgn1-gating-1) at (-7.25,-1.5){\tiny{$G_{i_1}$}};
\draw [-stealth,thick]   (dgn1-gating-1.south) -- (dgn1-galu-1.west);

\draw [-,thick,color=red]   (dgn1-gating-1-up.south) --(dgn1-gating-1.north);

\draw [-,thick,color=blue]   (dgn1-gating-2-up.south) --(dgn1-gating-2.north);

\draw [-,thick,color=green]   (dgn1-gating-3-up.south) --(dgn1-gating-4.north);

\draw [-,thick,color=brown]   (dgn1-gating-4-up.south) --(dgn1-gating-3.north);

%%%%%%%%%%%%%%%%%%%%%%%%%%%%%%%%%%%%%%%%%%%%%%%%%%%%%%%%%%%%%%%%%
	
\end{tikzpicture}
}
\end{minipage}
\begin{minipage}{0.48\columnwidth}
\centering
\resizebox{0.99\columnwidth}{!}{
\begin{tikzpicture}
%\node []  (fntext)at (-4.625,-3.5) {CNN-GAP-DLGN};

%\node []  (output) at (7.5,1.5) {$\hat{y}(x)$};

\node [] (dgn1-f-c4) at (-3.5,1.5){\tiny{$C^{\text{f}}_4$}};
\node [] (dgn1-f-c3) at (-5,1.5){\tiny{$C^{\text{f}}_3$}};
\node [] (dgn1-f-c2) at (-6.5,1.5){\tiny{$C^{\text{f}}_2$}};
\node [] (dgn1-f-c1) at (-8,1.5){\tiny{$C^{\text{f}}_1$}};
\node [] (dgn1-input-f) at (-9,1.5){$x^{\text{f}}$};
\draw [-stealth,thick]   (dgn1-f-c3.east) -- (dgn1-f-c4.west);
\draw [-stealth,thick]   (dgn1-f-c2.east) -- (dgn1-f-c3.west);
\draw [-stealth,thick]   (dgn1-f-c1.east) -- (dgn1-f-c2.west);
\draw [-stealth,thick]   (dgn1-input-f.east) -- (dgn1-f-c1.west);

\node []  (dgn1-output) at (-0.25,-2.5) {$\hat{y}(x)$};

\node [] (dgn1-smax) at (-1.25,-2.5){\tiny{FC}};
\draw [-stealth,thick]   (dgn1-smax.east)--(dgn1-output.west);

\node [rotate=-90] (dgn1-gap) at (-2,-2.5){\tiny{GAP}};
\draw [-stealth,thick]   (dgn1-gap.north)--(dgn1-smax.west);

\node [rotate=-90] (dgn1-galu-4) at (-2.75,-2.5){\tiny{GaLU}};
\draw [-stealth,thick]   (dgn1-galu-4.north)--(dgn1-gap.south);

\node [] (dgn1-v-c4) at (-3.5,-2.5){\tiny{$C^{\text{v}}_4$}};
\draw [-stealth,thick]   (dgn1-v-c4.east) -- (dgn1-galu-4.south);

\node [rotate=-90] (dgn1-galu-3) at (-4.25,-2.5){\tiny{GaLU}};
\draw [-stealth,thick]   (dgn1-galu-3.north) -- (dgn1-v-c4.west);

\node [] (dgn1-v-c3) at (-5,-2.5){\tiny{$C^{\text{v}}_3$}};
\draw [-stealth,thick]   (dgn1-v-c3.east) -- (dgn1-galu-3.south);

\node [rotate=-90] (dgn1-galu-2) at (-5.75,-2.5){\tiny{GaLU}};
\draw [-stealth,thick]   (dgn1-galu-2.north) -- (dgn1-v-c3.west);

\node [] (dgn1-v-c2) at (-6.5,-2.5){\tiny{$C^{\text{v}}_2$}};
\draw [-stealth,thick]   (dgn1-v-c2.east) -- (dgn1-galu-2.south);

\node [rotate=-90] (dgn1-galu-1) at (-7.25,-2.5){\tiny{GaLU}};
\draw [-stealth,thick]   (dgn1-galu-1.north) -- (dgn1-v-c2.west);

\node [] (dgn1-v-c1) at (-8,-2.5){\tiny{$C^{\text{v}}_1$}};
\draw [-stealth,thick]   (dgn1-v-c1.east) -- (dgn1-galu-1.south);

\node [] (dgn1-v-input) at (-9,-2.5){$x^{\text{v}}$};

\draw [-stealth,thick]   (dgn1-v-input.east) -- (dgn1-v-c1.west);

\node[] (dgn1-gating-4-up) at (-2.75,0.5){\tiny{$G_{4}$}};
\draw [-stealth,thick]   (dgn1-f-c4.east) to[out=-90,in=90] (dgn1-gating-4-up.north);

\node[] (dgn1-gating-3-up) at (-4.25,0.5){\tiny{$G_{3}$}};
\draw [-stealth,thick]   (dgn1-f-c3.east) to[out=-90,in=90] (dgn1-gating-3-up.north);

\node[] (dgn1-gating-2-up) at (-5.75,0.5){\tiny{$G_{2}$}};
\draw [-stealth,thick]   (dgn1-f-c2.east) to[out=-90,in=90] (dgn1-gating-2-up.north);

\node[] (dgn1-gating-1-up) at (-7.25,0.5){\tiny{$G_{1}$}};
\draw [-stealth,thick]   (dgn1-f-c1.east) to[out=-90,in=90] (dgn1-gating-1-up.north);

\node[] (dgn1-gating-4) at (-2.75,-1.5){\tiny{$G_{i_4}$}};
\draw [-stealth,thick]   (dgn1-gating-4.south) -- (dgn1-galu-4.west);

\node[] (dgn1-gating-3) at (-4.25,-1.5){\tiny{$G_{i_3}$}};
\draw [-stealth,thick]   (dgn1-gating-3.south) -- (dgn1-galu-3.west);

\node[] (dgn1-gating-2) at (-5.75,-1.5){\tiny{$G_{i_2}$}};
\draw [-stealth,thick]   (dgn1-gating-2.south) -- (dgn1-galu-2.west);

\node[] (dgn1-gating-1) at (-7.25,-1.5){\tiny{$G_{i_1}$}};
\draw [-stealth,thick]   (dgn1-gating-1.south) -- (dgn1-galu-1.west);

\draw [-,thick,color=red]   (dgn1-gating-1-up.south) --(dgn1-gating-4.north);

\draw [-,thick,color=blue]   (dgn1-gating-2-up.south) --(dgn1-gating-2.north);

\draw [-,thick,color=green]   (dgn1-gating-3-up.south) --(dgn1-gating-1.north);

\draw [-,thick,color=brown]   (dgn1-gating-4-up.south) --(dgn1-gating-3.north);

%%%%%%%%%%%%%%%%%%%%%%%%%%%%%%%%%%%%%%%%%%%%%%%%%%%%%%%%%%%%%%%%%
	
\end{tikzpicture}
}
\end{minipage}

\begin{minipage}{0.48\columnwidth}
\centering
\resizebox{0.99\columnwidth}{!}{
\begin{tikzpicture}
%\node []  (fntext)at (-4.625,-3.5) {CNN-GAP-DLGN};

%\node []  (output) at (7.5,1.5) {$\hat{y}(x)$};

\node [] (dgn1-f-c4) at (-3.5,1.5){\tiny{$C^{\text{f}}_4$}};
\node [] (dgn1-f-c3) at (-5,1.5){\tiny{$C^{\text{f}}_3$}};
\node [] (dgn1-f-c2) at (-6.5,1.5){\tiny{$C^{\text{f}}_2$}};
\node [] (dgn1-f-c1) at (-8,1.5){\tiny{$C^{\text{f}}_1$}};
\node [] (dgn1-input-f) at (-9,1.5){$x^{\text{f}}$};
\draw [-stealth,thick]   (dgn1-f-c3.east) -- (dgn1-f-c4.west);
\draw [-stealth,thick]   (dgn1-f-c2.east) -- (dgn1-f-c3.west);
\draw [-stealth,thick]   (dgn1-f-c1.east) -- (dgn1-f-c2.west);
\draw [-stealth,thick]   (dgn1-input-f.east) -- (dgn1-f-c1.west);

\node []  (dgn1-output) at (-0.25,-2.5) {$\hat{y}(x)$};

\node [] (dgn1-smax) at (-1.25,-2.5){\tiny{FC}};
\draw [-stealth,thick]   (dgn1-smax.east)--(dgn1-output.west);

\node [rotate=-90] (dgn1-gap) at (-2,-2.5){\tiny{GAP}};
\draw [-stealth,thick]   (dgn1-gap.north)--(dgn1-smax.west);

\node [rotate=-90] (dgn1-galu-4) at (-2.75,-2.5){\tiny{GaLU}};
\draw [-stealth,thick]   (dgn1-galu-4.north)--(dgn1-gap.south);

\node [] (dgn1-v-c4) at (-3.5,-2.5){\tiny{$C^{\text{v}}_4$}};
\draw [-stealth,thick]   (dgn1-v-c4.east) -- (dgn1-galu-4.south);

\node [rotate=-90] (dgn1-galu-3) at (-4.25,-2.5){\tiny{GaLU}};
\draw [-stealth,thick]   (dgn1-galu-3.north) -- (dgn1-v-c4.west);

\node [] (dgn1-v-c3) at (-5,-2.5){\tiny{$C^{\text{v}}_3$}};
\draw [-stealth,thick]   (dgn1-v-c3.east) -- (dgn1-galu-3.south);

\node [rotate=-90] (dgn1-galu-2) at (-5.75,-2.5){\tiny{GaLU}};
\draw [-stealth,thick]   (dgn1-galu-2.north) -- (dgn1-v-c3.west);

\node [] (dgn1-v-c2) at (-6.5,-2.5){\tiny{$C^{\text{v}}_2$}};
\draw [-stealth,thick]   (dgn1-v-c2.east) -- (dgn1-galu-2.south);

\node [rotate=-90] (dgn1-galu-1) at (-7.25,-2.5){\tiny{GaLU}};
\draw [-stealth,thick]   (dgn1-galu-1.north) -- (dgn1-v-c2.west);

\node [] (dgn1-v-c1) at (-8,-2.5){\tiny{$C^{\text{v}}_1$}};
\draw [-stealth,thick]   (dgn1-v-c1.east) -- (dgn1-galu-1.south);

\node [] (dgn1-v-input) at (-9,-2.5){$x^{\text{v}}$};

\draw [-stealth,thick]   (dgn1-v-input.east) -- (dgn1-v-c1.west);

\node[] (dgn1-gating-4-up) at (-2.75,0.5){\tiny{$G_{4}$}};
\draw [-stealth,thick]   (dgn1-f-c4.east) to[out=-90,in=90] (dgn1-gating-4-up.north);

\node[] (dgn1-gating-3-up) at (-4.25,0.5){\tiny{$G_{3}$}};
\draw [-stealth,thick]   (dgn1-f-c3.east) to[out=-90,in=90] (dgn1-gating-3-up.north);

\node[] (dgn1-gating-2-up) at (-5.75,0.5){\tiny{$G_{2}$}};
\draw [-stealth,thick]   (dgn1-f-c2.east) to[out=-90,in=90] (dgn1-gating-2-up.north);

\node[] (dgn1-gating-1-up) at (-7.25,0.5){\tiny{$G_{1}$}};
\draw [-stealth,thick]   (dgn1-f-c1.east) to[out=-90,in=90] (dgn1-gating-1-up.north);

\node[] (dgn1-gating-4) at (-2.75,-1.5){\tiny{$G_{i_4}$}};
\draw [-stealth,thick]   (dgn1-gating-4.south) -- (dgn1-galu-4.west);

\node[] (dgn1-gating-3) at (-4.25,-1.5){\tiny{$G_{i_3}$}};
\draw [-stealth,thick]   (dgn1-gating-3.south) -- (dgn1-galu-3.west);

\node[] (dgn1-gating-2) at (-5.75,-1.5){\tiny{$G_{i_2}$}};
\draw [-stealth,thick]   (dgn1-gating-2.south) -- (dgn1-galu-2.west);

\node[] (dgn1-gating-1) at (-7.25,-1.5){\tiny{$G_{i_1}$}};
\draw [-stealth,thick]   (dgn1-gating-1.south) -- (dgn1-galu-1.west);

\draw [-,thick,color=red]   (dgn1-gating-1-up.south) --(dgn1-gating-2.north);

\draw [-,thick,color=blue]   (dgn1-gating-2-up.south) --(dgn1-gating-4.north);

\draw [-,thick,color=green]   (dgn1-gating-3-up.south) --(dgn1-gating-1.north);

\draw [-,thick,color=brown]   (dgn1-gating-4-up.south) --(dgn1-gating-3.north);

%%%%%%%%%%%%%%%%%%%%%%%%%%%%%%%%%%%%%%%%%%%%%%%%%%%%%%%%%%%%%%%%%
	
\end{tikzpicture}
}
\end{minipage}
\begin{minipage}{0.48\columnwidth}
\centering
\resizebox{0.99\columnwidth}{!}{
\begin{tikzpicture}
%\node []  (fntext)at (-4.625,-3.5) {CNN-GAP-DLGN};

%\node []  (output) at (7.5,1.5) {$\hat{y}(x)$};

\node [] (dgn1-f-c4) at (-3.5,1.5){\tiny{$C^{\text{f}}_4$}};
\node [] (dgn1-f-c3) at (-5,1.5){\tiny{$C^{\text{f}}_3$}};
\node [] (dgn1-f-c2) at (-6.5,1.5){\tiny{$C^{\text{f}}_2$}};
\node [] (dgn1-f-c1) at (-8,1.5){\tiny{$C^{\text{f}}_1$}};
\node [] (dgn1-input-f) at (-9,1.5){$x^{\text{f}}$};
\draw [-stealth,thick]   (dgn1-f-c3.east) -- (dgn1-f-c4.west);
\draw [-stealth,thick]   (dgn1-f-c2.east) -- (dgn1-f-c3.west);
\draw [-stealth,thick]   (dgn1-f-c1.east) -- (dgn1-f-c2.west);
\draw [-stealth,thick]   (dgn1-input-f.east) -- (dgn1-f-c1.west);

\node []  (dgn1-output) at (-0.25,-2.5) {$\hat{y}(x)$};

\node [] (dgn1-smax) at (-1.25,-2.5){\tiny{FC}};
\draw [-stealth,thick]   (dgn1-smax.east)--(dgn1-output.west);

\node [rotate=-90] (dgn1-gap) at (-2,-2.5){\tiny{GAP}};
\draw [-stealth,thick]   (dgn1-gap.north)--(dgn1-smax.west);

\node [rotate=-90] (dgn1-galu-4) at (-2.75,-2.5){\tiny{GaLU}};
\draw [-stealth,thick]   (dgn1-galu-4.north)--(dgn1-gap.south);

\node [] (dgn1-v-c4) at (-3.5,-2.5){\tiny{$C^{\text{v}}_4$}};
\draw [-stealth,thick]   (dgn1-v-c4.east) -- (dgn1-galu-4.south);

\node [rotate=-90] (dgn1-galu-3) at (-4.25,-2.5){\tiny{GaLU}};
\draw [-stealth,thick]   (dgn1-galu-3.north) -- (dgn1-v-c4.west);

\node [] (dgn1-v-c3) at (-5,-2.5){\tiny{$C^{\text{v}}_3$}};
\draw [-stealth,thick]   (dgn1-v-c3.east) -- (dgn1-galu-3.south);

\node [rotate=-90] (dgn1-galu-2) at (-5.75,-2.5){\tiny{GaLU}};
\draw [-stealth,thick]   (dgn1-galu-2.north) -- (dgn1-v-c3.west);

\node [] (dgn1-v-c2) at (-6.5,-2.5){\tiny{$C^{\text{v}}_2$}};
\draw [-stealth,thick]   (dgn1-v-c2.east) -- (dgn1-galu-2.south);

\node [rotate=-90] (dgn1-galu-1) at (-7.25,-2.5){\tiny{GaLU}};
\draw [-stealth,thick]   (dgn1-galu-1.north) -- (dgn1-v-c2.west);

\node [] (dgn1-v-c1) at (-8,-2.5){\tiny{$C^{\text{v}}_1$}};
\draw [-stealth,thick]   (dgn1-v-c1.east) -- (dgn1-galu-1.south);

\node [] (dgn1-v-input) at (-9,-2.5){$x^{\text{v}}$};

\draw [-stealth,thick]   (dgn1-v-input.east) -- (dgn1-v-c1.west);

\node[] (dgn1-gating-4-up) at (-2.75,0.5){\tiny{$G_{4}$}};
\draw [-stealth,thick]   (dgn1-f-c4.east) to[out=-90,in=90] (dgn1-gating-4-up.north);

\node[] (dgn1-gating-3-up) at (-4.25,0.5){\tiny{$G_{3}$}};
\draw [-stealth,thick]   (dgn1-f-c3.east) to[out=-90,in=90] (dgn1-gating-3-up.north);

\node[] (dgn1-gating-2-up) at (-5.75,0.5){\tiny{$G_{2}$}};
\draw [-stealth,thick]   (dgn1-f-c2.east) to[out=-90,in=90] (dgn1-gating-2-up.north);

\node[] (dgn1-gating-1-up) at (-7.25,0.5){\tiny{$G_{1}$}};
\draw [-stealth,thick]   (dgn1-f-c1.east) to[out=-90,in=90] (dgn1-gating-1-up.north);

\node[] (dgn1-gating-4) at (-2.75,-1.5){\tiny{$G_{i_4}$}};
\draw [-stealth,thick]   (dgn1-gating-4.south) -- (dgn1-galu-4.west);

\node[] (dgn1-gating-3) at (-4.25,-1.5){\tiny{$G_{i_3}$}};
\draw [-stealth,thick]   (dgn1-gating-3.south) -- (dgn1-galu-3.west);

\node[] (dgn1-gating-2) at (-5.75,-1.5){\tiny{$G_{i_2}$}};
\draw [-stealth,thick]   (dgn1-gating-2.south) -- (dgn1-galu-2.west);

\node[] (dgn1-gating-1) at (-7.25,-1.5){\tiny{$G_{i_1}$}};
\draw [-stealth,thick]   (dgn1-gating-1.south) -- (dgn1-galu-1.west);

\draw [-,thick,color=red]   (dgn1-gating-1-up.south) --(dgn1-gating-1.north);

\draw [-,thick,color=blue]   (dgn1-gating-2-up.south) --(dgn1-gating-4.north);

\draw [-,thick,color=green]   (dgn1-gating-3-up.south) --(dgn1-gating-2.north);

\draw [-,thick,color=brown]   (dgn1-gating-4-up.south) --(dgn1-gating-3.north);

%%%%%%%%%%%%%%%%%%%%%%%%%%%%%%%%%%%%%%%%%%%%%%%%%%%%%%%%%%%%%%%%%
	
\end{tikzpicture}
}
\end{minipage}

\begin{minipage}{0.48\columnwidth}
\centering
\resizebox{0.99\columnwidth}{!}{
\begin{tikzpicture}
%\node []  (fntext)at (-4.625,-3.5) {CNN-GAP-DLGN};

%\node []  (output) at (7.5,1.5) {$\hat{y}(x)$};

\node [] (dgn1-f-c4) at (-3.5,1.5){\tiny{$C^{\text{f}}_4$}};
\node [] (dgn1-f-c3) at (-5,1.5){\tiny{$C^{\text{f}}_3$}};
\node [] (dgn1-f-c2) at (-6.5,1.5){\tiny{$C^{\text{f}}_2$}};
\node [] (dgn1-f-c1) at (-8,1.5){\tiny{$C^{\text{f}}_1$}};
\node [] (dgn1-input-f) at (-9,1.5){$x^{\text{f}}$};
\draw [-stealth,thick]   (dgn1-f-c3.east) -- (dgn1-f-c4.west);
\draw [-stealth,thick]   (dgn1-f-c2.east) -- (dgn1-f-c3.west);
\draw [-stealth,thick]   (dgn1-f-c1.east) -- (dgn1-f-c2.west);
\draw [-stealth,thick]   (dgn1-input-f.east) -- (dgn1-f-c1.west);

\node []  (dgn1-output) at (-0.25,-2.5) {$\hat{y}(x)$};

\node [] (dgn1-smax) at (-1.25,-2.5){\tiny{FC}};
\draw [-stealth,thick]   (dgn1-smax.east)--(dgn1-output.west);

\node [rotate=-90] (dgn1-gap) at (-2,-2.5){\tiny{GAP}};
\draw [-stealth,thick]   (dgn1-gap.north)--(dgn1-smax.west);

\node [rotate=-90] (dgn1-galu-4) at (-2.75,-2.5){\tiny{GaLU}};
\draw [-stealth,thick]   (dgn1-galu-4.north)--(dgn1-gap.south);

\node [] (dgn1-v-c4) at (-3.5,-2.5){\tiny{$C^{\text{v}}_4$}};
\draw [-stealth,thick]   (dgn1-v-c4.east) -- (dgn1-galu-4.south);

\node [rotate=-90] (dgn1-galu-3) at (-4.25,-2.5){\tiny{GaLU}};
\draw [-stealth,thick]   (dgn1-galu-3.north) -- (dgn1-v-c4.west);

\node [] (dgn1-v-c3) at (-5,-2.5){\tiny{$C^{\text{v}}_3$}};
\draw [-stealth,thick]   (dgn1-v-c3.east) -- (dgn1-galu-3.south);

\node [rotate=-90] (dgn1-galu-2) at (-5.75,-2.5){\tiny{GaLU}};
\draw [-stealth,thick]   (dgn1-galu-2.north) -- (dgn1-v-c3.west);

\node [] (dgn1-v-c2) at (-6.5,-2.5){\tiny{$C^{\text{v}}_2$}};
\draw [-stealth,thick]   (dgn1-v-c2.east) -- (dgn1-galu-2.south);

\node [rotate=-90] (dgn1-galu-1) at (-7.25,-2.5){\tiny{GaLU}};
\draw [-stealth,thick]   (dgn1-galu-1.north) -- (dgn1-v-c2.west);

\node [] (dgn1-v-c1) at (-8,-2.5){\tiny{$C^{\text{v}}_1$}};
\draw [-stealth,thick]   (dgn1-v-c1.east) -- (dgn1-galu-1.south);

\node [] (dgn1-v-input) at (-9,-2.5){$x^{\text{v}}$};

\draw [-stealth,thick]   (dgn1-v-input.east) -- (dgn1-v-c1.west);

\node[] (dgn1-gating-4-up) at (-2.75,0.5){\tiny{$G_{4}$}};
\draw [-stealth,thick]   (dgn1-f-c4.east) to[out=-90,in=90] (dgn1-gating-4-up.north);

\node[] (dgn1-gating-3-up) at (-4.25,0.5){\tiny{$G_{3}$}};
\draw [-stealth,thick]   (dgn1-f-c3.east) to[out=-90,in=90] (dgn1-gating-3-up.north);

\node[] (dgn1-gating-2-up) at (-5.75,0.5){\tiny{$G_{2}$}};
\draw [-stealth,thick]   (dgn1-f-c2.east) to[out=-90,in=90] (dgn1-gating-2-up.north);

\node[] (dgn1-gating-1-up) at (-7.25,0.5){\tiny{$G_{1}$}};
\draw [-stealth,thick]   (dgn1-f-c1.east) to[out=-90,in=90] (dgn1-gating-1-up.north);

\node[] (dgn1-gating-4) at (-2.75,-1.5){\tiny{$G_{i_4}$}};
\draw [-stealth,thick]   (dgn1-gating-4.south) -- (dgn1-galu-4.west);

\node[] (dgn1-gating-3) at (-4.25,-1.5){\tiny{$G_{i_3}$}};
\draw [-stealth,thick]   (dgn1-gating-3.south) -- (dgn1-galu-3.west);

\node[] (dgn1-gating-2) at (-5.75,-1.5){\tiny{$G_{i_2}$}};
\draw [-stealth,thick]   (dgn1-gating-2.south) -- (dgn1-galu-2.west);

\node[] (dgn1-gating-1) at (-7.25,-1.5){\tiny{$G_{i_1}$}};
\draw [-stealth,thick]   (dgn1-gating-1.south) -- (dgn1-galu-1.west);

\draw [-,thick,color=red]   (dgn1-gating-1-up.south) --(dgn1-gating-4.north);

\draw [-,thick,color=blue]   (dgn1-gating-2-up.south) --(dgn1-gating-1.north);

\draw [-,thick,color=green]   (dgn1-gating-3-up.south) --(dgn1-gating-2.north);

\draw [-,thick,color=brown]   (dgn1-gating-4-up.south) --(dgn1-gating-3.north);

%%%%%%%%%%%%%%%%%%%%%%%%%%%%%%%%%%%%%%%%%%%%%%%%%%%%%%%%%%%%%%%%%
	
\end{tikzpicture}
}
\end{minipage}
\begin{minipage}{0.48\columnwidth}
\centering
\resizebox{0.99\columnwidth}{!}{
\begin{tikzpicture}
%\node []  (fntext)at (-4.625,-3.5) {CNN-GAP-DLGN};

%\node []  (output) at (7.5,1.5) {$\hat{y}(x)$};

\node [] (dgn1-f-c4) at (-3.5,1.5){\tiny{$C^{\text{f}}_4$}};
\node [] (dgn1-f-c3) at (-5,1.5){\tiny{$C^{\text{f}}_3$}};
\node [] (dgn1-f-c2) at (-6.5,1.5){\tiny{$C^{\text{f}}_2$}};
\node [] (dgn1-f-c1) at (-8,1.5){\tiny{$C^{\text{f}}_1$}};
\node [] (dgn1-input-f) at (-9,1.5){$x^{\text{f}}$};
\draw [-stealth,thick]   (dgn1-f-c3.east) -- (dgn1-f-c4.west);
\draw [-stealth,thick]   (dgn1-f-c2.east) -- (dgn1-f-c3.west);
\draw [-stealth,thick]   (dgn1-f-c1.east) -- (dgn1-f-c2.west);
\draw [-stealth,thick]   (dgn1-input-f.east) -- (dgn1-f-c1.west);

\node []  (dgn1-output) at (-0.25,-2.5) {$\hat{y}(x)$};

\node [] (dgn1-smax) at (-1.25,-2.5){\tiny{FC}};
\draw [-stealth,thick]   (dgn1-smax.east)--(dgn1-output.west);

\node [rotate=-90] (dgn1-gap) at (-2,-2.5){\tiny{GAP}};
\draw [-stealth,thick]   (dgn1-gap.north)--(dgn1-smax.west);

\node [rotate=-90] (dgn1-galu-4) at (-2.75,-2.5){\tiny{GaLU}};
\draw [-stealth,thick]   (dgn1-galu-4.north)--(dgn1-gap.south);

\node [] (dgn1-v-c4) at (-3.5,-2.5){\tiny{$C^{\text{v}}_4$}};
\draw [-stealth,thick]   (dgn1-v-c4.east) -- (dgn1-galu-4.south);

\node [rotate=-90] (dgn1-galu-3) at (-4.25,-2.5){\tiny{GaLU}};
\draw [-stealth,thick]   (dgn1-galu-3.north) -- (dgn1-v-c4.west);

\node [] (dgn1-v-c3) at (-5,-2.5){\tiny{$C^{\text{v}}_3$}};
\draw [-stealth,thick]   (dgn1-v-c3.east) -- (dgn1-galu-3.south);

\node [rotate=-90] (dgn1-galu-2) at (-5.75,-2.5){\tiny{GaLU}};
\draw [-stealth,thick]   (dgn1-galu-2.north) -- (dgn1-v-c3.west);

\node [] (dgn1-v-c2) at (-6.5,-2.5){\tiny{$C^{\text{v}}_2$}};
\draw [-stealth,thick]   (dgn1-v-c2.east) -- (dgn1-galu-2.south);

\node [rotate=-90] (dgn1-galu-1) at (-7.25,-2.5){\tiny{GaLU}};
\draw [-stealth,thick]   (dgn1-galu-1.north) -- (dgn1-v-c2.west);

\node [] (dgn1-v-c1) at (-8,-2.5){\tiny{$C^{\text{v}}_1$}};
\draw [-stealth,thick]   (dgn1-v-c1.east) -- (dgn1-galu-1.south);

\node [] (dgn1-v-input) at (-9,-2.5){$x^{\text{v}}$};

\draw [-stealth,thick]   (dgn1-v-input.east) -- (dgn1-v-c1.west);

\node[] (dgn1-gating-4-up) at (-2.75,0.5){\tiny{$G_{4}$}};
\draw [-stealth,thick]   (dgn1-f-c4.east) to[out=-90,in=90] (dgn1-gating-4-up.north);

\node[] (dgn1-gating-3-up) at (-4.25,0.5){\tiny{$G_{3}$}};
\draw [-stealth,thick]   (dgn1-f-c3.east) to[out=-90,in=90] (dgn1-gating-3-up.north);

\node[] (dgn1-gating-2-up) at (-5.75,0.5){\tiny{$G_{2}$}};
\draw [-stealth,thick]   (dgn1-f-c2.east) to[out=-90,in=90] (dgn1-gating-2-up.north);

\node[] (dgn1-gating-1-up) at (-7.25,0.5){\tiny{$G_{1}$}};
\draw [-stealth,thick]   (dgn1-f-c1.east) to[out=-90,in=90] (dgn1-gating-1-up.north);

\node[] (dgn1-gating-4) at (-2.75,-1.5){\tiny{$G_{i_4}$}};
\draw [-stealth,thick]   (dgn1-gating-4.south) -- (dgn1-galu-4.west);

\node[] (dgn1-gating-3) at (-4.25,-1.5){\tiny{$G_{i_3}$}};
\draw [-stealth,thick]   (dgn1-gating-3.south) -- (dgn1-galu-3.west);

\node[] (dgn1-gating-2) at (-5.75,-1.5){\tiny{$G_{i_2}$}};
\draw [-stealth,thick]   (dgn1-gating-2.south) -- (dgn1-galu-2.west);

\node[] (dgn1-gating-1) at (-7.25,-1.5){\tiny{$G_{i_1}$}};
\draw [-stealth,thick]   (dgn1-gating-1.south) -- (dgn1-galu-1.west);

\draw [-,thick,color=red]   (dgn1-gating-1-up.south) --(dgn1-gating-4.north);

\draw [-,thick,color=blue]   (dgn1-gating-2-up.south) --(dgn1-gating-2.north);

\draw [-,thick,color=green]   (dgn1-gating-3-up.south) --(dgn1-gating-3.north);

\draw [-,thick,color=brown]   (dgn1-gating-4-up.south) --(dgn1-gating-1.north);

%%%%%%%%%%%%%%%%%%%%%%%%%%%%%%%%%%%%%%%%%%%%%%%%%%%%%%%%%%%%%%%%%
	
\end{tikzpicture}
}
\end{minipage}
\caption{Shows the permutations $13-24$ C4GAP-DLGN in Table I of \Cref{fig:c4gap}.}
\label{fig:perm1}
\end{figure}
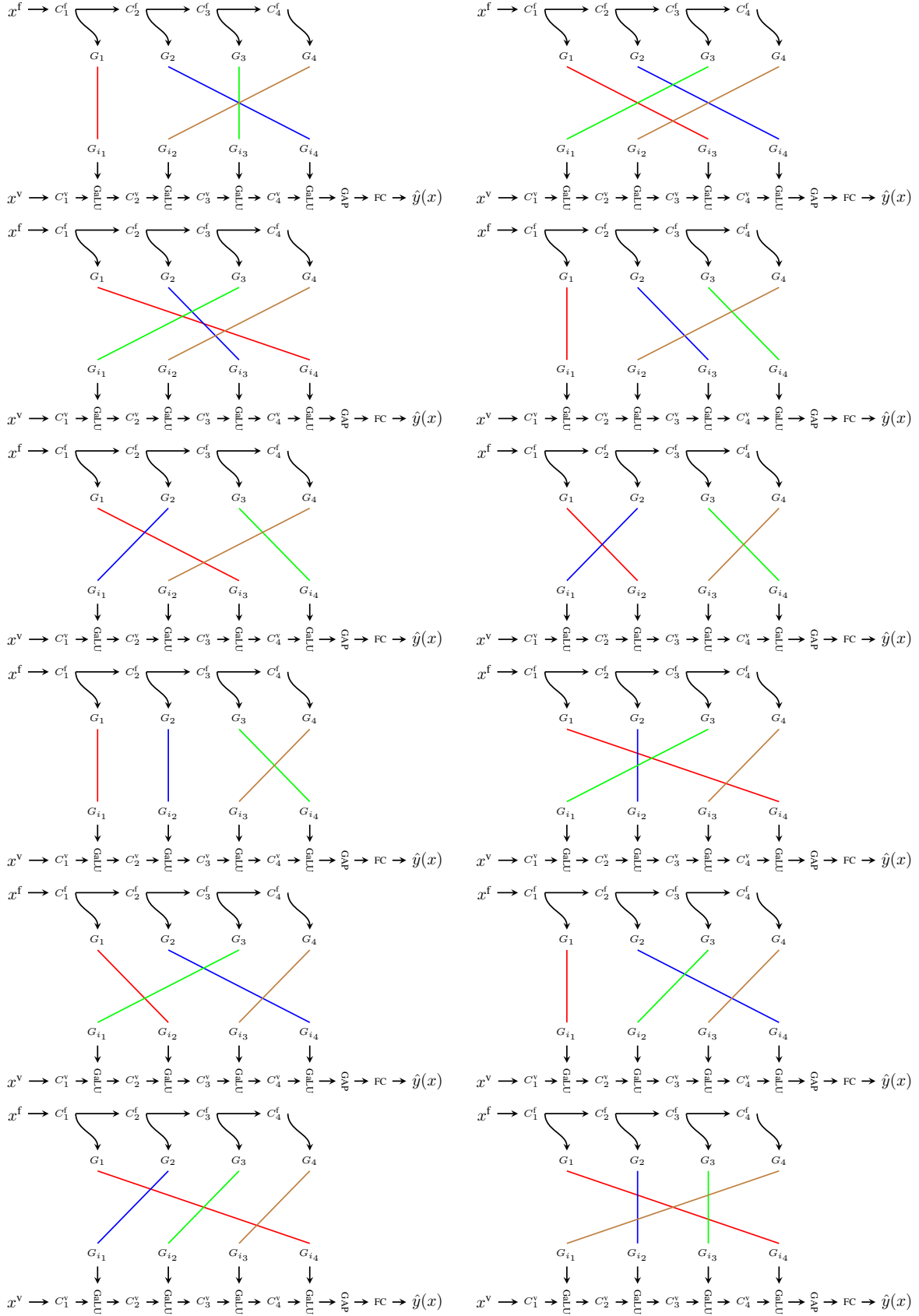

\begin{figure}
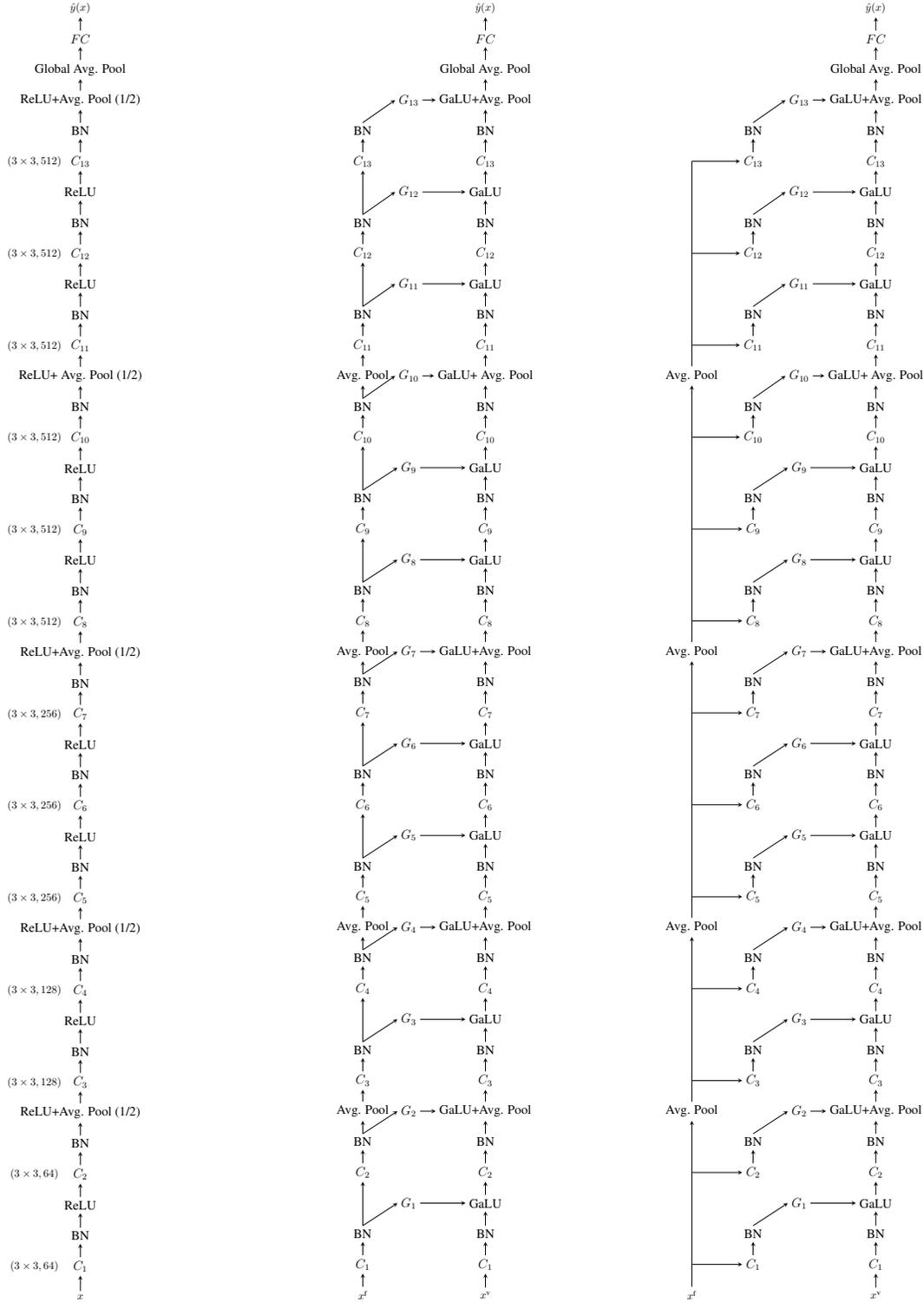

\centering
\begin{minipage}{0.3\columnwidth}
\resizebox{!}{20cm}{
\begin{tikzpicture}

\node [] (dgn-ReLU-14) at (4,42){$\hat{y}(x)$};
%\draw [-stealth,thick]   (dgn-ReLU-14.north) -- (dgn-c15.south);

\node [] (dgn-bn-14) at (4,41){\large{$FC$}};
\draw [-stealth,thick]   (dgn-bn-14.north) -- (dgn-ReLU-14.south);

\node [] (dgn-c14) at (4,40){\large{Global Avg. Pool}};
\draw [-stealth,thick]   (dgn-c14.north) -- (dgn-bn-14.south);

\node [] (dgn-ReLU-13) at (4,39){\large{ReLU+Avg. Pool (1/2)}};
\draw [-stealth,thick]   (dgn-ReLU-13.north) -- (dgn-c14.south);

\node [] (dgn-bn-13) at (4,38){\large{BN}};
\draw [-stealth,thick]   (dgn-bn-13.north) -- (dgn-ReLU-13.south);

\node [] (dgn-c13-text) at (2.5,37){$(3\times 3, 512)$};
\node [] (dgn-c13) at (4,37){\large{$C_{13}$}};
\draw [-stealth,thick]   (dgn-c13.north) -- (dgn-bn-13.south);

\node [] (dgn-ReLU-12) at (4,36){\large{ReLU}};
\draw [-stealth,thick]   (dgn-ReLU-12.north) -- (dgn-c13.south);

\node [] (dgn-bn-12) at (4,35){\large{BN}};
\draw [-stealth,thick]   (dgn-bn-12.north) -- (dgn-ReLU-12.south);

\node [] (dgn-c12-text) at (2.5,34){$(3\times 3, 512)$};
\node [] (dgn-c12) at (4,34){\large{$C_{12}$}};
\draw [-stealth,thick]   (dgn-c12.north) -- (dgn-bn-12.south);

\node [] (dgn-ReLU-11) at (4,33){\large{ReLU}};
\draw [-stealth,thick]   (dgn-ReLU-11.north) -- (dgn-c12.south);

\node [] (dgn-bn-11) at (4,32){\large{BN}};
\draw [-stealth,thick]   (dgn-bn-11.north) -- (dgn-ReLU-11.south);

\node [] (dgn-c11-text) at (2.5,31){$(3\times 3, 512)$};
\node [] (dgn-c11) at (4,31){\large{$C_{11}$}};
\draw [-stealth,thick]   (dgn-c11.north) -- (dgn-bn-11.south);

\node [] (dgn-ReLU-10) at (4,30){\large{ReLU+ Avg. Pool (1/2)}};
\draw [-stealth,thick]   (dgn-ReLU-10.north) -- (dgn-c11.south);

\node [] (dgn-bn-10) at (4,29){\large{BN}};
\draw [-stealth,thick]   (dgn-bn-10.north) -- (dgn-ReLU-10.south);

\node [] (dgn-c10-text) at (2.5,28){$(3\times 3, 512)$};
\node [] (dgn-c10) at (4,28){\large{$C_{10}$}};
\draw [-stealth,thick]   (dgn-c10.north) -- (dgn-bn-10.south);

\node [] (dgn-ReLU-9) at (4,27){\large{ReLU}};
\draw [-stealth,thick]   (dgn-ReLU-9.north) -- (dgn-c10.south);

\node [] (dgn-bn-9) at (4,26){\large{BN}};
\draw [-stealth,thick]   (dgn-bn-9.north) -- (dgn-ReLU-9.south);

\node [] (dgn-c9-text) at (2.5,25){$(3\times 3, 512)$};
\node [] (dgn-c9) at (4,25){\large{$C_9$}};
\draw [-stealth,thick]   (dgn-c9.north) -- (dgn-bn-9.south);

\node [] (dgn-ReLU-8) at (4,24){\large{ReLU}};
\draw [-stealth,thick]   (dgn-ReLU-8.north) -- (dgn-c9.south);

\node [] (dgn-bn-8) at (4,23){\large{BN}};
\draw [-stealth,thick]   (dgn-bn-8.north) -- (dgn-ReLU-8.south);

\node [] (dgn-c8-text) at (2.5,22){$(3\times 3, 512)$};
\node [] (dgn-c8) at (4,22){\large{$C_8$}};
\draw [-stealth,thick]   (dgn-c8.north) -- (dgn-bn-8.south);

\node [] (dgn-ReLU-7) at (4,21){\large{ReLU+Avg. Pool (1/2)}};
\draw [-stealth,thick]   (dgn-ReLU-7.north) -- (dgn-c8.south);

\node [] (dgn-bn-7) at (4,20){\large{BN}};
\draw [-stealth,thick]   (dgn-bn-7.north) -- (dgn-ReLU-7.south);

\node [] (dgn-c7-text) at (2.5,19){$(3\times 3,256)$};
\node [] (dgn-c7) at (4,19){\large{$C_7$}};
\draw [-stealth,thick]   (dgn-c7.north) -- (dgn-bn-7.south);

\node [] (dgn-ReLU-6) at (4,18){\large{ReLU}};
\draw [-stealth,thick]   (dgn-ReLU-6.north) -- (dgn-c7.south);

\node [] (dgn-bn-6) at (4,17){\large{BN}};
\draw [-stealth,thick]   (dgn-bn-6.north) -- (dgn-ReLU-6.south);

\node [] (dgn-c6-text) at (2.5,16){$(3\times 3,256)$};
\node [] (dgn-c6) at (4,16){\large{$C_6$}};
\draw [-stealth,thick]   (dgn-c6.north) -- (dgn-bn-6.south);

\node [] (dgn-ReLU-5) at (4,15){\large{ReLU}};
\draw [-stealth,thick]   (dgn-ReLU-5.north) -- (dgn-c6.south);

\node [] (dgn-bn-5) at (4,14){\large{BN}};
\draw [-stealth,thick]   (dgn-bn-5.north) -- (dgn-ReLU-5.south);

\node [] (dgn-c5-text) at (2.5,13){$(3\times 3,256)$};
\node [] (dgn-c5) at (4,13){\large{$C_5$}};
\draw [-stealth,thick]   (dgn-c5.north) -- (dgn-bn-5.south);

\node [] (dgn-ReLU-4) at (4,12){\large{ReLU+Avg. Pool (1/2)}};
\draw [-stealth,thick]   (dgn-ReLU-4.north) -- (dgn-c5.south);

\node [] (dgn-bn-4) at (4,11){\large{BN}};
\draw [-stealth,thick]   (dgn-bn-4.north) -- (dgn-ReLU-4.south);

\node [] (dgn-c4-text) at (2.5,10){$(3\times 3,128)$};
\node [] (dgn-c4) at (4,10){\large{$C_4$}};
\draw [-stealth,thick]   (dgn-c4.north) -- (dgn-bn-4.south);

\node [] (dgn-ReLU-3) at (4,9){\large{ReLU}};
\draw [-stealth,thick]   (dgn-ReLU-3.north) -- (dgn-c4.south);

\node [] (dgn-bn-3) at (4,8){\large{BN}};
\draw [-stealth,thick]   (dgn-bn-3.north) -- (dgn-ReLU-3.south);

\node [] (dgn-c3-text) at (2.5,7){$(3\times 3,128)$};
\node [] (dgn-c3) at (4,7){\large{$C_3$}};
\draw [-stealth,thick]   (dgn-c3.north) -- (dgn-bn-3.south);

\node [] (dgn-ReLU-2) at (4,6){\large{ReLU+Avg. Pool (1/2)}};
\draw [-stealth,thick]   (dgn-ReLU-2.north) -- (dgn-c3.south);

\node [] (dgn-bn-2) at (4,5){\large{BN}};
\draw [-stealth,thick]   (dgn-bn-2.north) -- (dgn-ReLU-2.south);

\node [] (dgn-c2-text) at (2.5,4){$(3\times 3,64)$};
\node [] (dgn-c2) at (4,4){\large{$C_2$}};
\draw [-stealth,thick]   (dgn-c2.north) -- (dgn-bn-2.south);

\node [] (dgn-ReLU-1) at (4,3){\large{ReLU}};
\draw [-stealth,thick]   (dgn-ReLU-1.north) -- (dgn-c2.south);

\node [] (dgn-bn-1) at (4,2){\large{BN}};
\draw [-stealth,thick]   (dgn-bn-1.north) -- (dgn-ReLU-1.south);

\node [] (dnn-c1-text) at (2.5,1){$(3\times 3,64)$};
\node [] (dgn-c1) at (4,1){\large{$C_1$}};
\draw [-stealth,thick]   (dgn-c1.north) -- (dgn-bn-1.south);

\node [] (dgn-input) at (4,0){$x$};
\draw [-stealth,thick]   (dgn-input.north) -- (dgn-c1.south);

%%%%%%%%%%%%%%%%%%%%%%%%%%%%%%%%%%%%%%%%%%%%%%%%%%%%%%%%%%%%%%%%%

\end{tikzpicture}
}
\end{minipage}
\begin{minipage}{0.3\columnwidth}
\resizebox{!}{20cm}{
\input{fig-vgg-dlgn}
}
\end{minipage}
\begin{minipage}{0.3\columnwidth}
\resizebox{!}{20cm}{
\input{fig-vgg-dlgn-sl}
}
\end{minipage}
\caption{Shows VGG-16 (left), VGG-16-DLGN (middle), VGG-16-DLGN-SF(right).}
\label{fig:vggnets}
\end{figure}

\end{document}